\documentclass{article}

\usepackage{iclr2026_conference, times}

\usepackage[hidelinks]{hyperref}
\usepackage[utf8]{inputenc} % allow utf-8 input
\usepackage[T1]{fontenc}    % use 8-bit T1 fonts
\usepackage{hyperref}       % hyperlinks
\usepackage{url}            % simple URL typesetting
\usepackage{booktabs}       % professional-quality tables
\usepackage{amsfonts}       % blackboard math symbols
\usepackage{nicefrac}       % compact symbols for 1/2, etc.
\usepackage{microtype}      % microtypography
\usepackage{xcolor}         % colors
\usepackage{algorithm, algpseudocode}
\usepackage{graphicx, subfig}
\usepackage[export]{adjustbox}
\usepackage{etoolbox}
\usepackage{minitoc}
\usepackage{soul}

\usepackage{amsmath}
\usepackage{amssymb}
\usepackage{mathtools}
\usepackage{amsthm}

\usepackage[capitalize,noabbrev]{cleveref}

\usepackage{mathtools}

\usepackage{graphicx} % Required for inserting images
\usepackage[utf8]{inputenc} % allow utf-8 input
\usepackage[T1]{fontenc}    % use 8-bit T1 fonts
\usepackage{hyperref}       % hyperlinks
\usepackage{url}            % simple URL typesetting

\usepackage{amsfonts}       % blackboard math symbols
\usepackage{nicefrac}       % compact symbols for 1/2, etc.
\usepackage{microtype}      % microtypography
\usepackage{xcolor}         % colors

\usepackage{booktabs}
\usepackage{array}
\usepackage{enumitem}

\usepackage{mathtools}
\usepackage{multirow}
\usepackage{adjustbox}
\usepackage{comment}
\usepackage{cases}
\usepackage{threeparttable}
\usepackage{wrapfig}
\usepackage{xspace}
\usepackage{graphicx}
\usepackage{caption}
\usepackage{subcaption}

\usepackage[capitalize,noabbrev]{cleveref}

\theoremstyle{plain}
\newtheorem{theorem}{Theorem}[section]
\newtheorem{proposition}[theorem]{Proposition}
\newtheorem{lemma}[theorem]{Lemma}

\theoremstyle{definition}
\newtheorem{definition}[theorem]{Definition}
\newtheorem{assumption}[theorem]{Assumption}
\theoremstyle{remark}
\newtheorem{remark}[theorem]{Remark}

\usepackage[textsize=tiny]{todonotes}

\DeclareMathOperator*{\argmin}{\arg\!\min}
\DeclareMathOperator*{\argmax}{\arg\!\max}

\DeclareMathOperator*{\argsup}{\arg\!\sup}
\DeclareMathOperator*{\esssup}{ess\!\sup}
\DeclareMathOperator*{\essinf}{ess\!\inf}

\def\E{\mathbb{E}}% for expectation
\def\I{\mathcal{I}}% for indicator function
\def\P{\mathbb{P}}% for probability
\def\R{\mathbb{R}}% for real numbers
\def\A{\mathcal{A}}% for action space
\def\cD{\mathcal{D}}% for data set
\def\M{\mathcal{M}}% for MDP
\def\H{\mathcal{H}}% for Entropy
\def\cP{\mathcal{P}}% for ambiguity set of transitions
\def\S{\mathcal{S}}% for state space
\def\X{\mathcal{X}}% measurable space
\def\G{\mathcal{G}}% for fucntion set
\def\bp{\mathbf{p}}

\newcommand{\EXP}[1]{\mathbb{E}\left[ #1 \right]}% for expectation
\newcommand{\DKL}{D_{\text{KL}}}
\def\vr{V_{\mathcal{M}_\delta}^{\pi}}
\def\vro{V_{\mathcal{M}_\delta}^{\star}}
\def\qr{Q_{\mathcal{M}_\delta}^\pi}
\def\tpi{\mathcal{T}^\pi}
\def\Mdelta{\mathcal{M}_\delta}
\def\a{\alpha}

\def\d{\delta}

\def\g{\gamma}

\newcommand{\pv}{\textit{Pendulum}}
\newcommand{\cp}{\textit{Cartpole}}
\newcommand{\lld}{\textit{LunarLander}}
\newcommand{\rv}{\textit{Reacher}}
\newcommand{\hc}{\textit{HalfCheetah}}

\newcommand{\pvs}{\textit{Pendulum} }
\newcommand{\cps}{\textit{Cartpole} }
\newcommand{\llds}{\textit{LunarLander} }
\newcommand{\rvs}{\textit{Reacher} }
\newcommand{\hcs}{\textit{HalfCheetah} }

\title{DR-SAC: Distributionally Robust Soft Actor-Critic for Reinforcement Learning under Uncertainty}

\author{%
  Mingxuan Cui$^{*}$, \quad Duo Zhou\thanks{Equal Contributaion. $^{\dagger}$ Corresponding to \texttt{duozhou2@illinois.edu}}$^{*\dagger}$  \\ \textbf{Yuxuan Han$^{1}$, Grani A. Hanasusanto, Qiong Wang, Huan Zhang, Zhengyuan Zhou$^{1}$}\\
  University of Illinois Urbana-Champaign \qquad $^{1}$New York University. 
}

\iclrfinalcopy
\begin{document}

\maketitle

\addtocontents{toc}{\protect\setcounter{tocdepth}{-1}}
\begin{abstract}

    Deep reinforcement learning (RL) has achieved remarkable success, yet its deployment in real-world scenarios is often limited by vulnerability to environmental uncertainties. Distributionally robust RL (DR-RL) algorithms have been proposed to resolve this challenge, but existing approaches are largely restricted to value-based methods in tabular settings. In this work, we introduce Distributionally Robust Soft Actor-Critic (DR-SAC), the first actor–critic based DR-RL algorithm for offline learning in continuous action spaces. DR-SAC maximizes the entropy-regularized rewards against the worst possible transition models within an KL-divergence constrained uncertainty set. We derive the distributionally robust version of the soft policy iteration with a convergence guarantee and incorporate a generative modeling approach to estimate the unknown nominal transition models. Experiment results on five continuous RL tasks demonstrate our algorithm achieves up to $9.8\times$ higher average reward than the SAC baseline under common perturbations. Additionally, DR-SAC significantly improves computing efficiency and applicability to large-scale problems compared with existing DR-RL algorithms. Code is publicly available at \href{https://github.com/Lemutisme/DR-SAC}{https://github.com/Lemutisme/DR-SAC}.

\end{abstract}

\section{Introduction}\label{sec:intro}

The field of deep reinforcement learning (DRL) has witnessed remarkable progress, enabling agents to learn complex behaviors across a wide range of domains, from game playing to robotic control~\citep{arulkumaran2017DRL, franccois2018DRL, chen2024corrected}. In particular, offline deep reinforcement learning, which learns policies from fixed datasets without additional environment interaction, has gained increasing attention due to its practical relevance. By eliminating online exploration, this paradigm improves safety guarantees and data efficiency. Among these methods, Soft Actor-Critic (SAC; ~\cite{haarnoja2018softv1, haarnoja2018softv2}) is a principled algorithm based on entropy regularized reinforcement learning, commonly known as the soft value framework. This maximum entropy formulation is grounded in theoretical foundations ~\citep{ziebart2010modeling} and has been applied to various contexts, including stochastic control ~\citep{todorov2008general, rawlik2012stochastic, messaoud2024s} and inverse reinforcement learning~\citep{ziebart2008irl, zhou2018irl}. 
%SAC achieves high sample efficiency and state-of-the-art performance in many continuous control tasks. 

%\huan{we can make this sentence more precise - such as mentioning that SAC is using a soft value function. There are many other earlier work that also used an entropy term in the training objective, not just SAC (e.g., check out the papers cited in SAC). Reviewers may feel that we don't know the literature well}.

However, a persistent challenge limiting the deployment of offline deep RL in real-world systems is the sensitivity of learned policies to environmental uncertainties~\citep{whittle1981risk, enders2024risk}. Policies trained in a nominal environment often exhibit significant performance degradation when deployed under a slightly different one. This model mismatch may arise from uncertain transition and reward functions, observation and actuator noise, parameter variations, or adversarial perturbations.

Distributionally robust reinforcement learning (DR-RL) addresses this challenge by optimizing policies against the worst-case scenario. Instead of assuming a single Markov Decision Process (MDP), DR-RL adopts a Robust Markov Decision Process (RMDP) framework, which includes a set of MDPs defined by an uncertainty set of distributions around the nominal one.

Although both value-based~\citep{liu2022distributionally, lu2024distributionally} and policy-gradient~\citep{wang2022policy, kumar2023policy} DR-RL algorithms have been proposed, most existing work focuses on performance guarantees and sample complexity in the tabular setting and cannot be deployed in continuous environments. 
The only notable exception is Robust Fitted Q-Iteration (RFQI;~\cite{RFQI}). However, several fundamental research gaps remain. First, RFQI considers uncertainty sets defined by the Total Variation (TV) distance, which is analytically convenient due to its piecewise linear dual formulation, but does not extend to other divergence measures. Second, its non-robust baseline, Fitted Q-Iteration (FQI;~\cite{ernst2005fqi}), is value-based and suffers from critical limitations, including deterministic learned policies, low applicability to high-dimensional action spaces and high sensitivity to errors in the state-action function~\citep{degris2012off}. 
In contrast, actor-critic algorithms combine low-variance value estimation with scalable policy optimization, making them preferred in continuous control benchmarks and practical applications~\citep{konda1999actor-critic, grondman2012actor-critic}. Despite their empirical success, no distributionally robust counterpart has been developed. This gap motivates our  Distributionally Robust Soft Actor-Critic (DR-SAC), \textit{the first actor–critic-based DR-RL algorithm for offline learning in continuous action spaces}.

In this work, we assume access only to a dataset collected from the training environment, and the transition distribution in the deployment environment lies within an uncertainty set, defined as a Kullback-Leibler (KL) divergence ball centered at the nominal distribution. The goal is to learn a policy that maximizes the soft value function under the worst possible distributions. The main contributions of this work are:

\begin{itemize}[leftmargin=10pt,labelsep=0.5em]
    \item \textbf{Distributionally Robust Maximum Entropy Framework.} We formulate the maximum entropy learning framework under transition uncertainties modeled by KL-divergence-constrained ambiguity sets. Within this framework, we derive distributionally robust soft policy iteration with convergence guarantees and develop distributionally robust counterpart of SAC.
    
    \item \textbf{Scalable Reformulation via Functional Inner Optimization.} We exploit the interchange property to reformulate the per-$(s,a)$ scalar inner optimization problems into shared functional optimization. This reformulation eliminates the dependence on state–action space dimensionality, enables application to continuous action space and saves training time by over $80.0\%$ compared to the existing DR-RL algorithm RFQI.

    \item \textbf{Generative Modeling for KL-Based Robustness.} We introduce generative models (VAEs) to estimate unknown nominal transition distributions and construct empirical measures while preserving robustness. This design resolves the double-sampling issue inherent in the non-linear KL-divergence dual formulation and enables distributionally robust soft policy learning in offline and continuous control tasks. Extensive experiments across five environments demonstrate that DR-SAC achieves up to $9.8\times$ higher average reward than the SAC baseline under perturbations.

\end{itemize}
% {
% \color{blue}

% Contributions:

% 1. from TV to KL(one novelty is the use of VAE)

% 2. DR counterpart of SAC is proposed, which is one of most successful empirical RL algorithm, theoretical guarantee is also established. 

% 3. Expierment is good.(give some quantitived)  1. faster speed, thus can perofrm more scalable dataset. 2. better performance(value improve xxx percent, xxx percent). 
% }

%\huan{One thing I am a bit confused about - are we advertising our algorithm as an offline RL algorithm? It looks all experiments are done in offline setting, but the intro gave me an impression that our algorithm works in both off policy and offline settings, and offline seems to be an add-on?}

\subsection{Related Works}

% \gh{A colleague (Chin Pang Ho) who works on robust MDPs once told me that the existing distributionally robust reinforcement learning (RL) models are often poorly motivated, as the focus on robustifying against worst-case distributions can inadvertently discourage exploration. In contrast, our approach explicitly addresses this limitation by incorporating entropy maximization, which promotes exploration while preserving robustness to distributional ambiguity. } \gh{not sure if true, but you can consider it when writing the literature review} \duo{Sounds good to the story, thanks Dr. Grani and Dr. Ho!}

\paragraph{Robust RL.} The RMDP and Robust Dynamic Programming method were first introduced in~\citep{iyengar2005robustdp, nilim2005robustdp} and have been widely studied in planning settings~\citep{xu2010robustmdp, wiesemann2013robust, yu2015drmdp}. Beyond classical robust formulations, various approaches have been developed to address uncertainty in reinforcement learning, including soft-robustness \citep{derman2018soft_robust,lobo2020soft_robust,park2024distributionally}, risk sensitivity \citep{tamar2015risk, pan2019risk, singh2020improving, queeney2023risk}, and adversarial training \citep{pinto2017adtrain, zhang2020adtrain, cheng2022adtrain}. In recent years, many distributionally robust RL algorithms have been proposed with provable guarantees in the tabular setting. These include model-free algorithms based on $Q$-learning~\citep{Bound-DR-Qlearning, Bound-DR-VR-Qlearning, single-DR-Qlearning} and model-based algorithms extending value iteration \citep{zhou2021DR-VI, panaganti2022gen, DR-Phase-VI, linear-DR-VI, liu2024minimax}. However, these algorithms are not applicable to continuous action space environments. 
% \duo{Add: Can We Estimate the Entropy of Arbitrary Distributions Known Up to a Normalization Constant?,  NeurIPS 2025 Workshop on Structured Probabilistic Inference \& Generative Modeling
% https://openreview.net/forum?id=iOCG9cRoSB
% S²AC: Energy-Based Reinforcement Learning with Stein Soft Actor-Critic, ICLR 24
% https://arxiv.org/pdf/2405.00987}

\paragraph{Model-Free Algorithms for Distributionally Robust RL.}
In DR-RL, the nominal transition distributions typically appear in the optimization objective but are unknown in practice. To address this challenge, some model-free algorithms assume access to a simulator generating \textit{i.i.d.} samples from the nominal environment~\citep{liu2022distributionally, zhou2023ac_simulator, ramesh2024distributionally}, which violates the offline setting. Other algorithms compute empirical transition frequencies from offline datasets~\citep {derman2020distributional, clavier2023towards, shi2024modelbased}, which are not applicable in continuous spaces. Lastly, Empirical Risk Minimization (ERM) has also been used to estimate the robust objectives under special structures \citep{mankowitz2019robust, bahari2022safe}.

\paragraph{VAE in Offline RL.}
Variational Autoencoders (VAEs) have been widely used in non-robust offline RL. A common application is to estimate the behavior policy from offline data, and impose policy constraints or pessimistic value regularization to mitigate distributional shift~\citep{fujimoto2019bcq, wei2021boosting, xu2022constraints, lyu2022mildly}. VAEs have also been used for state reconstruction and representation learning in RL~\citep{van2016stable}. \textit{To the best of our knowledge, this work is the first to incorporate VAE models in a DR-RL algorithm to estimate nominal transition distributions and generate synthetic samples in the absence of a simulator.} 

It is worth noting that~\citet{smirnova2019distributionally} proposed a algorithm with a similar name. However, the problem formulation is completely different from ours and most DR-RL literature. Their framework accounts for estimation error in the evaluation step and employs KL divergence to constrain deviations from the behavior policy within a single MDP rather than an RMDP.

\section{Formulation}

\subsection{Notation and Basics of Soft Actor-Critic}
\label{sec:notation}
A standard framework for reinforcement learning is the discounted Markov Decision Process, formally defined as a tuple $\M = (\mathcal{S}, \mathcal{A}, R, P, \gamma)$, where $\S$ and $\A$ denote the state and action spaces, respectively, both assumed to be continuous in this work. The random reward function modeled as a mapping $R :\S \times \A \mapsto \P([0, R_{\max}])$, where $\P([0, R_{\max}])$ is the set of random variables supported on $[0, R_{\max}]$. The transition distribution is denoted by $P: \S \times \A \mapsto \Delta(\S)$, where $\Delta(\S)$ is the set of probability measures over $\S$, and $\gamma \in [0,1)$ is the discount factor. We denote $r = R(s,a)$ as the random reward and $s'$ as the next state drawn from the transition distribution $p_{s, a} = P(\cdot\mid s,a)$. A policy $\pi:\S \mapsto \Delta(\A)$ represents the conditional probability over actions given state. We consider the stochastic stationary policy class, denoted by $\Pi$. The entropy of a stochastic policy $\pi$ at state $s$ is defined as $\H(\pi(s)) = \EXP{-\log \pi(a|s)} $, measuring the randomness of action. The set of integers from $1$ to $n$ is denoted as $[n]$.
 
In maximum entropy RL, the objective comprises the cumulative discounted reward and an entropy regularization term to encourage exploration. Specifically, given an MDP $\M$, the soft value function under policy $\pi$ is defines as
\begin{equation}
    V^\pi_\M(s) = \E\left[
    \sum_{t=1}^\infty \gamma^{t-1}\Big(r_t + \alpha\cdot\H\big(\pi(s_t)\big)\Big)
    \,\middle\vert\, 
    \pi, s_1 = s
    \right].
\label{eq: v}
\end{equation}
The temperature $\alpha\ge0$ controls the trade-off between reward maximization and policy stochasticity. The optimal soft value and optimal policy are defined as 
\begin{equation}
V_\M^\star = \max_{\pi\in\Pi} V^\pi_\M,\;\;
\pi_\M^\star = \argmax_{\pi\in\Pi} V^\pi_\M.
\end{equation}
Similarly, the soft state-action value function (soft $Q$-function) under policy $\pi$ is defined as
\begin{equation}\begin{aligned}
    &Q^\pi_\M(s,a) =
    & \EXP{
    r_1 + \sum_{t=2}^\infty \gamma^{t-1}\Big(r_t + \alpha\cdot\H\big(\pi(s_t)\big)\Big)
    \,\middle\vert\, \pi, s_1 = s, a_1 = a
    }.
\label{eq: q}
\end{aligned}\end{equation}
For any mapping $Q:\S\times\A\to\R$, \citet{haarnoja2018softv1} defined soft Bellman operator as
\begin{equation}\begin{aligned}
    &\tpi Q(s,a) = 
    &\E[r] + \gamma\cdot\E_{p_{s,a},\pi}\left[Q\left(s',a'\right) - \alpha\log\pi\left(a'\mid s'\right)\right].
\label{eq: q_bellman}    
\end{aligned}\end{equation}

The Soft Actor-Critic (SAC) algorithm updates the policy through soft policy iteration, which is guaranteed to converge in the tabular case. In each iteration, the soft Bellman operator $\tpi$ is applied to update the estimation of the soft $Q$-function under the current policy $\pi$. The policy is then updated by minimizing the KL divergence between the candidate policy and the exponential of the soft Q-function:
\begin{equation}
    \pi_{k+1}
    = \argmin_{\pi\in\Pi} \DKL\left(
    \pi(\cdot\mid s) \,\middle\Vert\,
    \exp\left(\frac{1}{\alpha}Q_\M^{\pi_k}(s,\cdot)\right)
    \,\Big/\, Z(s)\right),\,k=0,1,\cdots
\label{eq: policy improvement}
\end{equation}
where $
\DKL(P \parallel Q) = \E_{P}\left[\log\left(\frac{P(x)}{Q(x)}\right)\right]$ denotes the KL divergence and the function $Z(\cdot)$ is the normalizing function ensuring that the exponential term defines a valid probability distribution.

\subsection{Robust Markov Decision Process}
In real-world RL tasks, the transition distribution $P$ and reward function $R$ in the deployment environment may differ from the environment which the model is trained in or the offline dataset is collected from. Such environmental shifts motivate the study of the Robust Markov Decision Process framework and the goal of learning policies robust to distributional perturbations. Since the analysis and algorithm design for reward perturbations are similar, we assume the reward function $R$ is unchanged and focus only on uncertainty in the transition distributions.

The RMDP framework is denoted as $\M_\delta = (\mathcal{S}, \mathcal{A}, R, \mathcal{P}(\delta), \gamma)$. We consider transition distributions perturbed within a KL-divergence ball. Specifically, let $\cP^0 = \{p_{s,a}^0\}_{(s,a)\in\S\times\A}$ be the nominal transition distributions. For each state-action pair $(s,a)\in\S\times\A$ and $\delta>0$, we define the KL ball centered at $p_{s,a}^0$ as 
\begin{equation}
\mathcal{P}_{s,a}(\delta):=\left\{
p_{s,a}\in\Delta(\mathcal{S}):D_{\text{KL}}(p_{s,a}\| p_{s,a}^0)\leq\delta
\right\}.
\label{eq: kl-dataset}
\end{equation}
The ambiguity set $\cP(\delta)$ is defined as the Cartesian product of $\cP_{s,a}(\delta)$ for all pairs $(s,a)\in\S\times\A$, which belongs to the $(s,a)$-rectangular set~\citep{wiesemann2013robust}. 

Under the RMDP framework, the goal is to optimize performance under the worst-case transition model within the ambiguity set. Given $\M_\delta$, the distributionally robust soft value function under policy $\pi$ is defined as
\begin{equation}\begin{aligned}
    &V^{\pi}_{\M_\delta}(s) =
    & \inf_{\bp \in \cP(\d)} \E_{\bp}\left[
    \sum_{t=1}^\infty \g^{t-1} \Big(r_t + \a \cdot \H\big(\pi(s_t)\big)\Big) 
    \,\middle\vert\, \pi, s_1 = s
    \right].
    \label{eq: dr_v}
\end{aligned}\end{equation}
Similarly, the distributionally robust soft Q-function is given by
\begin{equation}\begin{aligned}
    &Q_{\M_\delta}^{\pi}(s,a) =
    & \inf_{\bp \in \cP(\d)} \E_{\bp} \left[r_1 + \sum_{t=2}^\infty \g^{t-1} \Big(r_t + \a \cdot\H\big(\pi(s_t)\big)\Big) 
    \,\middle\vert\, \pi, s_1 = s, a_1 = a\right].
    \label{eq: dr_q}
\end{aligned}\end{equation}
The distributionally robust optimal value and optimal policy are defined as:
\begin{equation}
    \vro (s) = \max_{\pi \in \Pi} \vr(s)\;\;\textup{and}\;\;
    \pi^\star_{\Mdelta} = \argmax_{\pi \in \Pi} \vr(s).
\end{equation}
\section{Algorithm: Distributionally Robust Soft Actor-Critic}

In this section, we develop the Distributionally Robust Soft Actor-Critic algorithm. We first derive the distributionally robust soft policy iteration and establish its convergence to the optimal policy. To improve computational efficiency, we develop a scalable implementation by replacing the per-$(s,a)$ scalar inner optimization with a shared parametric optimization. Lastly, to handle the unknown nominal distribution in offline settings, we incorporate generative modeling to construct the empirical transition measures.

\begin{assumption} \label{assump: finite action space}
To ensure that the policy entropy $\H(\pi(s)) = \mathbb{E}_{a \sim \pi(\cdot\mid s)}[-\log \pi(a|s)]$ is bounded, we assume $|\A| < \infty$.
\end{assumption}

\begin{remark}
Assumption~\ref{assump: finite action space} is inherited from the non-robust baseline SAC~\citep{haarnoja2018softv1}, which establishes theoretical guarantees in the tabular setting while being empirically used in continuous control benchmarks. Our work extends the performance properties of SAC to the DR-RL framework.
In Section~\ref{sec:algorithm}, we design a practical algorithm in continuous action spaces.
\end{remark}

\subsection{Distributionally Robust Soft Policy Iteration}
We begin with the distributionally robust soft policy iteration, which iterates between DR soft policy evaluation and DR soft policy improvement. We further show that this iteration is guaranteed to converge to the DR optimal policy.

\paragraph{DR soft policy evaluation.} For a fixed policy $\pi$, the DR soft $Q$-function is estimated by iteratively applying the distributionally robust soft Bellman operator, which considers the worst possible transition distribution within the uncertainty set.
For any bounded mapping $Q:\S\times\A\to\R$, the distributionally robust soft Bellman operator is defined as:

\begin{equation}\begin{aligned}
    \tpi_\delta Q(s,a) :=\E[r] + \gamma\cdot
    \inf_{p_{s,a}\in \cP_{s,a}(\d)}\left\{\E_{p_{s,a}, \pi}\left[Q(s',a') - \a\cdot\log\pi(a'\mid s')\right]\right\}.
\label{eq: q_bellman primal}
\end{aligned}\end{equation}

Following~\citet{iyengar2005robustdp, xu2010robustmdp}, the DR soft $Q$-function can be computed via distributionally robust dynamic programming, and $\qr$ is a fixed point of $\tpi_\delta$. However, operator $\tpi_\delta$ is generally intractable because it requires solving an infinite-dimensional optimization problem over the transition distributions.
To address this, we apply the strong duality for worst-case expectations over a KL-divergence ball and derive a equivalent dual formulation.

\begin{proposition}[Dual Formulation of the Distributionally Robust Soft Bellman Operator]
\label{prop: q_bellman dual}
    Suppose $Q(s,a)$ is bounded, the distributionally robust soft Bellman operator in Equation~\eqref{eq: q_bellman primal} can be reformulated into:
    \begin{equation}
        \tpi_\delta Q(s,a) = \E[r] 
        +\gamma\cdot \sup_{\beta\ge0}\left\{-\beta\log\left(\E_{p_{s,a}^0 }\left[\exp\left(-\frac{V(s')}{\beta}\right)\right]\right) - \beta\d\right\},
    \label{eq: q_bellman dual}
    \end{equation}
    where 
    \begin{equation}
    V(s) = \E_{a\sim \pi}\left[Q(s,a) - \a\cdot\log\pi(a\mid s)\right].
    \end{equation}
\end{proposition}

Derivation is provided in Appendix~\ref{prof: prop q_bellman dual}. Importantly, the dual form depends only on the nominal transition distribution $\cP_{s, a}^0$, instead of an infinite number of distributions in the uncertainty set $\cP(\d)$. 
Also, the inner optimization problem is reduced to a one-dimensional problem over the scalar $\beta$, rather than infinite-dimensional distributions. Using this tractable dual form operator, the DR soft $Q$-value for a fixed policy $\pi$ can be computed by iteratively applying $\tpi_\delta$.

\begin{proposition}[Distributionally Robust Soft Policy Evaluation]
\label{prop: soft policy evaluation}
For any fixed policy $\pi\in\Pi$, starting from any bounded mapping $Q^0:\S\times\A\to\R$, define a sequence $\{Q^k\}$ by iteratively applying distributionally robust soft Bellman operator: $Q^{k+1} = \tpi_\delta Q^k$. This sequence converges to the DR soft $Q$-value $Q_{\M_\delta}^\pi$ as $k\to\infty$.
\end{proposition}

% \begin{proof}[Proof Sketch of Proposition~\ref{prop: soft policy evaluation}]
% We show that the operator $\tpi_\delta$ is a $\gamma$-contraction mapping by delineating the possible values of the optimal solution $\beta^\star$ in \eqref{eq: q_bellman dual}. With DR soft Q-value as the fixed point of $\tpi_\delta$, convergence is a direct conclusion. See details in Appendix~\ref{prof: prop policy evaluation}.
% \end{proof}
The proof shows that the operator $\tpi_\delta$ is a $\gamma$-contraction mapping, with details in Appendix~\ref{prof: prop policy evaluation}.

\paragraph{DR soft policy improvement.}
The distributionally robust soft policy improvement step is similar to that in SAC, with DR soft $Q$-value $Q_{\M_\delta}$ replacing its non-robust counterpart. The new policy in each update is defined as 
\begin{equation}
    \pi_{k+1}
    = \argmin_{\pi\in\Pi} \DKL\left(
    \pi(\cdot\mid s) 
    \,\middle\Vert\,
    \exp\left(\frac{1}{\alpha}Q_{\M_\delta}^{\pi_k}(s,\cdot)\right)
    \,\Big/\,Z^{\pi_k}(s)\right),\,k=0,1,\cdots
\label{eq: dr policy improvement}
\end{equation}

With the above policy updating rule, $\pi_k$ has a non-decreasing value with respect to the DR soft $Q$-function. This extends the non-robust soft policy improvement to the uncertain transition distribution case.

\begin{proposition}[Distributionally Robust Soft Policy Improvement]
\label{prop: soft policy improvement}
Suppose $\lvert\A\rvert<\infty$, let $\pi_{k+1}$ be the solution to the optimization problem above. Then 
\begin{equation}
Q^{\pi_{k+1}}_{\Mdelta}(s,a) \ge Q^{\pi_k}_{\Mdelta}(s,a),\quad\forall(s,a)\in\S\times\A.
\end{equation}
\end{proposition}

% \begin{proof}[Proof Sketch of Proposition~\ref{prop: soft policy improvement}]
% The proof is similar to \cite{haarnoja2018softv1}, but replacing soft Bellman equation with DR version.  See details in Appendix~\ref{prof: prop policy improvement}.
% \end{proof}

The proof is provided in Appendix~\ref{prof: prop policy improvement}. The DR soft policy iteration algorithm alternates DR soft policy evaluation and DR soft policy improvement. The following theorem shows convergence to the DR optimal policy, with proof in Appendix~\ref{prof: thm policy iteration}. 

\begin{theorem}[Distributionally Robust Soft Policy Iteration]
\label{thm: soft policy iteration}
Suppose $\lvert\A\rvert<\infty$, starting from any policy $\pi^0\in\Pi$, the policy sequence $\{\pi^k\}$ converges to the optimal policy $\pi^\star$ as $k\to\infty$. 
\end{theorem}

% \duo{Mention here, to solve \ref{eq: q_bellman dual}, we still have 2 challenges to be solved: 1. how to estimate the nominal distribution and 2. how to solve the optimization in scale}

\paragraph{Key Challenges.}
Although DR soft policy iteration converges to the optimal policy in the tabular setting, several challenges arise in extending it to continuous action space and offline setting: 1) the DR soft policy evaluation step is computationally expensive at scale due to the per-$(s,a)$ inner optimization over $\beta$; 2) the nominal transition distribution $p_{s, a}^0$ is typically unknown in offline RL tasks, and 3) exact DR soft policy iteration is not directly applicable in continuous action space. We will resolve these issues step by step in the rest of this section.

\subsection{Solving Dual Optimization using Generative Model}
\label{section: func approx}
In offline reinforcement learning, the goal is to learn the optimal policy from a pre-collected dataset $\cD = \{(s_i,\,a_i,\,r_i,\,s'_i)\}_{i=1}^N$, where $(s_i,a_i)\sim\mu$, with $\mu$ denoting the data generation distribution determined by the behavior policy, $r_i = R(s_i,a_i)$ and $s'_i\sim P^0(\cdot\mid s_i,a_i)$. 
% In this section, we derive a practical functional optimization method to compute the dual formulation of DR soft Bellman operator in \eqref{eq: q_bellman dual} with higher efficiency to address challenge 1, and propose a generative modeling scheme to address challenge 2. 
In this section, we address the key computational and modeling challenges of DR soft policy iteration. Specifically, we (i) develop a scalable reformulation of the dual Bellman operator by approximating the per-$(s,a)$ scalar optimization into a shared optimization problem over a function space, and (ii) introduce a generative modeling scheme to estimate unknown nominal transition distributions.

\paragraph{Dual Reformulation via Functional Optimization.} 
In DR soft policy evaluation, the Bellman operator $\tpi_\d$ is iteratively applied to the $Q$-function. From the dual form operator in \eqref{eq: q_bellman dual}, each application requires solving an optimization problem over a scalar $\beta>0$. While this optimization is tractable, it must be solved separately for every $(s, a)$ pair, which becomes computationally expensive for large-scale problems. 
To improve training efficiency, we convert a group of scalar optimization problems into a single functional optimization problem. Intuitively, instead of solving for the optimal 
$\beta^\star$ separately at each state–action pair, we learn a function that approximates these optimal values jointly across the dataset. This can be achieved by the interchange of minimization and integration in decomposable spaces~\citep{rockafellar2009variational}.

Formally, consider the probability space $(\S\times\A, \Sigma(\S\times\A), \mu)$ and let $L^1(\S\times\A, \Sigma(\S\times\A), \mu)$ be the set of absolutely integrable functions on that space, abbreviated as $L^1$.
For any $\d>0$ and value function $V:\S\to\left[0, (R_{\max}+\a\log\lvert\A\rvert)/(1-\gamma)\right]$, define
\begin{equation}
f((s,a), \beta) := -\beta\log\left(\E_{p_{s,a}^0 }\left[\exp\left(-\frac{V(s')}{\beta}\right)\right]\right) - \beta\d.
\label{eq: beta opt problem}
\end{equation}
Assume $\lvert\A\rvert<\infty$, define the admissible function set
\begin{equation}
\G := \left\{g\in L_1: g(s,a)\in\left[0,\frac{R_{\max}+\a\log\lvert\A\rvert}{(1-\gamma)\delta}\right],\,\forall(s,a)\in\S\times\A\right\}.
\end{equation}

\begin{proposition}[Interchange of Minimization and Expectation]
\label{prop: functional}
For any $\delta>0$,
\begin{equation}
\E_{(s,a)\sim\cD}\left[\sup_{\beta\ge0}f\big((s,a),\beta\big)\right]
=\sup_{g\in \G}\E_{(s,a)\sim\cD}\Big[f\big((s,a), g(s,a)\big)\Big].
\label{eq: interchange v}
\end{equation}
\end{proposition}

% \begin{proof}[Proof Sketch of Proposition~\ref{prop: functional}]
% This is supported by the \textit{interchange of minimization and integration} property of decomposable spaces in \cite{rockafellar2009variational}, which is applied to replace pointwise optimal conditions with optimality in a function space. We find a uniform bound on the pointwise optimal value $\beta^\star$ and utilize the previous property to prove this proposition. Details in Appendix~\ref{prof: prop functional}. 
% \end{proof}
The proof is provided in Appendix~\ref{prof: prop functional}. The results allow us to solve a single optimization problem over the function $g$ instead of the $\lvert\cD\rvert$ scalar optimization problems. In practice, $g$ is learned jointly across the dataset, substantially reducing training time while preserving robustness.

Based on Proposition~\ref{prop: functional}, we define a \textit{functional} DR soft Bellman operator by replacing the scalar $\beta$ with a function $g(s,a)$ and removing the inner optimization. For any function $g\in\G$ and mapping $Q:\S\times\A\to\left[0, (R_{\max}+\a\log\lvert\A\rvert)/(1-\gamma)\right]$, let 
\begin{equation}\begin{aligned}
    \tpi_{\delta,g}Q(s,a) :=&
    \E[r] +\gamma\cdot f\big((s,a), g(s,a)\big)\\
     =&\E[r] 
    +\gamma\cdot \left\{-g(s,a)\log\left(\E_{p_{s,a}^0 }\left[\exp\left(-\frac{V(s')}{g(s,a)}\right)\right]\right) - g(s,a)\d\right\}, 
\label{eq: q_bellman dual functional}
\end{aligned}\end{equation}
where $V(s) = \E_{a\sim \pi}\left[Q(s,a) - \a\cdot\log\pi(a\mid s)\right].$
% From Proposition~\ref{prop: functional}, we have a direct conclusion that $\lVert\tpi_\d Q -  \tpi_{\delta,g^\star} Q\rVert_{1,\mu}=0$, where $g^\star = \argsup_{g\in\G}\E_{(s,a)\sim\cD}\Big[f\big((s,a), g(s,a)\big)\Big]$. 
% In standard SAC algorithm, the soft Q-function is trained to minimize the the mean squared error of current estimate and target Q-value from Bellman equation:
% $$
% J_Q = \E_{(s,a)\sim\cD}\left[Q^\pi_{\M}(s,a) - \mathcal{T}^\pi Q^\pi_{\M}(s,a)\right]^2.
% $$
% In robust MDP $\M_\d$, we replace Bellman update with its distributionally robust version and change target Q-value accordingly. The soft Q-function loss becomes:
% $$
% J_Q^{\text{rob}} = \E_{(s,a)\sim\cD}\left[Q^\pi_{\M}(s,a) - \mathcal{T}^\pi_{\d} Q^\pi_{\M_\d}(s,a)\right]^2.
% $$
% To make training process more efficient, we replace it with
% $$
% J^{\text{rob}}_Q =  \E_{(s,a)\sim\cD}\left[\qr(s,a) -  \tpi_{\d,g^\star}\qr(s,a)\right]^2.
% $$
% where 
% \begin{equation}
% \label{eq: opt g}
% g^\star = \arg\max_{g\in\G}\E_{(s,a)\sim\cD}\Big[f\big((s,a),g(s,a)\big)\Big]
% \end{equation} 

% \gh{This seems intractable because the decision variables are infinite dimensional? }
% \mx{g is infinite dimensional, in practical algorithm we use neural network to approximate it. we do need a better way to explain.} \gh{ok }

% Algorithm with new loss function requires solving three optimization problems in each step, instead of an independent one for each $(s,a)$ pair.

\paragraph{Generative Modeling for Nominal Distributions.}
\label{sec:gen model}
% In offline RL problem, we assume the nominal distributions $\cP^0$ is unknown and no simulator is available to generate additional samples. Note that Empirical Risk Minimization (ERM) is not applicable in our case since operator $\tpi_{\delta,g^\star}$ returns to the non-robust one. To empirically apply operator $\tpi_{\delta,g}$ in discrete space, the nominal distribution can be estimated as $\widetilde{P}(s'\mid s,a) = N(s,a,s') / \lvert\cD\rvert$, where $N(s,a,s')$ is the number of tuple $(s,a,s')$ in $\cD$. Similar idea has been used in \cite{panaganti2022gen, Bound-DR-VR-Qlearning, shi2024modelbased}. \duo{Those should be moved to related works} In continuous state and action space, we incorporate a generative distribution model into RL framework and utilize self-generated data to solve the optimization problem empirically. \mx{rethink the statement logic.}

% Variational Auto Encoder (VAE) is one of the most popular methods to learn complex distributions and has showed superior performance in generating different type of data. In VAE, the encoder maps data tuple $(s,a,s')\in\cD$ into a latent space $z$ by assuming Gaussian prior distribution and approximating the posterior distribution with neural network. The decoder then reconstruct next state $s'$ by sampling $z$ from latent space and feed $(s,a,z)$ to decoder.

In offline RL, we assume that the nominal distributions $\cP^0$ are unknown, and no simulator is available to generate additional samples. Under the KL-constrained uncertainty set, the dual optimization problem in the DR soft Bellman operator (both original and functional) is non-linear. Directly estimating the required expectations from the offline dataset $\cD$ suffers from the \textit{double-sampling issue}, making empirical risk minimization inapplicable. A detailed discussion is provided in Appendix~\ref{sec:discuss vae}.

To enable practical implementation of operator $\tpi_{\delta,g}$ in continuous space, we incorporate a generative model to estimate the nominal transition distributions. To be specific, we train a variational autoencoder (VAE) model on collected data $(s,a,s')\in\cD$ to learn transition $p_{s,a}^0$. The trained VAE generates next-state samples $\{\tilde{s}_i'\}_{i=1}^m$ and construct an empirical measure $\tilde{p}_{s,a}^0$. For any function $h:\S\mapsto\R$, the empirical expectation is defined as  $\E_{s'\sim\tilde{p}_{s,a}^0}[h(s')] = \frac1m \sum_{i=1}^m h(\tilde{s}'_i)$. We define the \textit{empirical} DR soft Bellman operator as
\begin{equation}
    \widetilde{\mathcal{T}}^\pi_{\delta,g}Q(s,a) :=
     \E[r] + \gamma\cdot \widetilde{f}\big((s,a), g(s,a)\big),
\label{eq: q_bellman dual functional empirical}
\end{equation}
where 
\begin{equation}
    \widetilde{f}\big((s,a), \beta\big) = 
-\beta\log\left(\E_{\widetilde{p}_{s,a}^0 }\left[\exp\left(-\frac{V(s')}{\beta}\right)\right]\right) - \beta\d.
\label{eq: empirical beta opt problem}
\end{equation}
% We discuss why other model-free methods are not applicable in our case in Appendix~\ref{sec:discuss vae}. Also, our algorithm is still model-free since it does not require storing every transition function and building an empirical MDP. Updating rules in \eqref{eq: empirical beta opt problem} can be implemented for each $(s,a)$-pair separately, and all generated samples $\tilde{s}_i'$ can be discarded immediately.

%\huan{But we also have no guarantees? What if the VAE is a bad one? What is the assumption made on VAE to make our method work? Our VAE is just a empirical trick? I am a bit confused with the VAE part.}\duo{VAE is an important component in the dual optimization, without the nominal distribution, it will degrade to un-robust algorithm.}
% \yx{several related work about using VAE to estimate nominal distribution in non-robust offline RL: \cite{caselles2018continual,lyu2022mildly,fujimoto2019bcq,xu2022constraints,wei2021boosting}, see also section~3 of \cite{chen2024deep} for more related works.  It may be helpful to include these works in the related work section and this section to illustrate why applying a VAE in our setting can lead to strong empirical performance.}

\subsection{Distributionally Robust Soft Actor-Critic}
\label{sec:algorithm}
We now extend the action space to the continuous setting and use neural networks to approximate the DR soft value function and policy. We consider RMDP $\M_\d$ and omit subscripts in $V$ and $Q$. Our algorithm includes a value network $V_\psi(s)$, $Q$-networks $Q_\theta(s,a)$ and a stochastic policy $\pi_\phi(a\mid s)$, parametrized by $\psi, \theta$ and $\phi$. $\bar{\psi}$ and $\bar{\theta}$ are the target network parameters to stabilize training~\citep{mnih2015soft_update}. Let $\varphi$ be the parameters of VAE model. We use a parametrized neural network $\G_\eta$ to approximate the function set $\G$.

The core idea of our DR-SAC algorithm is to alternate between \textit{empirical} DR soft policy evaluation and DR soft policy improvement. The loss of $Q$-network is
\begin{equation}
\label{eq: dr q loss}
    J^{\text{DR}}_Q(\theta) = \E_{(s,a)\sim\cD}\left[\frac12\left(Q_\theta(s,a) - \tpi_{\d, \widetilde{g}^\star}Q_\theta(s,a)\right)^2\right],
\end{equation}
where
\begin{equation}
\label{eq: empirical opt g}
    {\widetilde{g}}^\star = \argsup_{g\in\G_\eta} \E_{(s,a)\in\cD}\left[\widetilde{f}((s,a),g(s,a))\right].
\end{equation}
The loss functions of $\psi,\,\phi$ and $\alpha$ are the same as SAC, and the loss function of $\varphi$ is the standard VAE evidence lower bound (ELBO) loss.
We adopt the SAC-v1 algorithm~\citep{haarnoja2018softv1} with explicit $V$-function, as we observe empirically that including a $V$-network reduces sensitivity to the behavior policy underlying the offline dataset. Ablation studies are presented in Appendix~\ref{sec:ablation use v}. 
To mitigate overestimation bias, we employ multiple Q-functions $Q_{\theta_i},(i\in[n])$, train independently, and use the minimum in updating the value and policy networks. This has been shown to outperform the clipped double $Q$-learning ($n=2$) in offline RL~\citep{an2021sac-n}. We formally present the Distributionally Robust Soft Actor-Critic in Algorithm~\ref{alg:DR-SAC}. Detailed loss functions are provided in Appendix~\ref{sec:algo detail}. We also derived a regret bound in Appendix~\ref{sec:bound}.

\begin{algorithm}[h]
	\caption{Distributionally Robust Soft Actor-Critic (DR-SAC)} 
    \label{alg:DR-SAC}
	\begin{algorithmic}[1]
    \Require {
    Offline dataset $\mathcal{D} = \{(s_i,a_i,r_i,s'_i)^N _{i=1}\}$, 
    $V$-function network weights $\psi$, 
    $Q$-function network weights $\theta_i, i \in [n]$,
    policy network weights $\phi$, 
    transition VAE network weights $\varphi$,
    weight $\tau$ for moving average,
    function class $\G_\eta$}
    \State {$\bar{\psi}\leftarrow\psi,
            \bar{\theta}_i \leftarrow \theta_i$ for $i\in[n]$}
    \Comment {Initialize target network weights for soft update}
		\For {each gradient step}
            \State{$\varphi \leftarrow \varphi - \lambda_\varphi\hat{\nabla}_\varphi J_{\text{VAE}}(\varphi)$} %for $i\in\{r,s\}$}
            \Comment{Update transition VAE weights}

            \State{Generate samples $\{\tilde{s}'_i\}_{i=1}^m$ from VAE, form empirical measures $\widetilde{p}_{s,a}^0$}
            
		\State {Compute optimal function $\widetilde{g}^\star$ according to \eqref{eq: empirical opt g}}

            \State{$\psi \leftarrow \psi - \lambda_\psi\hat{\nabla}_\psi J_V(\psi)$} 
            \Comment{Update $V$-function weights} 
            
            \State{$\theta_i \leftarrow \theta_i - \lambda_Q\hat{\nabla}_{\theta_i}J^{\text{DR}}_Q(\theta_i)$ for $i\in [n]$ } 
            \Comment{Update $Q$-function weights}

            \State{$\phi \leftarrow \phi - \lambda_\pi\hat{\nabla}_\phi J_\pi(\phi)$} 
            \Comment{Update policy weights}

            \State{$\alpha \leftarrow \alpha - \lambda_\alpha\hat{\nabla}_\alpha J(\alpha)$} 
            \Comment{Adjust temperature}

            \State{$\bar{\psi} \leftarrow \tau\psi + (1-\tau)\bar{\psi}$, $\bar{\theta}_i \leftarrow \tau\theta_i + (1-\tau)\bar{\theta}_i$ for $i\in[n]$} 
            \Comment{Update target network weights}
            
		\EndFor

        \Ensure{ $\phi$}
	\end{algorithmic} 
\end{algorithm}
%\vspace{-1em}
\section{Experiments}
\label{sec:exp}
The goal of our experiments is to demonstrate the robustness of DR-SAC under environmental uncertainties in offline RL tasks. 
We measure performance by the average episode rewards under different perturbations, and compare DR-SAC with non-robust baselines and RFQI. To the best of our knowledge, RFQI is the only offline DR-RL algorithm applicable to continuous action spaces. 
Moreover, extensive ablation studies demonstrate that VAE-based DR-SAC with functional approximation achieves the best trade-off between robustness and computational efficiency. 
%This performance can be attributed to three factors: (i) functional approximation does not compromise robustness in practice; (ii) VAE-based uncertainty modeling is stable and efficient; and (iii) the resulting optimization problem remains simple compared to alternative robust formulations.

\subsection{Settings}
 \vspace{-0.5em}
We implement SAC and DR-SAC based on the multiple critic version SAC-N \citep{an2021sac-n}. Besides RFQI, we also compare DR-SAC with Fitted Q-Iteration (FQI), Deep Deterministic Policy Gradient (DDPG; \cite{lillicrap2015ddpg}), and Conservative Q-Learning (CQL; \cite{kumar2020cql}).

We consider \pv, \cp, \lld, \rvs and \hcs environments from Gymnasium \citep{towers2024gymnasium}. For \cp, we use the continuous action space version in \cite{mehta2021continuous-cp}. For \lld, we also adopt a continuous action space setting. All algorithms are trained in the nominal environment and evaluated under various perturbations. In our experiments, perturbations include environment parameter changes, random noise added to observed states and random actuator noise applied to actions. Detailed experimental settings are provided in Appendix~\ref{sec:exp_setting}

\subsection{Performance Analysis}
\label{sec:experiment performance}

 \vspace{-2em}
\begin{figure}[bh]
  \centering
  \subfloat[\centering Length Perturbation \\ \pv]
  {%
    \includegraphics[width=0.33\textwidth,valign=t]{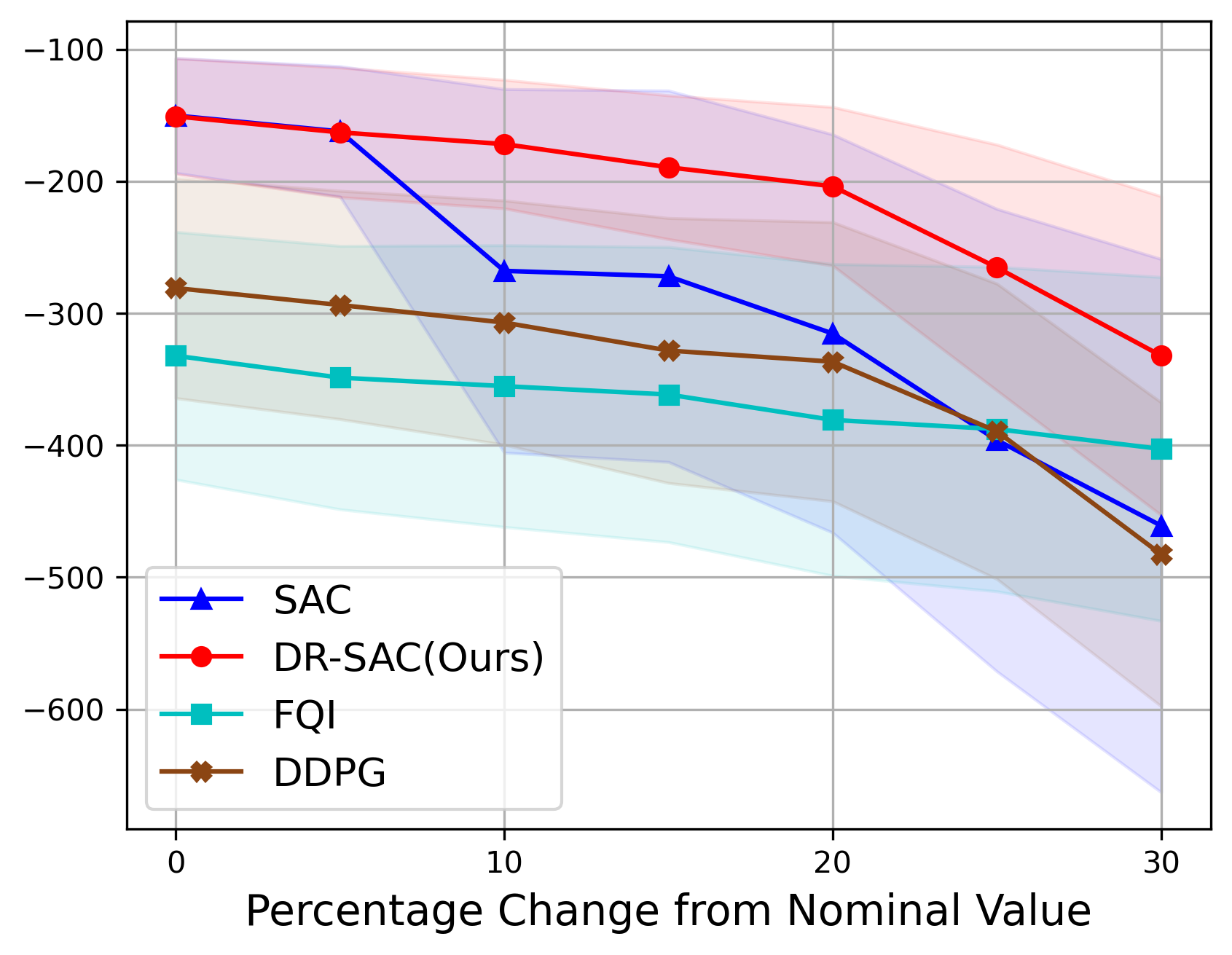}%
  } \hfill
  \subfloat[\centering Action Perturbation \\ \cp]
  {%
    \includegraphics[width=0.33\textwidth,valign=t]{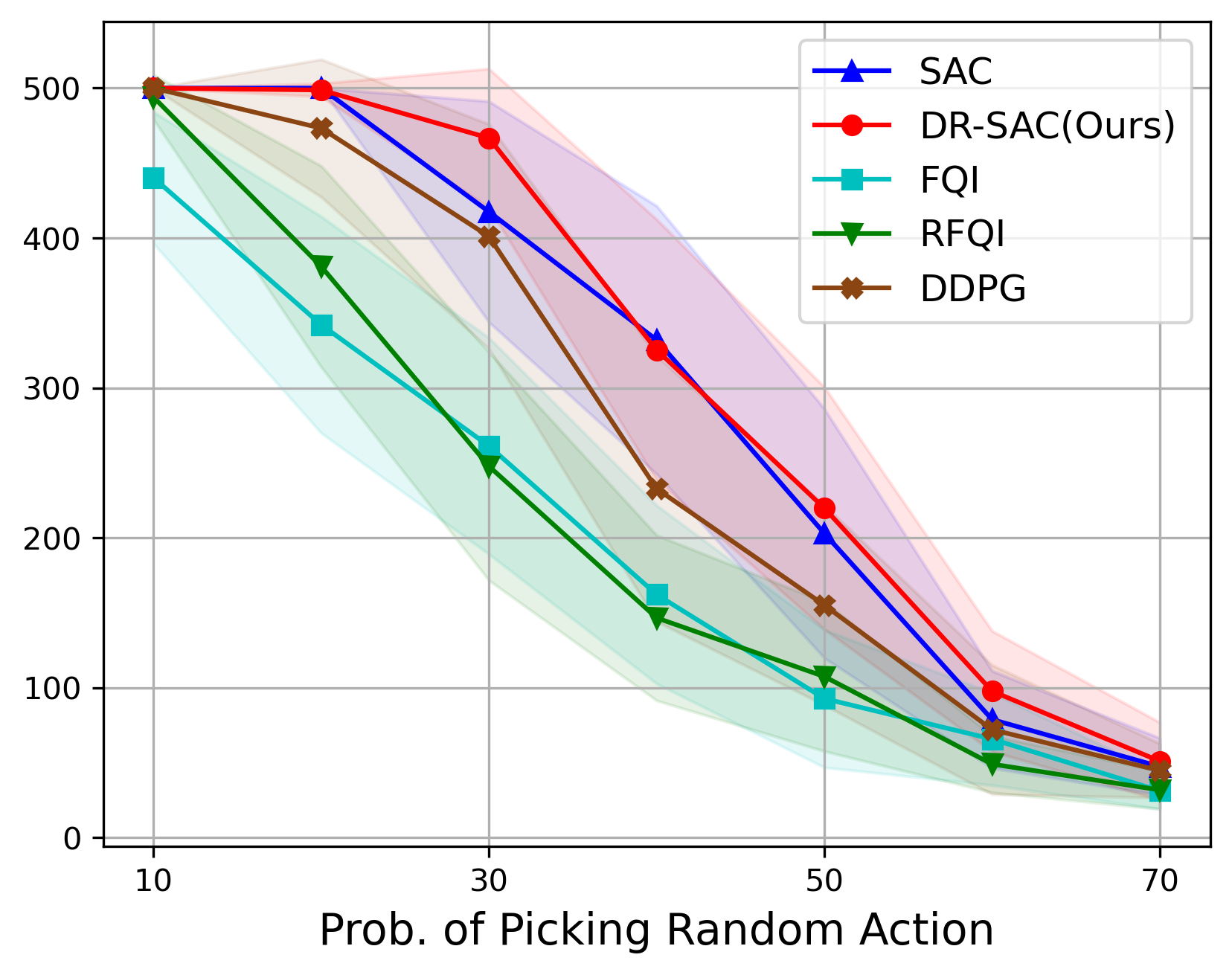}%
  } \hfill
  \subfloat[\centering Engine Perturbation \\ \lld]
  {%
    \includegraphics[width=0.33\textwidth,valign=t]{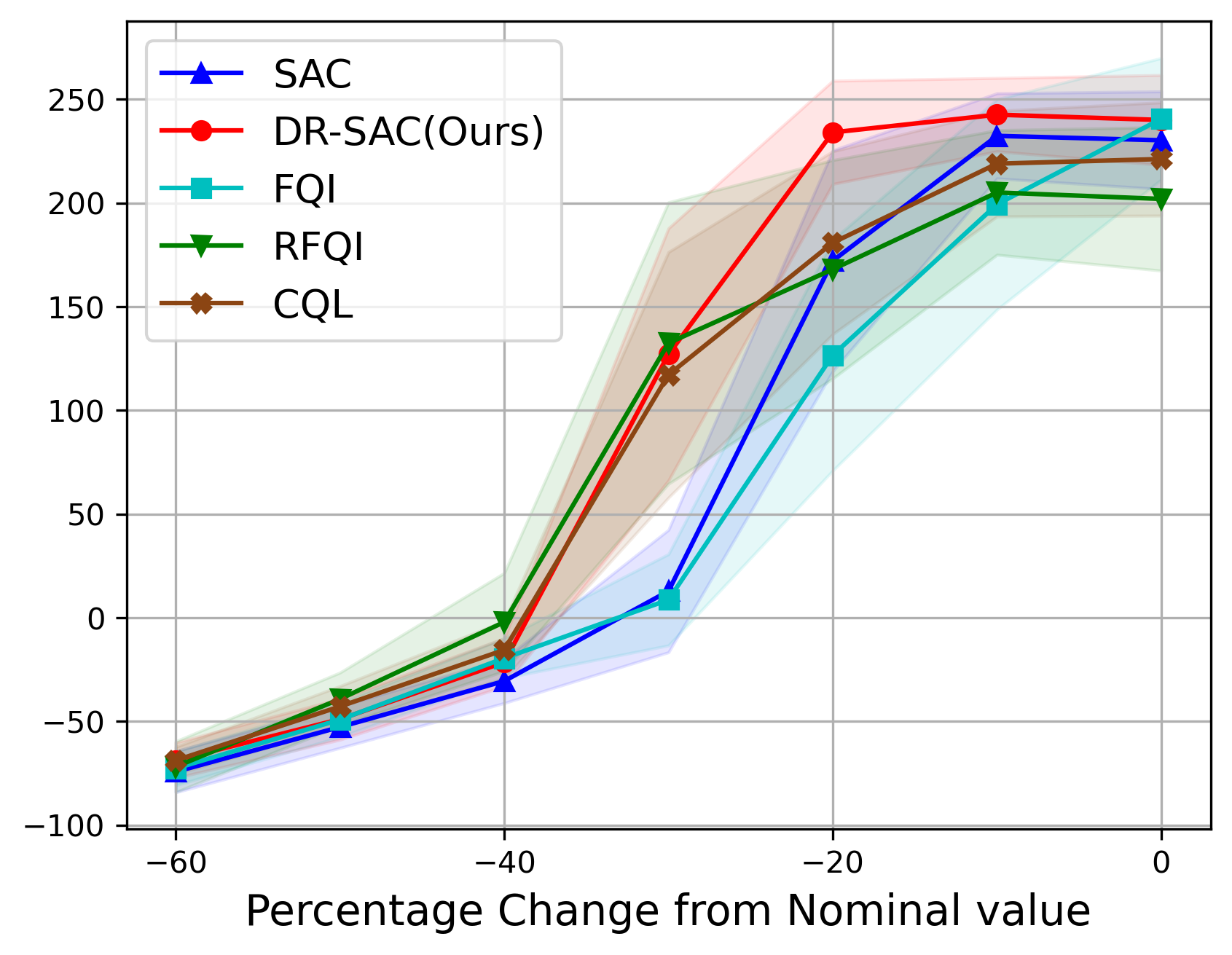}%
  }
    \vspace{-1em}
  \subfloat[\centering Observation Perturbation\\ \rv]
  {%
    \includegraphics[width=0.33\textwidth,valign=t]{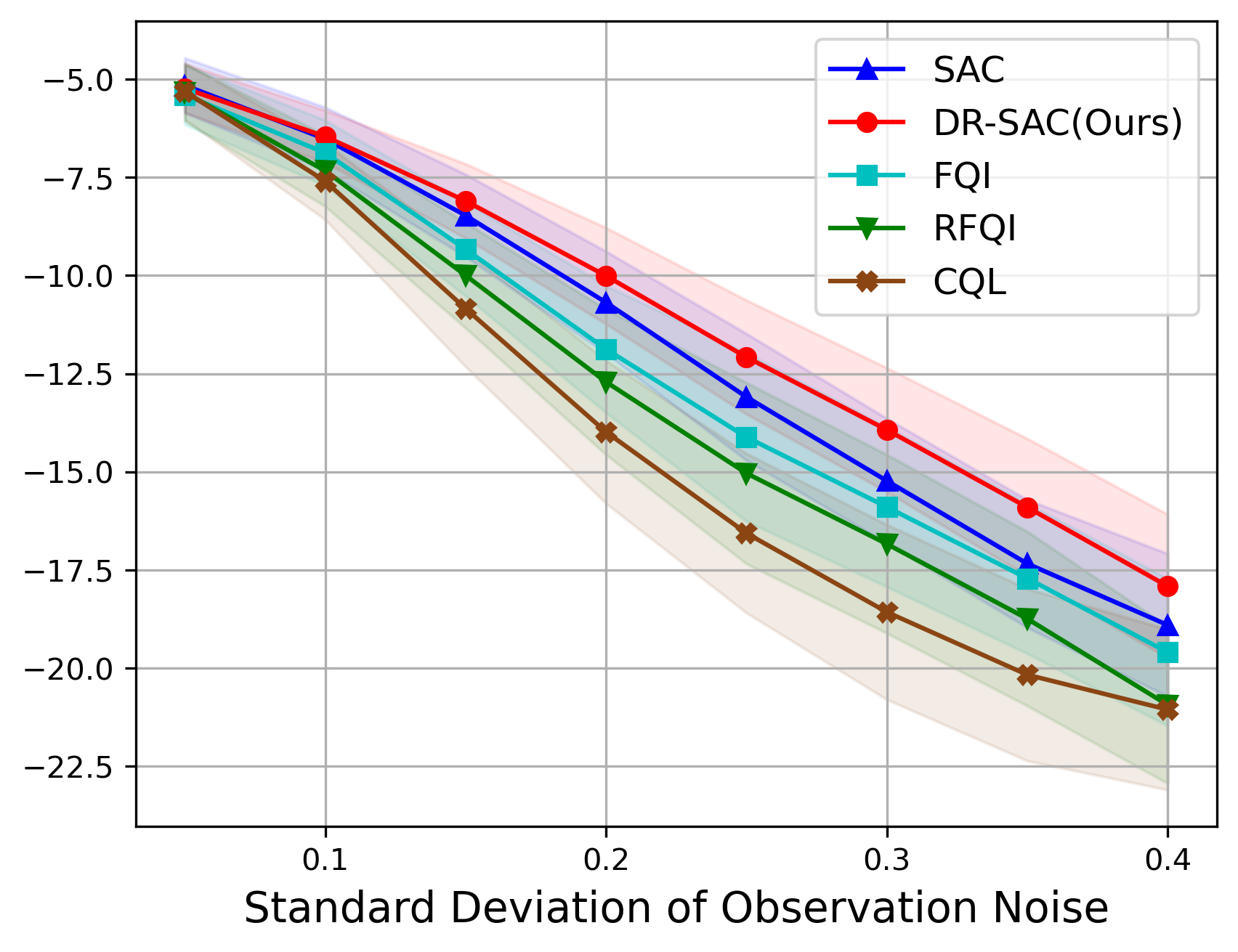}%
  } \hfill
  \subfloat[\centering Damping Perturbation \\ \rv]
  {%
    \includegraphics[width=0.33\textwidth,valign=t]{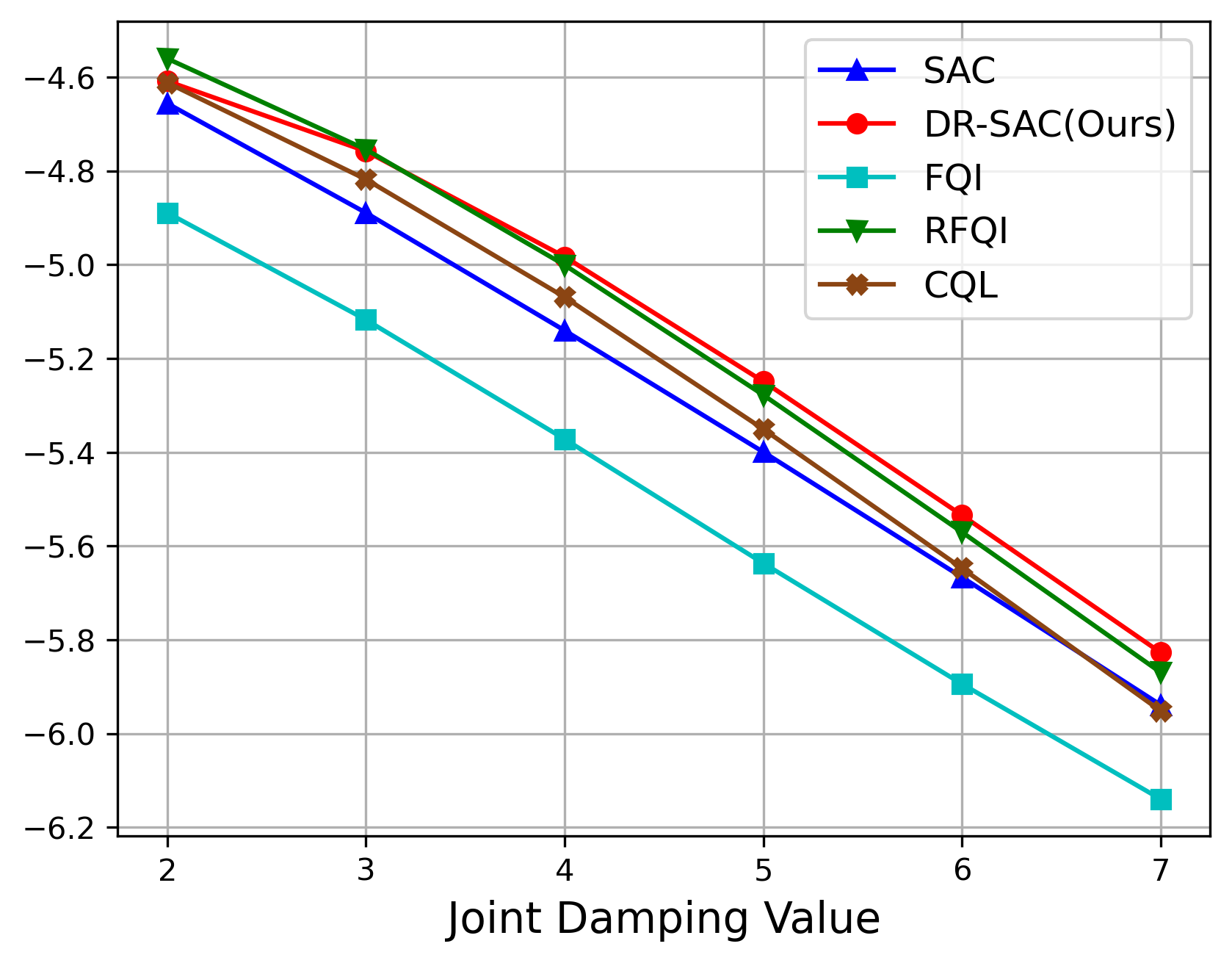}%
  } \hfill
  \subfloat[\centering Back Damping Perturbation\\ \hc]
  {%
    \includegraphics[width=0.33\textwidth,valign=t]{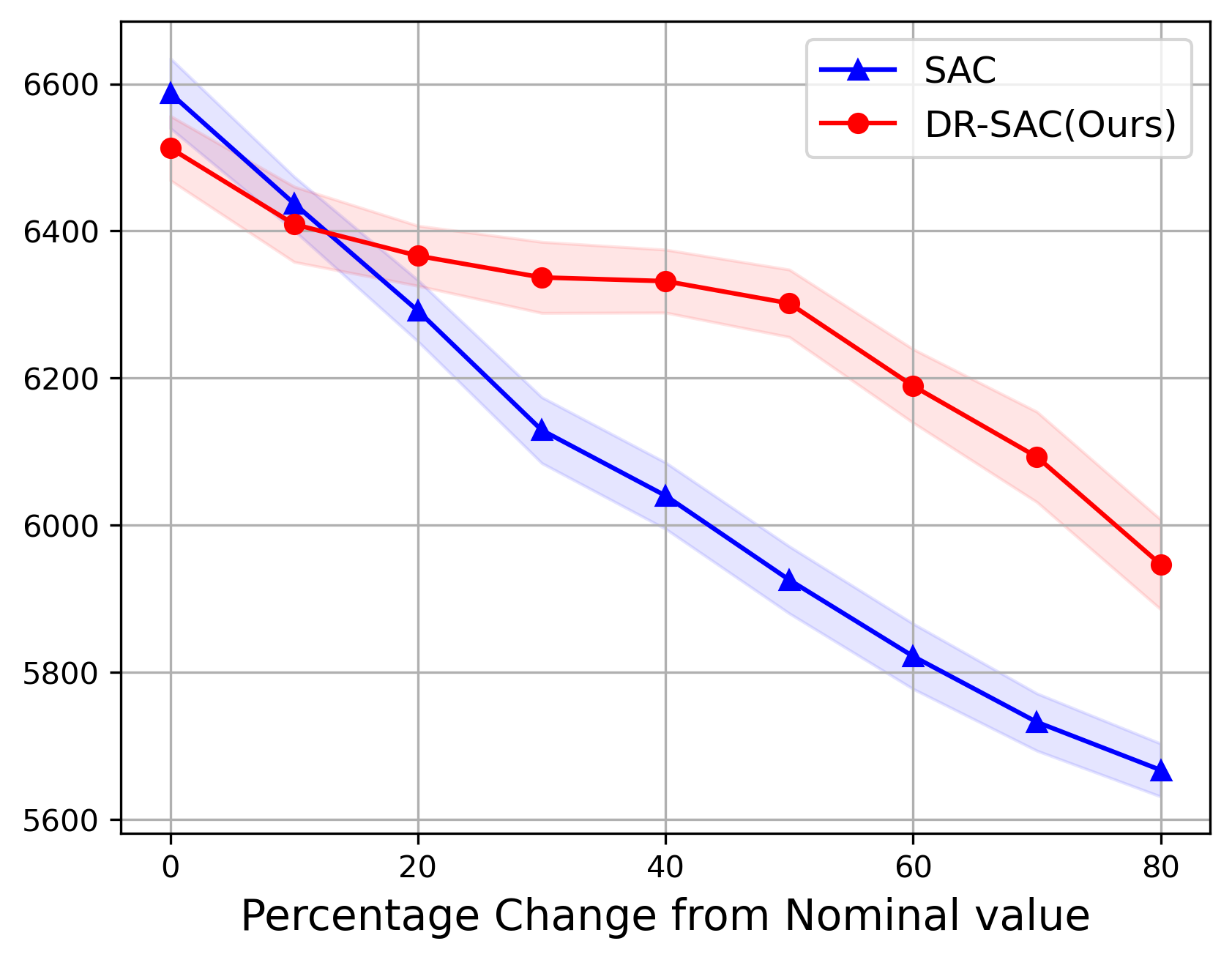}%
  }
\caption{\textbf{Robustness performance comparison.} 
The curves show the average reward over 50 episodes, with shaded regions indicating $\pm 0.5$ standard deviation (see Table~\ref{tabel:reacher std} for standard deviations in sub-figure~(e)). Environmental perturbations include parameter shifts, state and actuator noise.}

%The curves show the average reward over 50 episodes, shaded by $\pm$0.5 standard deviation (Figure (e) see Table~\ref{tabel:reacher std}). In \textit{Pendulum}, the environment parameter \textit{length} changes. In \textit{Cartpole}, random actions are taken by the actuator. In \textit{LunarLander}, the environment parameters \textit{main\_engine\_power} and \textit{side\_engine\_power} change together. In \textit{Reacher}, a Gaussian noise is added to nominal states; and the environment parameter \textit{joint\_damping} changes. In \textit{HalfCheetah}, environment parameter \textit{back\_damping} changes.}
\label{fig:performance}
\end{figure}

This section reports selected experiment results. Additional experiments are provided in Appendix~\ref{sec:extra_exp}.
In the \pvs environment, we evaluate robustness against parameter $\textit{length}$ perturbations. RFQI is omitted due to poor performance in the nominal environment. In Figure~\ref{fig:performance}(a), DR-SAC performance outperforms SAC by $35\%$ when the length changes by 20\%. 
In the \cps environment, the actuator is perturbed by taking random actions with different probabilities. DR-SAC consistently outperforms RFQI, especially when the probability of random action is below $50\%$. 
In the \llds environment, we jointly perturb \textit{main\_engine\_power} and \textit{side\_engine\_power} to model engine power disturbance. DR-SAC shows consistently robust performance compared to other algorithms. 
In Figure~\ref{fig:performance}(c), under $20\%$ perturbation, DR-SAC achieves an average reward of $240$ while the rewards of all other algorithms drop under $180$. Moreover, DR-SAC achieves $9.8\times$ higher reward than SAC when parameters reduce by $30\%$.

To further evaluate robustness in more complex environments, we conduct experiments in \hcs and \rvs from MuJoCo~\citep{todorov2012mujoco}. In the \rvs environment, we consider Gaussian observation noise and parameter \textit{joint\_damping} perturbation. In Figure~\ref{fig:performance}(d), DR-SAC shows the best performance across all observation noise levels. In Figure~\ref{fig:performance}(e), DR-SAC clearly outperforms SAC. In the \hcs environment, we present the experiments of SAC and DR-SAC due to the poor performance of FQI and RFQI. When \textit{back\_damping} varies within $50\%$, DR-SAC maintains a stable average reward of over $6300$, while the average reward of SAC keeps decreasing below $5950$.

\paragraph{Discussion on FQI Failure.}
It is worth noting that FQI and RFQI perform poorly even in unperturbed \pvs and \hcs environments. One possible reason is that offline RL algorithms exhibit different sensitivities to dataset distributions. SAC works well when the dataset provides a broad coverage over the action space~\citep{kumar2019bear}. In contrast, FQI is implemented on Batch-Constrained Deep Q-learning (BCQ;~\citet{fujimoto2019bcq}), which restricts the agent to selecting actions close to the behavior policy. This conflicts with the epsilon-greedy-method data generation process in our experiments. One primary goal of our experiments is to demonstrate that DR-SAC improves robustness over SAC under common environmental perturbations. Investigating the sensitivity of offline RL algorithms to dataset distribution is out of the scope of this work.

\subsection{Ablation Studies}
\label{sec:ablation}

To better understand the design choices in DR-SAC, we conduct a series of ablation studies focusing on computational efficiency and generative model selection. Specifically, we examine (i) the impact of functional approximation on robustness and training time, (ii) the optimization efficiency compared to RFQI, and (iii) the sensitivity of DR-SAC to different generative modeling choices. These studies aim to validate that the proposed design achieves a favorable trade-off between robustness and efficiency.

\subsubsection{Training Efficiency of DR-SAC.}

\paragraph{Comparison with Accurate Bellman Operator}
In Section \ref{section: func approx}, we approximate the Bellman operator $\tpi_\d$ with $\tpi_{\d,g}$ to avoid solving optimization problems for each $(s,a)$ pair. To evaluate the impact of this approximation, we additionally implement an algorithm using the accurate operator. We refer to this variant as \textit{DR-SAC-Accurate} and denote Algorithm \ref{alg:DR-SAC} as \textit{DR-SAC-Functional} in this section. As shown in Figure~\ref{fig:pv acc}, \textit{DR-SAC-Functional} achieves comparable and even better robustness performance while requiring less than $2\%$ training time. These results validate that functional approximation significantly improves computational efficiency without sacrificing robustness. Variant algorithm details and training time are provided in Appendix~\ref{sec:ablation efficiency}.

\begin{figure}[h]
  \centering
  \subfloat[\centering Mass Perturbation. \pv]
  {%
    \includegraphics[width=0.45\textwidth,valign=t]{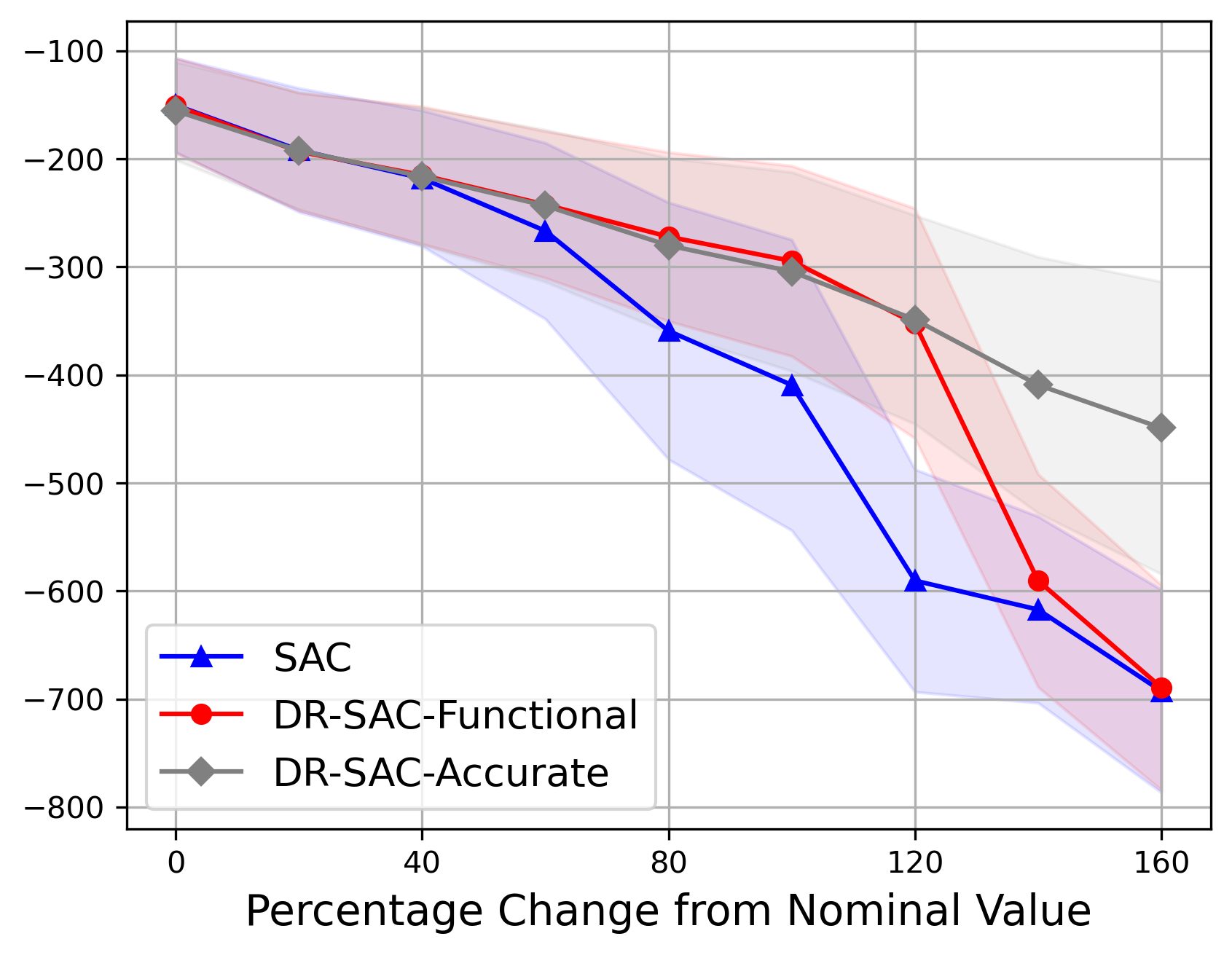}%
  } \hfill
  \subfloat[\centering Length Perturbation. \pv]
  {%
    \includegraphics[width=0.45\textwidth,valign=t]{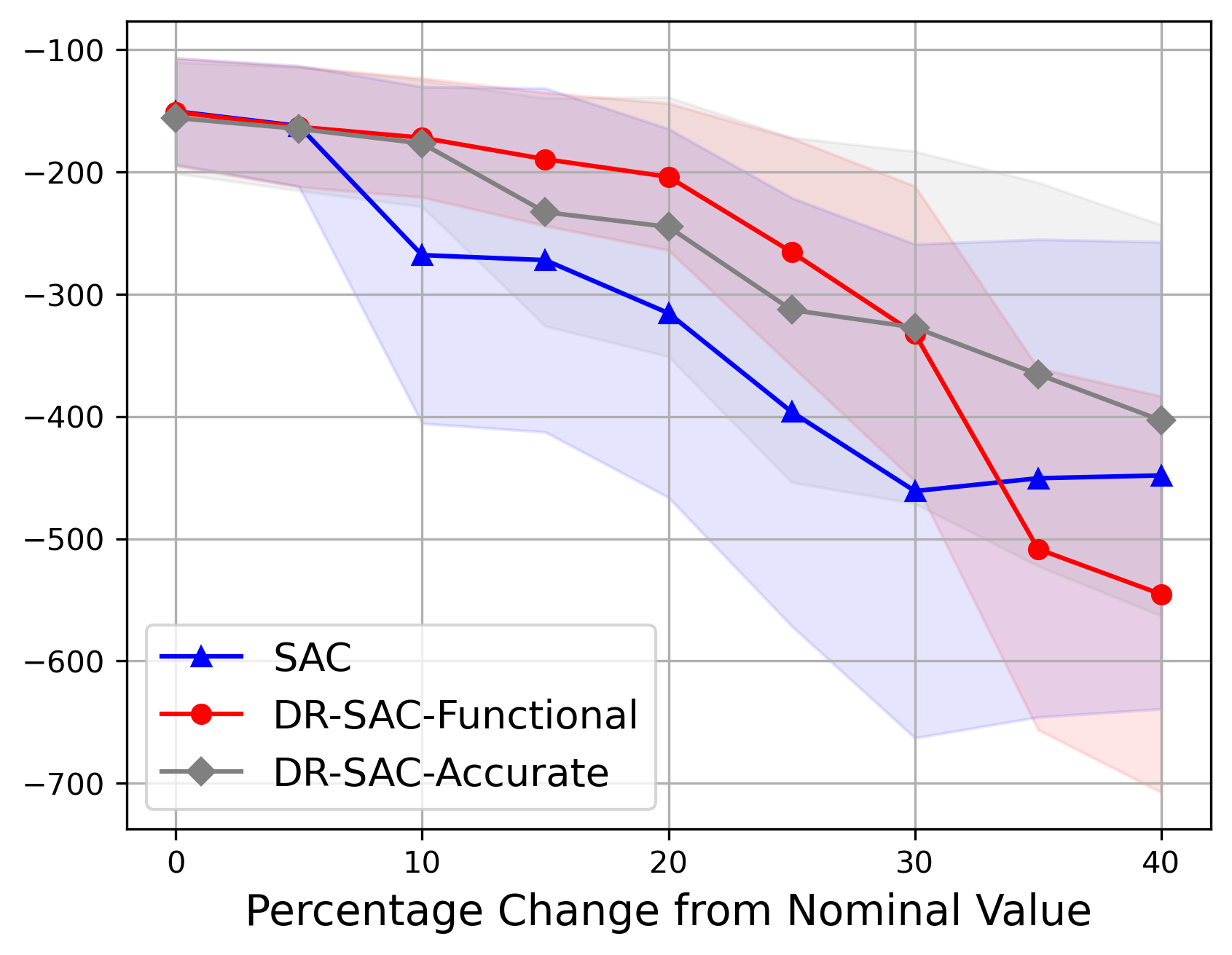}%
  }
\caption{\textbf{Efficiency-Robustness Trade-Off.} These figures show that DR-SAC with functional approximation maintains robustness.}
\label{fig:pv acc}
\end{figure}

\paragraph{Comparison with RFQI}
In Section \ref{sec:experiment performance}, RFQI achieves comparable robustness to DR-SAC in certain environments. However, DR-SAC demonstrates substantial improvement in training efficiency. Table~\ref{table:training time fqi} shows that RFQI requires up to  $23.2\times$ the training time of DR-SAC. Compared with their non-robust counterparts, RFQI requires at least $11.3\times$ the training time of FQI, while DR-SAC training time is at most $2.6\times$ that of SAC.

Further analysis suggests that this efficiency gap primarily stems from optimization complexity. Although RFQI involves a similar functional approximation step as Equation~\eqref{eq: empirical opt g}, it requires 1000 gradient descent (GD) steps in each update to find the optimal function. In contrast, DR-SAC requires only 5 GD steps to achieve comparable performance. Empirically, reducing the number of GD steps in RFQI leads to a severe performance drop, even in unperturbed environments (results see Appendix~\ref{sec:ablation efficiency}). This indicates that the structure of RFQI's loss function inherently results in slower convergence and more optimization steps.

\begin{table}[h]
\centering
\caption{Training time in different environments (minute)}

\label{table:training time fqi}
\begin{tabular}{ c|c|c|c|c } \hline
\texttt{Env} & \texttt{SAC}  & \texttt{DR-SAC} & \texttt{FQI}  & \texttt{RFQI}  \\ \hline
\cp & 2  & 4 & 7  & 93  \\ %\hline
\lld & 16  & 36 & 17  & 238  \\ %\hline
\rv & 13  & 32 & 14  & 159  \\ \hline
\end{tabular}
\end{table}

 \vspace{-0.5em}
\subsubsection{Selection of Generative Model}
Although the VAE models inevitably introduce estimation error when constructing empirical transition measures, we find that DR-SAC is largely insensitive to such modeling choices. As shown in Figure~\ref{fig:gen_type}(a), varying the VAE latent dimension between 5 and 20 in \pvs does not degrade robustness, and DR-SAC consistently outperforms the SAC baseline.

To demonstrate the choice of VAE over other generative models, we implemented diffusion models and normalizing flows as alternative generative models in DR-SAC. Ablation studies in Figure~\ref{fig:gen_type}(b) and (c) show that flow-based models exhibit unstable performance, even in the unperturbed \pvs environment. The diffusion-based model achieves comparable robustness but requires at least $4.5\times$ the training time of the VAE-based model.
We emphasize that the VAE is not necessarily the optimal choice in all settings. Our DR-SAC algorithm can incorporate alternative transition models depending on the task and computational constraints.

\vspace{-1em}
\begin{figure}[h]
  \centering
 \subfloat[\centering Mass Perturbation. \pv]
  {%
    \includegraphics[width=0.329\textwidth,valign=t]{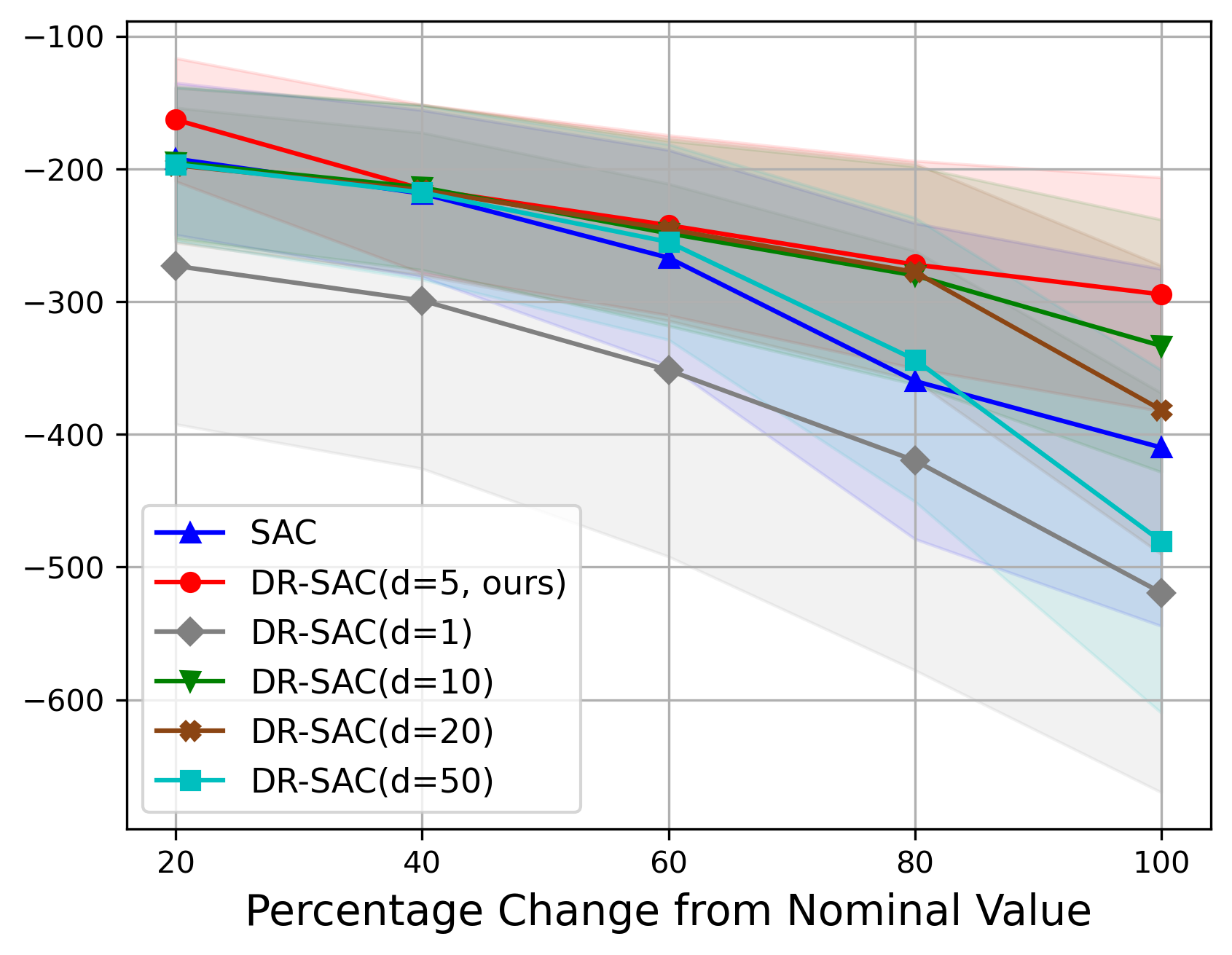}%
  } \hfill
  \subfloat[\centering Length Perturbation. \pv]
  {%
    \includegraphics[width=0.329\textwidth,valign=t]{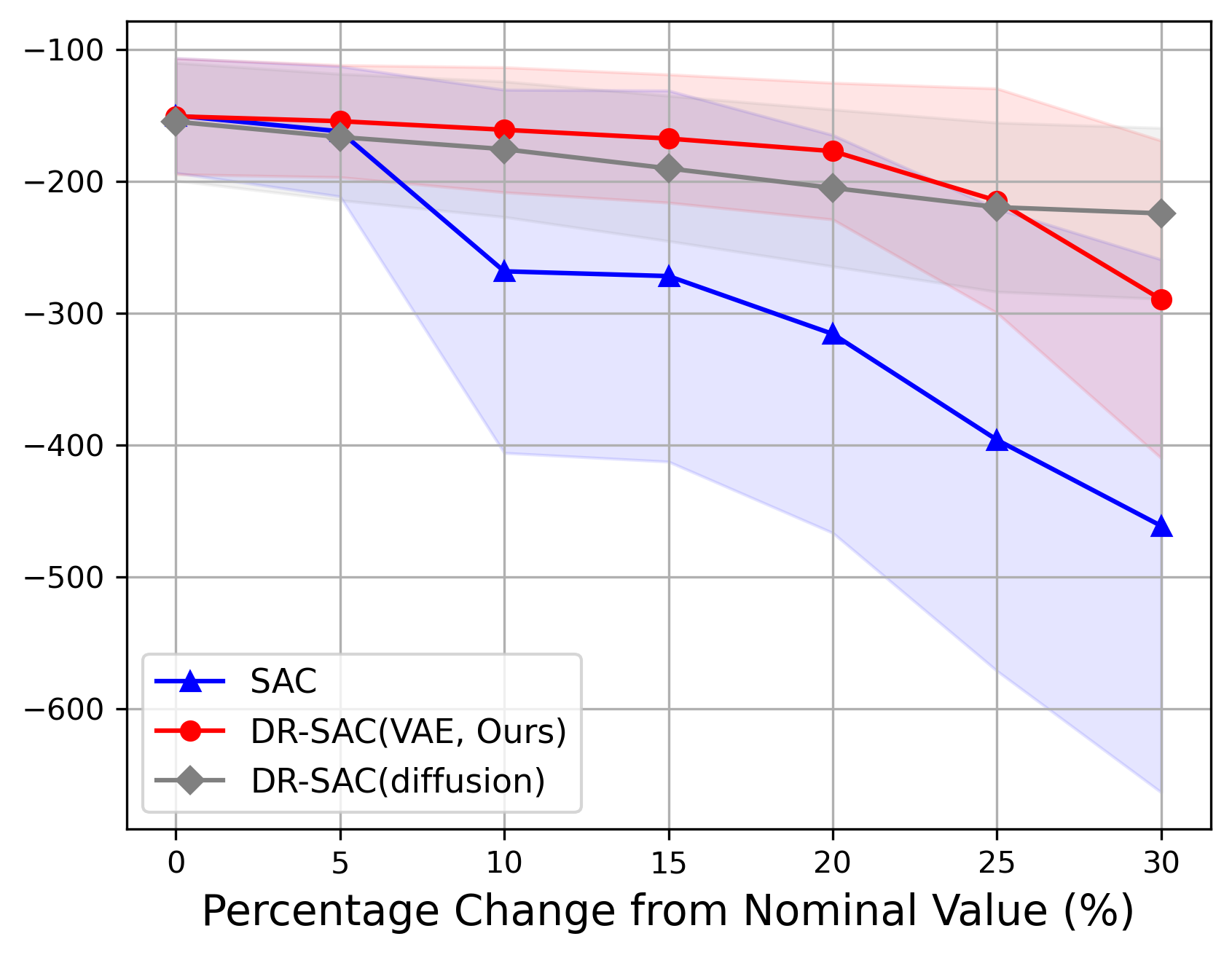}%
  } \hfill
  \subfloat[\centering Observation Noise. \cp]
  {%
    \includegraphics[width=0.329\textwidth,valign=t]{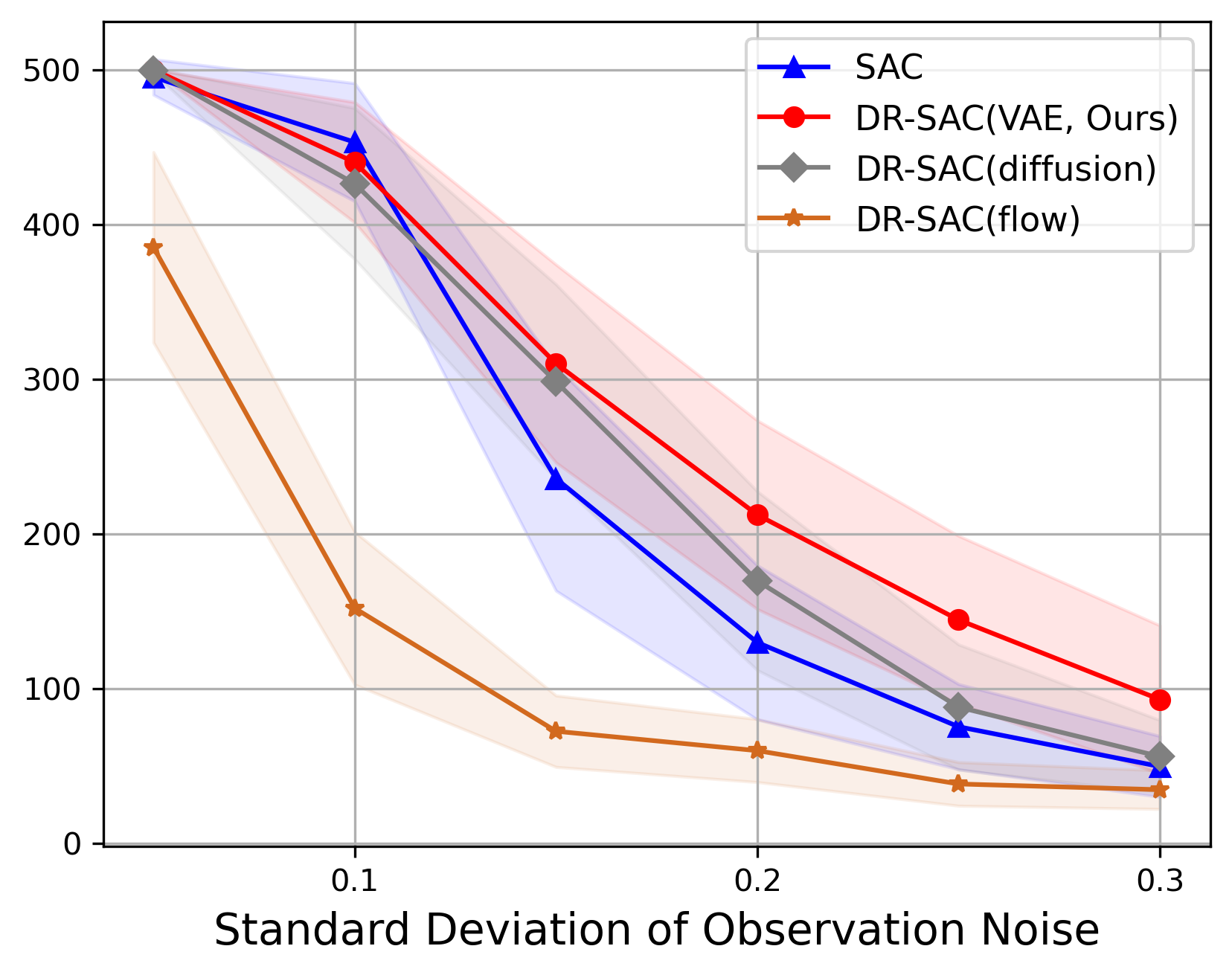}%
  }
\caption{\textbf{Generative Model Selection.} Subfigure (a) shows that VAE-based DR-SAC is insensitive to the latent dimension. Subfigures (b) and (c) compare robustness performance with diffusion and flow-based models.}
\label{fig:gen_type}
\end{figure}

% \vspace{-0.5em}
% \subsubsection{Usage of V-Network.}
% In the DR-SAC algorithm, we include a $V$-network following the SAC-v1 design \citep{haarnoja2018softv1} to improve the applicability across a wider range of offline datasets. Although the $V$-network is removed in SAC-v2 \citep{haarnoja2018softv2}, this version is indeed on-policy, while our setting is off-policy. We observe empirically that SAC with a $V$-network is less sensitive to the behavior policy used in dataset generation. Details are discussed in Appendix~\ref{sec:ablation use v}.
%\section{Related Works}

%\duo{these are future works}

% Future Direction
% sample complexity
% soft bellman iteration in inverse RL
% DR-DSAC 
\section{Conclusions}
\label{sec:conclusion}
\vspace{-0.5em}
We propose DR-SAC, the first actor-critic-based DR-RL algorithm for offline learning in continuous action spaces. Our framework establishes distributionally robust soft policy iteration with convergence guarantees, reduces training time by over 80.0\% compared to RFQI through functional optimization, and resolves the double-sampling issue in estimating nominal distributions via generative modeling. Experimental results show that DR-SAC attains up to $9.8\times$ higher reward than SAC under perturbations, highlighting both robustness and efficiency in practical offline RL tasks.

\newpage

\paragraph{Acknowledgement.} Grani A. Hanasusanto is supported in part by NSF (CCF2343869 and ECCS-2404413). Huan Zhang is supported in part by the AI2050 program at Schmidt Sciences (AI2050 Early Career
Fellowship) and NSF (IIS-2331967). Zhengyuan Zhou is supported in part by NSF (CCF-2312205, ECCS-2419564), ONR-13983263, and the 2027 New York University Center for Global Economy and Business grant. We thank Yunfan Zhang for valuable discussions and the anonymous reviewers for constructive feedback.

\paragraph{Ethics Statement.} All authors of this submission have read and adhered to the ICLR Code of Ethics.

\paragraph{Reproducibility Statement.} We provide our code with detailed comments in the supplementary materials. The detailed experiment settings, dataset processing steps and the devices used in our experiments are provided in Appendix~\ref{sec:exp_details} to ensure reproducibility.

\paragraph{The Use of Large Language Models.} The authors use Large Language Models (LLMs) to assist with grammar checking and language polishing in this submission. LLMs do not play a significant role in research ideation or writing to the extent that they could be regarded as a contributor.

% \paragraph{Limitation}
% The limitation of this work is that we only consider the uncertainty set defined by KL divergence. In future work, we will explore theoretical guarantees and algorithm design when transition functions are constrained by other divergences, e.g., $\chi^2$ divergence, total variation divergence. 

% \paragraph{Broader Impacts}
% Deep RL algorithms may fail due to environmental perturbations. DR-SAC shows more robust performance and are more applicable to real-life applications.

\newpage
\bibliography{reference}
\bibliographystyle{iclr2026_conference}

%%%%%%%%%%%%%%%%%%%%%%%%%%%%%%%%%%%%%%%%%%%%%%%%%%%%%%%%%%%%

\appendix
\appendix
\onecolumn

\textbf{\Large Appendix}

\tableofcontents

\addtocontents{toc}{\protect\setcounter{tocdepth}{2}}
\newpage
\section{Discussion}

\subsection{Necessity of Generative Model}
\label{sec:discuss vae}
In this section, we explain why model-free empirical risk minimization (ERM) is not applicable under a KL-divergence-constrained uncertainty set and why a generative model (VAE) is necessary.

In offline RL, the nominal transition distributions $\cP^0$ are unknown, and no simulator is available to generate additional samples. Under KL-based uncertainty sets, the dual formulation of the DR soft Bellman operator is nonlinear in the transition distribution. As a result, naive empirical estimation from the dataset $\cD$ leads to the well-known \textit{double-sampling issue}~\citep{baird1995residual}, caused by the nested structure of inner and outer expectations. In our algorithm, we want to apply operator $\tpi_{\delta,g^\star}$ where $g^\star = \argsup_{g\in\G}\E_{(s,a)\sim\cD}\Big[f\big((s,a), g(s,a)\big)\Big]$. Define the objective function as 
\begin{equation}\begin{aligned}
J(g) :=& \E_{(s,a)\sim\cD}\Big[f\big((s,a), g(s,a)\big)\Big]\\
=& \E_{(s,a)\sim\cD}\left[-g(s,a)\log\left(\E_{p_{s,a}^0 }\left[\exp\left(\frac{-V(s')}{g(s,a)}\right)\right]\right) - g(s,a)\d\right].
\label{eq:func opt problem population}
\end{aligned}\end{equation}
% We rewrite the offline dataset as $\cD=\{(s_i,\,a_i,\,\{s'_{ij}\}_{j=1}^{m_i})\}_{i=1}^N$, where $s'_{ij}\sim p_{s_i,a_i}^0$ are the potential next states transited from $(s_i,a_i)$. 

To construct a consistent empirical estimator of $J(g)$, we would need independent samples to approximate both the outer expectation over $(s,a)$ and the inner expectation over $s'\sim p_{s,a}^0$. This requires splitting the dataset into disjoint subsets $\mathcal{D}_{\text{outer}}$ and $\mathcal{D}_{\text{inner}}$. For each $(s,a) \in \mathcal{D}_{\text{inner}}$, we aggregate the corresponding samples starting from $(s,a)$ contained in $\mathcal{D}_{\text{outer}}$, denoted by $\mathcal{D}_{(s,a)}$, and the empirical risk of objective $J(g)$ becomes
{\small
\begin{equation}\begin{aligned}
\widehat{J}(g) :=& \frac1{\lvert \mathcal{D}_{\text{out}}\rvert}\sum_{(s,a,s') \in  \mathcal{D}_{\text{out}}}\left[-g(s,a)\log\left(\frac{1}{\lvert \mathcal{D}_{(s,a)}\rvert}\sum_{(\bar{s},\bar{a},\bar{s}') \in \mathcal{D}_{(s,a)}}\exp\left(\frac{-V(\bar{s}')}{g(s,a)}\right)\right) - g(s,a)\d\right] .
\label{eq:func opt problem erm}
\end{aligned}\end{equation}
}

However, in continuous state and action spaces, it is nearly impossible to revisit the exact same state–action pair, leading to empty conditional sample sets $\mathcal{D}_{(s,a)} = \emptyset$ and making ERM infeasible.

% If we assume $m_i=1$ for all $i\in[N]$, equation \eqref{eq:func opt problem erm} simplifies to
% \begin{equation}\begin{aligned}
% \widehat{J}(g) =& \frac1N\sum_{i=1}^N\left[-g(s_i,a_i)\log\left(\exp\left(\frac{-V(s_{i}')}{g(s_i,a_i)}\right)\right) - g(s_i,a_i)\d\right] \\
% =& \frac1N\sum_{i=1}^N\left[V(s_{i}') - g(s_i,a_i)\d\right].
% \end{aligned}\end{equation}

% \yx{
% the limit of~\eqref{eq-ERM} as $N\to +\infty$ is given by\begin{align*}
%     \E_{(s,a) \sim D, s' \sim p^0(\cdot \lvert s,a)} \bigg[-g(s,a) \log\bigg[\exp(\frac{-V(s')}{g(s,a)})\bigg]- g(s,a)\delta \bigg].
% \end{align*}
% A consistent approximation to the~\eqref{eq-population-obj} should be \begin{align*}
% \frac1{N_1}\sum_{i=1}^{N_1}\left[-g(s_i,a_i)\log\left(\exp\left(\frac{-V(s_i')}{g(s_i,a_i)}\right)\right) - g(s_i,a_i)\d\right]
% \end{align*}
% }

% Since all functions are non-negative in the function set $\G$, the optimal dual function $\hat{g}^\star$ under this empirical objective function is the zero function, and operator $\tpi_{\delta,\hat{g}^\star}$ will become non-robust operator $\tpi$.

Importantly, this difficulty arises from the nonlinear structure of the KL-based dual formulation rather than from our functional optimization approach. Even if we abandon the functional approximation and use the exact dual Bellman operator, the same double-sampling issue persists due to the nested expectation:
\begin{equation}\begin{aligned}
\mathcal{L}_Q :=& \E_{(s,a)\in\cD}\left[Q(s,a) - \tpi_\d Q(s,a)\right] \\
=&  \E_{(s,a)\in\cD}\left[Q(s,a) - \E[r] 
        -\gamma\cdot \sup_{\beta\ge0}\left\{-\beta\log\left(\E_{p_{s,a}^0 }\left[\exp\left(\frac{-V(s')}{\beta}\right)\right]\right) - \beta\d\right\}\right].
\label{eq: beta opt problem population}
\end{aligned}\end{equation}
In contrast, under TV-based uncertainty sets, the dual formulation is linear in the transition distribution~\citep{RFQI}, which avoids this nested expectation and therefore does not suffer from the same issue.

% If we keep the assumption that $m_i=1,\,\forall i \in[N]$, the empirical risk approximation of equation \eqref{eq: beta opt problem population} simplifies to the non-robust case again:
% \begin{equation}\begin{aligned}
% \widehat{\mathcal{L}}_Q :=& \frac1N\sum_{i=1}^N\left[Q(s_i,a_i) - r_i 
%         -\gamma\cdot \sup_{\beta>0}\left\{-\beta\log\left(\exp\left(\frac{-V(s_i')}{\beta}\right)\right) - \beta\d\right\}\right]\\
%         =&\frac1N\sum_{i=1}^N\left[Q(s_i,a_i) - r_i 
%         -\gamma\cdot \sup_{\beta>0}\left\{V(s_i') - \beta\d\right\}\right] \\
%         =& \frac1N\sum_{i=1}^N\Big[Q(s_i,a_i) - r_i 
%         -\gamma\cdot V(s_i') \Big] .
% \label{eq: beta opt problem erm}
% \end{aligned}\end{equation}
Existing KL-based distributionally robust RL algorithms overcome this difficulty by using a Monte-Carlo rollouts~\citep{liu2022distributionally, Bound-DR-Qlearning}, estimation from transition frequency~\citep{Bound-DR-VR-Qlearning}, or direct estimation of nominal expectations~\citep{single-DR-Qlearning}. However, these approaches rely on either simulators or repeated visitation of state–action pairs and are therefore not applicable in continuous offline RL settings.

\subsection{Algorithm Details}
\label{sec:algo detail}
In this section, we present a detailed description of the DR-SAC algorithm. In our algorithm, we use neural networks $V_\psi(s)$, $Q_\theta(s,a)$ and $\pi_\phi(a\mid s)$ to approximate the value function, the $Q$-function and the stochastic policy, respectively, with $\psi, \theta, \phi$ as the network parameters. We also utilize target network $V_{\bar{\psi}}(s)$ and $Q_{\bar{\theta}}(s,a)$, where parameters $\bar{\psi}$ and $\bar{\theta}$ are the exponential moving average of respective network weights. Similar to SAC-v1 algorithm \citep{haarnoja2018softv1}, the loss function of $V$-network is
\begin{equation}
    J_V(\psi) = \E_{s\sim\cD} \left[\frac12\left(V_\psi(s) - \E_{a\sim\pi_\phi}\left[Q_{\bar{\theta}}(s,a) - \a\log\pi_\phi(a\mid s)\right]\right)^2\right].
\end{equation}
As introduced in Section~\ref{sec:algorithm}, in our algorithm, we modify the loss function of $Q$-network to 
$$
J^{\text{DR}}_Q(\theta) = \E_{(s,a)\sim\cD}\left[\frac12\left(Q_\theta(s,a) - \tpi_{\d, \widetilde{g}^\star}Q_\theta(s,a)\right)^2\right],
$$
where 
\begin{equation}\begin{aligned}
 {\widetilde{g}}^\star &= \argsup_{g\in\G_\eta} \E_{(s,a)\in\cD}\left[\widetilde{f}((s,a),g(s,a))\right]\\
 &= \argsup_{g\in\G_\eta} \E_{(s,a)\in\cD}\left[-g(s,a)\log\left(\E_{\widetilde{p}_{s,a}^0 }\left[\exp\left(\frac{-V_{\bar{\psi}}(s')}{g(s,a)}\right)\right]\right) - g(s,a)\d\right].
\end{aligned}\end{equation}
Optimal dual function ${\widetilde{g}}^\star$ can be found with backpropagation through $\eta$. We also keep the assumption of policy network in the standard SAC algorithm by reparameterizing the policy using a neural network transformation $a = f_\phi(\epsilon; s)$, where $\epsilon$ is an input noise vector sampled from a spherical Gaussian. The loss of policy is 
\begin{equation}
    J_\pi(\phi) = \E_{s\sim\cD, \epsilon\sim\mathcal{N}} \Big[\alpha\log\pi_\phi(f_\phi(\epsilon;s) \mid s) - Q_{\bar{\theta}}(s, f_\phi(\epsilon;s))\Big].
\end{equation}
In the SAC-v2 algorithm \citep{haarnoja2018softv2}, the authors propose an automated entropy temperature adjustment method by using an approximate solution to a constrained optimization problem. The loss of temperature is 
\begin{equation}
J(\alpha) = \E_{a\sim\pi_\phi}\left[-\alpha\log\pi_\phi(a\mid s) - \alpha\bar{\H}\right],
\end{equation}
where $\bar{\H}$ is the desired minimum expected entropy and is usually implemented as the dimensionality of the action space.

In addition, we incorporate generative model into our algorithm. VAE is one of the most popular methods to learn complex distributions and has shown superior performance in generating different types of data. In the DR-SAC algorithm, we use VAE to learn the transition function $P^0(s' \mid s, a)$ by modeling the conditional distribution of next states. It assumes a standard normal prior over the latent variable, $ p(z) = \mathcal{N}(0, I) $. The encoder maps $(s, a, s')$ to an approximate posterior $ q(z \mid s,a,s') $, and the decoder reconstructs $s'$ from the latent sample $z$ and input $(s,a)$. The training loss is the evidence lower bound (ELBO):
\begin{equation}
J_{\text{VAE}}(\varphi) = \mathbb{E}_{q(z \mid s,a,s')}\left[ \Vert s' - \hat{s}' \Vert^2 \right] + D_{\text{KL}}\left( q(z \mid s,a,s') \,\middle\Vert\, \mathcal{N}(0, I) \right),
\end{equation}
where $ \hat{s}' $ are the reconstructed states from the decoder.

\newpage
\section{Proofs}
\label{sec:proof}
\subsection{Proof of Proposition \ref{prop: q_bellman dual}}
\label{prof: prop q_bellman dual}
We first provide an established result in DRO to compute the worst-case expectation under perturbation in a KL-divergence constrained uncertainty set.

\begin{lemma}[\cite{hu2013KL}, Theorem 1] \label{lemma: KL dual}
    Suppose $G(X)$ has a finite moment generating function in the neighborhood of zero. Then for any $\delta>0$,
    \begin{equation}
        \sup_{P:\DKL(P\|P_0)\leq\delta}\mathbb{E}_P[G(X)]
        =\inf_{\beta\geq0}\left\{\beta\log\left(\mathbb{E}_{P_0}\left[\exp\left(\frac{G(X)}{\beta}\right)\right]\right) + \beta\delta\right\}
    \label{eq: KL dual}
    \end{equation}

\end{lemma}

\begin{proof}[Proof of Proposition \ref{prop: q_bellman dual}]
\begin{align*}
\mathcal{T}_\pi^\delta Q(s,a) &= \mathbb{E}[r] + \gamma \cdot \inf_{p \in \mathcal{P}_{s,a}(\delta)} \left\{ \mathbb{E}_{s' \sim p(\cdot|s,a)} \left[ \mathbb{E}_{a' \sim \pi(\cdot|s')} [ Q(s',a') - \alpha \log \pi(a'|s') ] \right] \right\}\\
    %&= \mathbb{E}[r(s,a)] + \gamma \cdot \inf_{p \in \mathcal{P}_{s,a}(\delta)} \left\{ \mathbb{E}_{s' \sim p(\cdot|s,a)} [V^\pi(s')] \right\} \\
    &= \mathbb{E}[r] - \gamma \cdot \sup_{p \in \mathcal{P}_{s,a}(\delta)} \left\{ \mathbb{E}_{s' \sim p(\cdot|s,a)} [-V(s')] \right\} \\
    &= \mathbb{E}[r] - \gamma \cdot \inf_{\beta\geq0}\left\{\beta\log\left(\mathbb{E}_{s' \sim p^0(\cdot|s,a)}\left[\exp\left(\frac{-V(s')}{\beta}\right)\right]\right) + \beta\delta \right\} \tag{Lemma~\ref{lemma: KL dual}}\\
    &= \mathbb{E}[r] + \gamma \cdot \sup_{\beta\geq0}\left\{-\beta\log\left(\mathbb{E}_{s' \sim p^0(\cdot|s,a)}\left[\exp\left(\frac{-V(s')}{\beta}\right)\right]\right) - \beta\delta\right\}
\end{align*}
To apply Lemma~\ref{lemma: KL dual}, let $P = p(\cdot|s,a)$, $P_0 = p^0(\cdot|s,a)$, and $G(X) = G(s') = -V(s')$. We assume rewards $r = R(s,a)$ are bounded, and the discount factor $\gamma \in [0,1)$. 
% As a consequence of Assumption~\ref{assump: finite action space}, $Q(s,a)$ and thus $V(s')$ are bounded. \yf{Specifically, they are bounded by ...} 
Since $Q(s,a)$ is bounded, we know $V(s')$ is bounded as well. This implies that $G(s') = -V(s')$ has a finite moment generating function (MGF) under the nominal distribution $p^0(\cdot|s,a)$, i.e., $\mathbb{E}_{s' \sim p^0(\cdot|s,a)}[e^{\lambda G(s')}] < \infty$, for $\lambda$ in a neighborhood of zero. This ensures that $G(s')$ has a finite MGF under $P_0$ as required by Lemma~\ref {lemma: KL dual}.
\end{proof}

\subsection{Proof of Proposition \ref{prop: soft policy evaluation}}
\label{prof: prop policy evaluation}

Before providing the proof of Proposition \ref{prop: soft policy evaluation}, we present the optimality conditions of Lemma \ref{lemma: KL dual}.
\begin{lemma}[\cite{hu2013KL}, Proposition 2]
Let $\beta^\star$ be an optimal solution of the optimization problem in Equation~\eqref{eq: KL dual}. Let $H = \esssup_{X\sim P_0}G(X)$ and $\kappa = \P_{X\sim P_0}(G(X) = H)$. Suppose the assumption in Lemma \ref{lemma: KL dual} still holds, then $\beta^\star=0$ or $G(X)$ has a finite moment generating function at $1/\beta^\star$. Moreover, $\beta^\star=0$ if and only if $H<\infty$, $\kappa>0$ and $\log\kappa+\delta\ge0$.
\label{lemma: KL opt condition}
\end{lemma}
This lemma tells us the optimal solution is unique when $\beta^\star=0$. This happens if and only if there is a large enough probability mass on the finite essential supremum of $X$, under the distribution center $P_0$. We use this lemma to discuss either $\beta^\star=0$ or $\beta^\star>0$ in the following proof.

\begin{proof}[Proof of Proposition \ref{prop: soft policy evaluation}]
Similar to the standard convergence proof of policy evaluation, we want to prove that the operator $\tpi_\delta$ is a $\gamma$-contraction mapping. Suppose there are two mappings $Q_{1,2}:\S\times\A\to\R$ and define $V_i(s) = \E_{a\sim\pi}[Q_i(s,a)] - \alpha \H(\pi(s)),\,i=1,2$. For any state $s\in\S$, we have 
$$
\lvert V_1(s) - V_2(s)\rvert
= \lvert\E_{a\sim\pi}[Q_1(s,a) - Q_2(s,a)]\rvert
\le \lVert Q_1  - Q_2\rVert_\infty.
$$
Thus, $\lVert V_1  - V_2\rVert_\infty\le\lVert Q_1  - Q_2\rVert_\infty$. 

Next, for any $\beta>0$ and $(s,a)$ fixed, define function 
\begin{equation}
F_\beta(V) := -\beta\log\E_{p_{s,a}^0}\left[\exp\left(-\frac{V(s')}{\beta}\right)\right] - \beta\delta.
\end{equation}
Let $\lVert V_1  - V_2\rVert_\infty = d$. Then for any $s'\in\S$, $V_2(s') - d \le V_1(s') \le V_2(s') + d$. After exponential, expectation, and logarithm operations, monotonicity is preserved. We have
$$\begin{aligned}
-\beta\log\E_{p_{s,a}^0}\left[\exp\left(-\frac{V_2(s')}{\beta}\right)\right] - d
\le&
-\beta\log\E_{p_{s,a}^0}\left[\exp\left(-\frac{V_1(s')}{\beta}\right)\right]\\
\le& 
-\beta\log\E_{p_{s,a}^0}\left[\exp\left(-\frac{V_2(s')}{\beta}\right)\right] + d.
\end{aligned}$$
This gives us $\lvert F_\beta(V_1)  - F_\beta(V_2)\rvert\le\lVert V_1  - V_2\rVert_\infty$.

Lastly, we reformulate DR soft Bellman operator as 
\begin{equation}
\tpi_\delta Q(s,a) = \E[r] + \gamma\cdot\sup_{\beta\ge0}F_\beta(V).
\end{equation}
Let $\beta_i^\star$ be an optimal solution of $\sup_{\beta\ge0}F_\beta(V_i),\,i=1,2$. From Lemma \ref{lemma: KL opt condition}, we know $\beta_i^\star$ is unique when $\beta_i^\star=0$ is optimal. And the optimal value is the essential infimum $H_i$ when $\beta_i^\star=0$. We want to show $\lvert F_{\beta_1^\star}(V_1) - F_{\beta_2^\star}(V_2)\rvert$ is bounded in all cases of $\beta_i^\star$.

\begin{itemize}[leftmargin=0pt,labelsep=0.5em]
\item Case 1: $\beta_1^\star = \beta_2^\star = 0$.

In this case, the optimal value is the essential infimum value for both $V_i$. We have
$$
\lvert F_{\beta_1^\star}(V_1) - F_{\beta_2^\star}(V_2)\rvert
= \left\lvert \essinf_{s'\sim P_{s,a}^0}V_1(s') - \essinf_{s'\sim P_{s,a}^0}V_2(s')\right\rvert
\le\lVert V_1  - V_2\rVert_\infty. \\
$$
The last inequality holds because monotonicity is preserved after taking the essential infimum.

\item Case 2: $\beta_1^\star = 0,\,\beta_2^\star > 0$, WLOG.

In this case, we know from optimality that 
$$
H_1 = \essinf_{s'\sim P_{s,a}^0}V_1(s') \ge F_{\beta_2^\star}(V_1),\,
H_2 = \essinf_{s'\sim P_{s,a}^0}V_2(s') \le F_{\beta_2^\star}(V_2).
$$
Then we have 
$$\begin{aligned}
&H_1 - F_{\beta_2^\star}(V_2) \le H_1 - H_2 \le \lVert V_1  - V_2\rVert_\infty, \\
&F_{\beta_2^\star}(V_2) - H_1 \le F_{\beta_2^\star}(V_2) -F_{\beta_2^\star}(V_1) \le \lVert V_1  - V_2\rVert_\infty.
\end{aligned}$$
Thus, $\lvert F_{\beta_1^\star}(V_1) - F_{\beta_2^\star}(V_2)\rvert
= \lvert H_1 - F_{\beta_2^\star}(V_2)\rvert
\le\lVert V_1  - V_2\rVert_\infty$.

\item Case 3: $\beta_1^\star >0,\,\beta_2^\star > 0$.

Suppose $F_{\beta_1^\star}(V_1) \le F_{\beta_2^\star}(V_2)$, WLOG. Then 
$$
\lvert F_{\beta_1^\star}(V_1) - F_{\beta_2^\star}(V_2)\rvert
= F_{\beta_2^\star}(V_2) - F_{\beta_1^\star}(V_1)
\le F_{\beta_2^\star}(V_2) - F_{\beta_2^\star}(V_1)
\le\lVert V_1  - V_2\rVert_\infty,
$$
where the first inequality comes from the optimality of $\beta_1^\star$. 
\end{itemize}
Thus for any $(s,a)$ pair, we have
we
$$\begin{aligned}
\left\lvert\tpi_\delta Q_1 (s,a)- \tpi_\delta Q_2(s,a) \right\rvert 
=& \gamma\cdot\left\lvert\sup_{\beta_1\ge0}F_{\beta_1}(V_1) - \sup_{\beta_2\ge0}F_{\beta_2}(V_2)\right\rvert\\
\le& \gamma\cdot\lVert V_1  - V_2\rVert_\infty\\
\le& \gamma\cdot\lVert Q_1  - Q_2\rVert_\infty.
\end{aligned}$$
Since $\tpi_\delta$ is a $\gamma$-contraction mapping, the Banach Fixed-Point Theorem implies that the sequence $\{Q^k\}$ convergences to the unique fixed-point of $\tpi_\delta$. From~\citet{iyengar2005robustdp, xu2010robustmdp},this fixed point corresponds to the distributionally robust soft $Q$-value.
% \yf{Can add a lemma to show this. We can add entropy into the reward and use the existing results.}
\end{proof}

\subsection{Proof of Proposition \ref{prop: soft policy improvement}}
\label{prof: prop policy improvement}
%The proof of Proposition \ref{prop: soft policy improvement} is similar to the proof of Lemma 2 in \cite{haarnoja2018softv1}, except that it replaces the soft Bellman equation with the DR version. We present it here for completeness.

\begin{proof}
Given $\pi_k\in\Pi$, let $Q^{\pi_k}_{\Mdelta}$ and $V^{\pi_k}_{\Mdelta}$ be the corresponding DR soft $Q$-function and value function. Denote the function for determining the new policy as
\begin{equation}
J_\pi(\pi'(\cdot\mid s)) := D_{\text{KL}} \left(\pi'(\cdot\mid s)) 
\,\middle\Vert\,
\exp\left(\frac{1}{\alpha}Q^{\pi}_{\Mdelta}(s,\cdot) - \log Z^{\pi_k}(s)\right)\right).
\end{equation}
According to Equation \eqref{eq: dr policy improvement}, $\pi_{k+1} = \argmin_{\pi'\in\Pi}J_{\pi_k}(\pi')$ and $J_{\pi_k}(\pi_{k+1})\le J_{\pi_k}(\pi_k)$. Hence 
$$\begin{aligned}
&\E_{a\sim\pi_{k+1}}[\alpha\log\pi_{k+1}(a\mid s) - Q^{\pi_k}_{\Mdelta}(s,\cdot) + \alpha\log Z^{\pi_k}(s)] \\
\le&
\E_{a\sim\pi_k}[\alpha\log\pi_k+(a\mid s) - Q^{\pi_k}_{\Mdelta}(s,\cdot) + \alpha\log Z^{\pi_k}(s)],
\end{aligned}$$
and after deleting $Z^{\pi_k}(s)$ on both sides, the inequality is reformulated to 
$$
\E_{a\sim\pi_{k+1}}\left[Q^{\pi_k}_{\Mdelta}(s,\cdot) - \alpha\log\pi_{k+1}(a\mid s) \right]
\ge V^{\pi_k}_{\Mdelta}(s).
$$
Next, consider the DR soft Bellman equation:
\begin{equation}\begin{aligned}
Q^{\pi_k}_{\Mdelta}(s,a) 
=& \E[r] + \gamma \cdot\inf_{p_{s,a}\in\cP_{s,a}(\delta)}\left\{\E_{s'\sim p_{s,a}}\left[V^{\pi_k}_{\Mdelta}(s')\right]\right\} \\
\le& \E[r] + \gamma \cdot\inf_{p_{s,a}\in\cP_{s,a}(\delta)}\left\{\E_{s'\sim p_{s,a}}\left[\E_{a'\sim\pi_{k+1}}\left[Q^{\pi_k}_{\Mdelta}(s',a') - \alpha\log\pi_{k+1}(a'\mid s')\right]\right]\right\} \\
=& \mathcal{T}^{\pi_{k+1}}_\delta\left(Q^{\pi_k}_{\Mdelta}\right)(s,a)\\
\vdots \\
\le& Q^{\pi_{k+1}}_{\Mdelta}(s,a),\,\forall(s,a)\in\S\times\A
\end{aligned}\end{equation}
where operator $\mathcal{T}^{\pi_{k+1}}_\delta$ is repeatedly applied to $Q^{\pi_k}_{\Mdelta}$ and its convergence is guaranteed by Proposition~\ref{prop: soft policy evaluation}. 
\end{proof}

\subsection{Proof of Theorem \ref{thm: soft policy iteration}}
\label{prof: thm policy iteration}

\begin{proof}
By Proposition~\ref{prop: soft policy improvement}, $Q^{\pi_k}_{\Mdelta}$ is non-decreasing with $k$. Since function $Q^{\pi_k}_{\Mdelta}$ is bounded by $(R_{\text{max}}+\alpha\log\lvert\A\rvert)/(1-\gamma)$, sequence $\{Q^{\pi_k}_{\Mdelta}\}$ converges. Thus policy sequence $\{\pi_k\}$ convergences to some $\pi^\star$. It remains to show that $\pi^\star$ is indeed optimal. According to Equation \eqref{eq: dr policy improvement}, $J_{\pi^\star}(\pi^\star)\le J_{\pi^\star}(\pi),\,\forall \pi\in\Pi$. Using the same argument in proof of Proposition~\ref{prop: soft policy improvement}, we can show that $Q^\pi_{\Mdelta}(s,a)\le Q^{\pi^\star}_{\Mdelta}(s,a)$ for any $\pi\in\Pi$ and $(s,a)\in\S\times\A$. Hence $\pi^\star$ is an optimal policy.
\end{proof}

\subsection{Proof of Proposition \ref{prop: functional}}
\label{prof: prop functional}
Before providing the proof, we first introduce two technical lemmas. Specifically, Lemma~\ref{lemma: interchange} establishes the \textit{interchange of minimization and integration} property in decomposable spaces. This property has wide applications in replacing point-wise optimality conditions by optimization in a functional space \citep{shapiro2017interchange_app, RFQI}.

\begin{lemma}[\cite{rockafellar2009variational}, Exercise 14.29]
Function $f : \Omega \times \R^n \mapsto \R$ (finite-valued) is a normal integrand if  $f(\omega, x)$ is measurable in $\omega$ for each $x$ and continuous in $x$ for each $\omega$.
\label{lemma: normal integrand}
\end{lemma}

\begin{lemma}[\cite{rockafellar2009variational}, Theorem 14.60, Exercise 14.61]
Let $f : \Omega\times\R\mapsto\R$ (finite-valued) be a normal integrand. Let $\M(\Omega,\A;\R)$ be the space of all measurable functions $x:\Omega\to\R$, $\M_f$ be the collection of all $x \in \M(\Omega,\A; \R)$ with $\int_{\omega\in\Omega}f(\omega,x(\omega))\mu(d\omega)<\infty$. Then, for any space with $\M_f\subset \X \subset \M(\Omega,\A;\R)$, we have   
$$
\inf_{x\in\X}\int_{\omega\in\Omega}f(\omega, x(\omega)) \mu(d\omega)
= \int_{\omega\in\Omega}\left(\inf_{x\in\R}f(\omega,x)\right) \mu(d\omega).
$$
\label{lemma: interchange}
\end{lemma}

\begin{proof}[Proof of Proposition \ref{prop: functional}]
First we want to prove $\beta^\star = \argsup_{\beta\ge0} f((s,a),\beta)$ is bounded in interval $\I_\beta :=\left[0,\frac{R_{\max}+\alpha\log\lvert\A\rvert}{(1-\gamma)\delta}\right]$ for any $(s,a) \in\S\times\A$. Rewriting the optimization problem to its primal form, it is clear that 
$$
f((s,a),\beta^\star) = \inf_{p_{s,a}\in \cP_{s,a}(\d)}\E\left[V(s')\right] \ge 0.
$$ 
When $\beta$ is greater than $\frac{R_{\max}+\alpha\log\lvert\A\rvert}{(1-\gamma)\delta}$, it can never be optimal since 
$$\begin{aligned}
f((s,a),\beta) 
=& -\beta\log\left(\E_{p_{s,a}^0}\left[\exp\left(-\frac{V(s')}{\beta}\right)\right]\right) - \beta\d \\
\le& -\beta\log\left(\exp\left(-\frac{R_{\max}+\alpha\log\lvert\A\rvert}{(1-\gamma)\beta}\right)\right) - \beta\d \\
=& \frac{R_{\max}+\alpha\log\lvert\A\rvert}{1-\gamma} - \beta\d < 0.
\end{aligned}$$

Now we know that $f((s,a), \beta)$ is a finite-valued function for each $(s,a)\in\S\times\A$ and $\beta\in\I_\beta$. Also, it is $\Sigma(\S\times\A)$-measurable in $(s, a) \in \S \times \A$ for each $\beta \in \I_\beta$ and is continuous in $\beta$ for each $(s, a)\in\S \times\A$. From Lemma \ref{lemma: normal integrand}, we know that $f((s,a), \beta)$ is a normal integrand.

Moreover, all functions in $\G$ is upper bounded and measurable so $\M_{f}\subset \G \subset \M((\S\times\A), \Sigma(\S\times\A); \R)$. Proposition \ref{prop: functional} is a direct conclusion of Lemma \ref{lemma: interchange}.
\end{proof}
\newpage
\section{Experiment Details}
\label{sec:exp_details}

\subsection{More Setting details}
\label{sec:exp_setting}

To allow for comparability of results, all tools were evaluated on equal-cost hardware, a Ubuntu 24.04 LTS system with one Intel(R) Core(TM) i7-6850K CPU, one NVIDIA GTX 1080 Ti GPU with 11 GB memory, and 64 GB RAM. All experiments use 12 CPU cores and 1 GPU. 

We implement FQI and RFQI algorithms from \url{https://github.com/zaiyan-x/RFQI}. DDPG and CQL are implemented from the offline RL library \textit{d3rlpy}~\citep{d3rlpy}.

\paragraph{Hyperparameter Selection} 
Across all environments, we use $\gamma=0.99$ for discount rate, $\tau=0.005$ for both $V$ and $Q$ critic soft-update, $\alpha=0.12$ as initial temperature, $\lvert B\rvert=256$ for mini-batch size, $\lvert \cD\rvert=10^6$ for data buffer size. Actor, Q and V critic and VAE networks are multilayer perceptrons (MLPs) with $[256, 256]$ as hidden dimension. In the \hcs and \rvs environments, we use two hidden layers in the actor and critic networks. All other networks have one hidden layer. 

There are multiple learning rates in our algorithm. Learning rate for VAE network $\lambda_\varphi$ is $5\times10^{-5}$ in the \pvs environment and $5\times10^{-4}$ in others. In Step 5 of Algorithm~\ref{alg:DR-SAC}, optimal function $\widetilde{g}^\star$ is found via backpropagation with learning rate $\lambda_\eta$. All other learning rates $\lambda_\psi$, $\lambda_\theta$, $\lambda_\phi$ and $\lambda_\alpha$ are the same in each environment and represented by $\lambda_\psi$. 

Value of learning rates $\lambda_\psi$ and $\lambda_\eta$, number of Q-critics and latent dimensions in VAE are separately tuned in each environment and presented in Table \ref{table:hyperparams}.

\begin{table}[h]
\centering
\caption{Hyper-parameters selection in SAC and DR-SAC algorithm training.}
\label{table:hyperparams}
\begin{tabular}{ c|c|c|c|c } \hline
 Environment & $\lambda_\psi$  & $\lambda_\eta$  & Q-Critic Number & latent dimensions \\ \hline
 \pv & $5\times10^{-4}$ & $5\times10^{-5}$   & 2 & 5 \\ %\hline
 \cp & $3\times10^{-4}$ & $5\times10^{-4}$   & 2 & 5 \\ %\hline
 \lld & $5\times10^{-4}$ & $5\times10^{-4}$   & 2 & 10 \\ %\hline
 \hc & $3\times10^{-4}$  & $5\times10^{-5}$  & 5 & 32 \\ %\hline
 \rv & $3\times10^{-4}$  & $5\times10^{-5}$  & 5 & 10 \\ \hline
\end{tabular}
\end{table}

\paragraph{Offline Dataset}
To ensure a fair performance comparison, all models within each environment are trained on the same offline dataset. Each datasets contains $10^6$ samples, generated by first training a behavior policy and applying the epsilon-greedy method. For most environments, the behavior policy is trained by the Twin Delayed DDPG (TD3;~\cite{fujimoto2018td3}) implemented from the \textit{d3rlpy} library \citep{d3rlpy}. In the \cps environment we use SAC to train the behavior policy. To ensure fair robustness evaluation, all models are trained to achieve the same performance ($500$, the maximum reward) in the unperturbed \cps environment. Datasets generated by TD3-trained (or SAC-trained) behavior policies are denoted as TD3-datasets (or SAC-datasets). Additional details, including the algorithm to train behavior policy, training steps and the random-action probability $\epsilon$ are presented in Table~\ref{table:dataset}. 

\begin{table}[h]
\centering
\caption{Experiment details in dataset generation}
\label{table:dataset}
\begin{tabular}{ c|c|c|c } \hline
 Environment & Behavior Policy Algorithm  &  Training Steps & Random-Action Probability $\epsilon$ \\ \hline
 \pv & TD3   & $5\times10^4$ & $0.5$ \\ %\hline
 \cp & SAC   & $5\times10^5$ & $0.5$ \\ %\hline
 \lld & TD3   & $3\times10^5$ & $0.5$ \\ %\hline
 \hc & TD3   & $10^6$ & $0.3$ \\ %\hline
 \rv & TD3   & $10^6$ & $0.3$ \\ \hline
\end{tabular}
\end{table}

% \begin{itemize}[leftmargin=0pt,labelsep=0.5em]
% \item{TD3-Dataset} For \pv, \lld, \hc and \rv environment, we train TD3 \cite{fujimoto2018td3} policies for $5\times10^4$ and $3\times10^5$ steps respectively and use epsilon-greedy method ($\epsilon=0.5$) to generate $10^6$ data.

% \item{SAC-Dataset} For \cp environment, we first use the online version of SAC algorithm to train a model for $2\times10^5$ steps and use epsilon-greedy method ($\epsilon=0.5$) to generate $10^6$ data. To ensure fairness in evaluating the robustness under perturbation, all models are trained to have the same performance (500, the highest possible reward in this environment) when no changes are made.
% \end{itemize}

\subsection{Extra Experiment Results}
\label{sec:extra_exp}
%%%%%%%%%%%%%%%%%%%%%%%%%%%%%%%%%%%%%%%%%% Pendulum
\paragraph{\pv}
In the \pvs environment, we compare DR-SAC with SAC, FQI, and DDPG. All models are trained on TD3-dataset. The robust algorithm RFQI does not perform well even in the nominal environment. To evaluate the robustness of trained models, we change the environment parameters \textit{length}, \textit{mass}, and \textit{gravity}, with nominal values as $1.0$, $1.0$ and $10.0$ respectively. To further show model performance under heavy-tailed perturbation, we also add Cauchy-distributed noise to state observations. The distribution of noise is defined as standard Cauchy distribution multiplied by a parameter \textit{noise scale}. We grind search $\delta\in\{0.1, 0.2, \cdots, 1.0\}$ and find model under $\d=0.5$ have the best overall robustness.

DR-SAC shows consistent robustness improvement compared to all other algorithms. The performance under length perturbation is presented in Figure~\ref{fig:performance}(a). In the mass perturbation test, DR-SAC has the best performance in all cases. For example, the average reward is over $40\%$ higher than SAC when mass changes $120\%$. In Figure \ref{fig:pv} (b), there is a notable gap between DR-SAC and SAC performance when gravity acceleration changes $40\%$. In Figure \ref{fig:pv} (c), DR-SAC achieves consistent the best performance when \textit{noise scale} increases.

\begin{figure}[h]
  \centering

  \subfloat[Mass Perturbation.]
  {%
    \includegraphics[width=0.329\textwidth,valign=t]{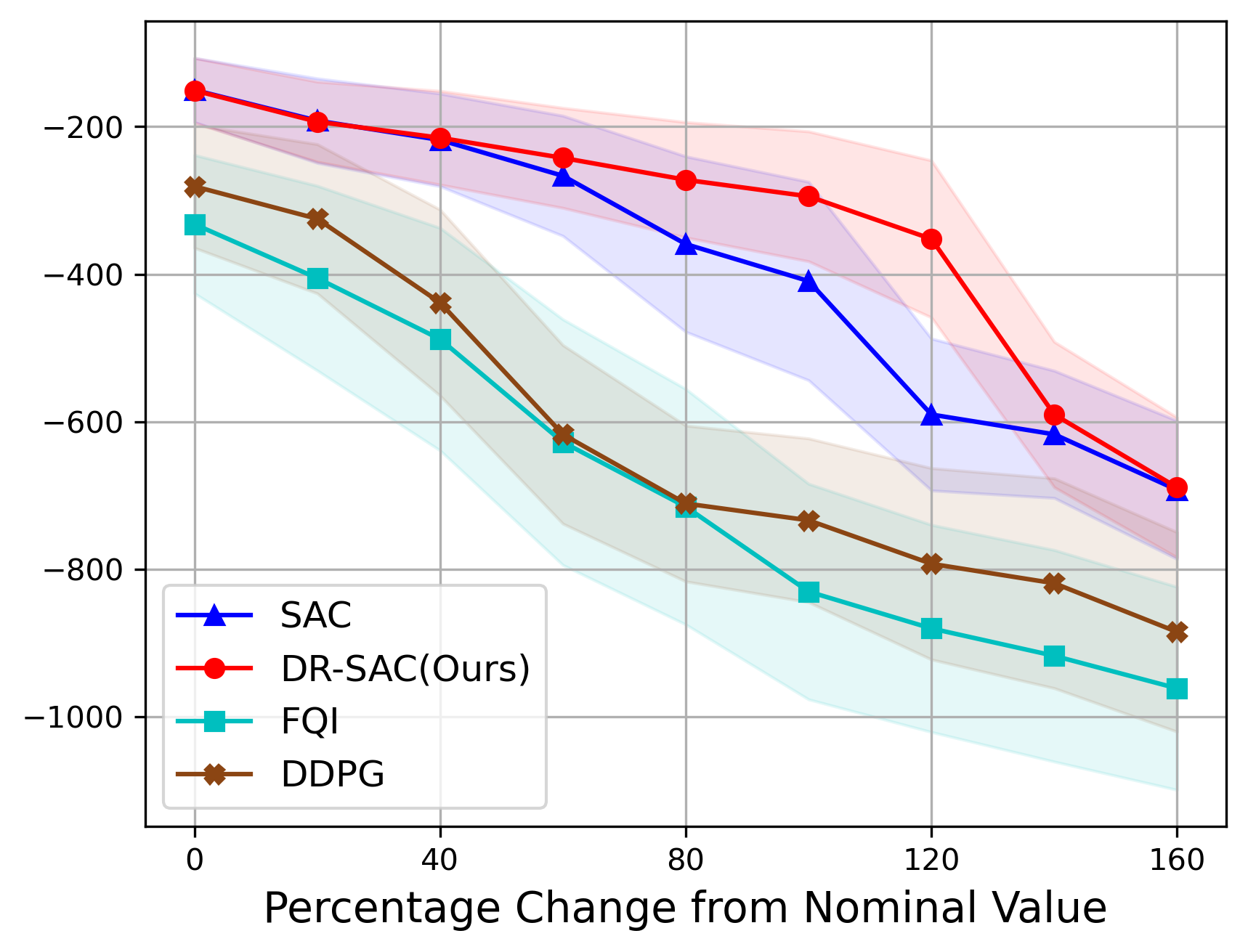}%
  } \hfill
  \subfloat[Gravity Perturbation.]
  {%
    \includegraphics[width=0.329\textwidth,valign=t]{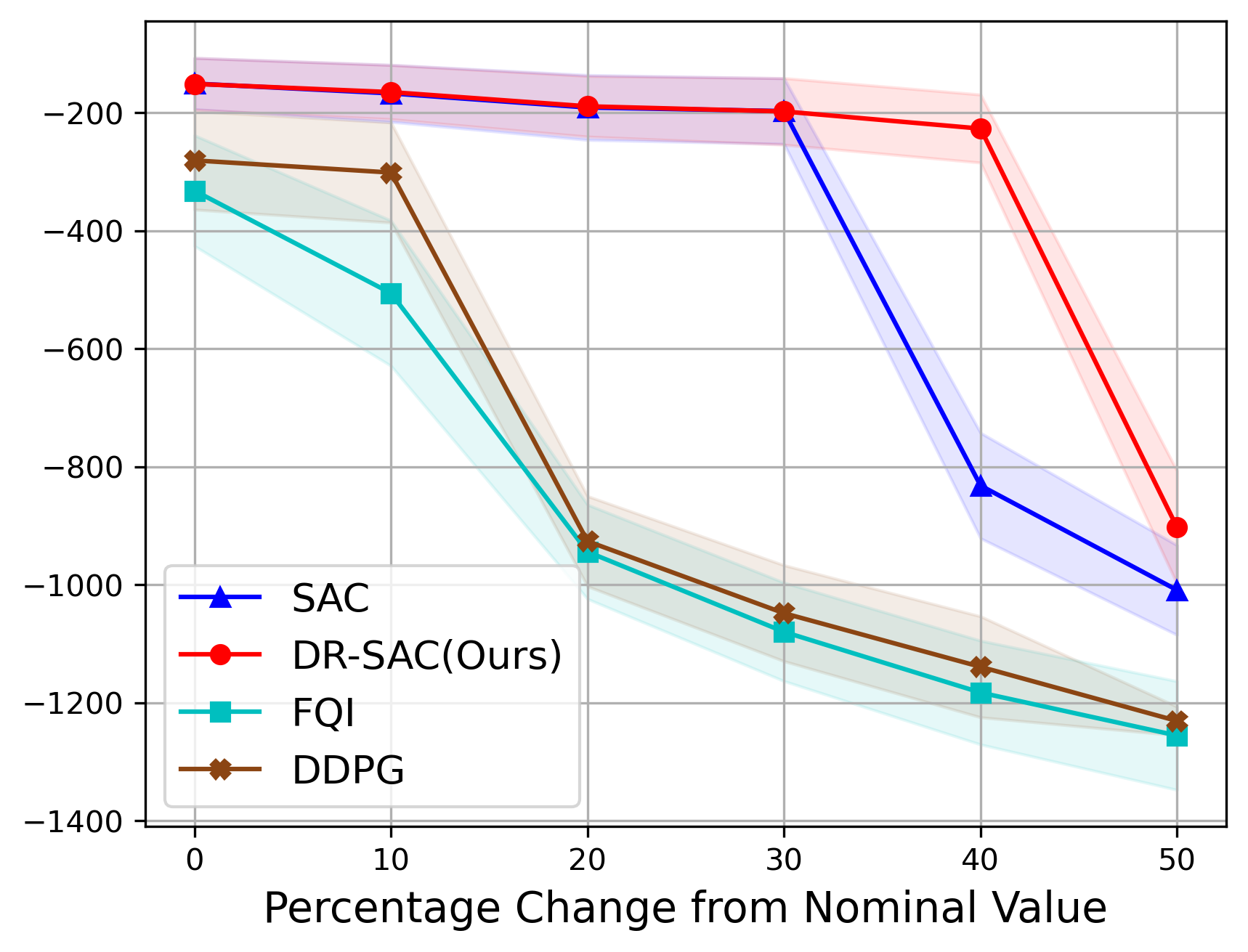}%
  }
  \subfloat[Observation Noise.]
  {%
    \includegraphics[width=0.329\textwidth,valign=t]{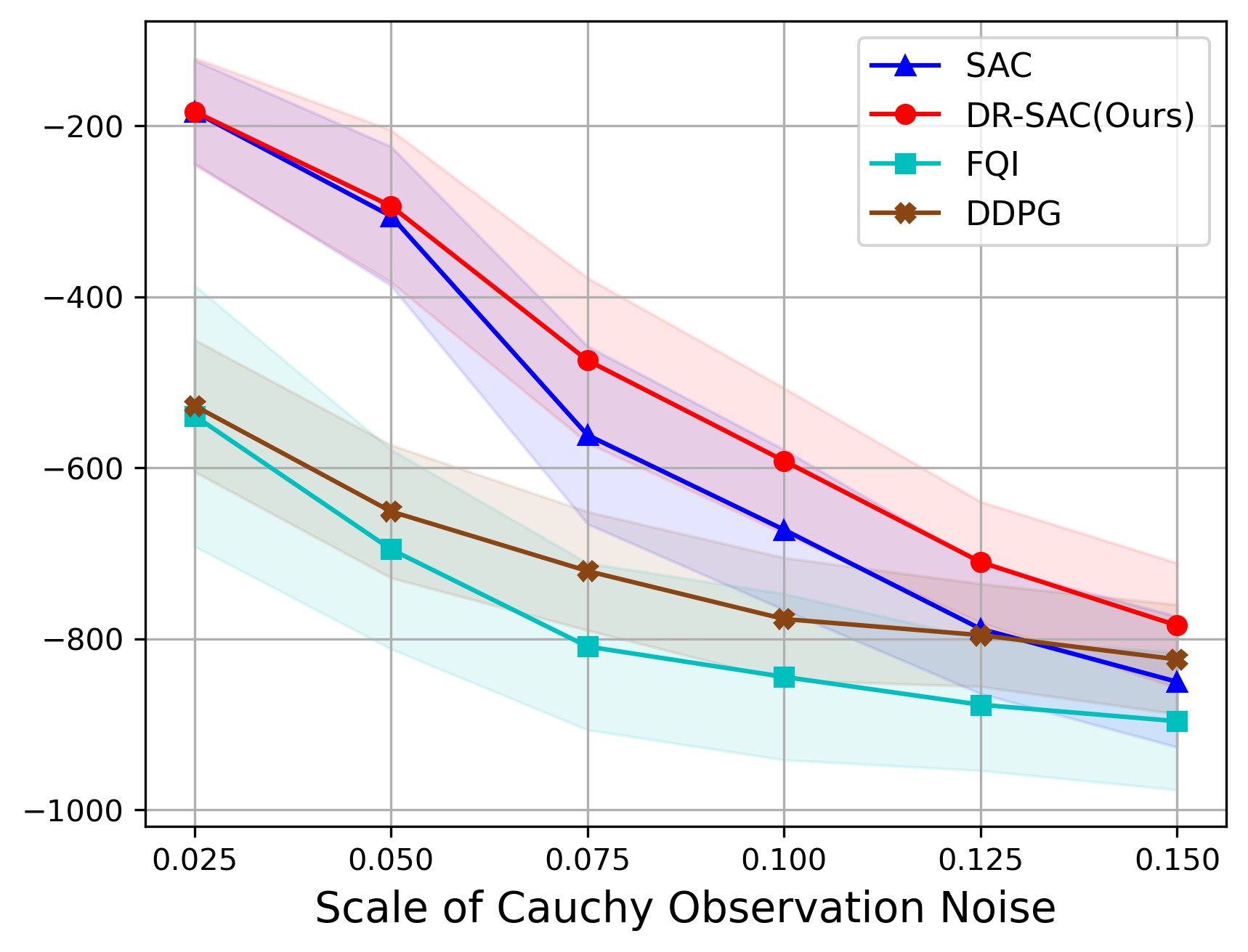}%
  }
\caption{Additional \pvs experiment results on TD3-dataset.}
\label{fig:pv}
\end{figure}

%%%%%%%%%%%%%%%%%%%%%%%%%%%%%%%%%%%%%%%%%% Cartpole
\paragraph{\cp}
In the \textit{Cartpole} environment, we compare the DR-SAC algorithm with non-robust algorithms SAC, DDPG, FQI, and robust algorithm RFQI. All algorithms are trained on the SAC-dataset. In the \textit{Cartpole} environment, the force applied to the cart is continuous and determined by the actuator's action and parameter \textit{force\_mag}. The highest possible reward is 500 in each episode. To ensure fair comparison, all models are trained to have average rewards of 500 in the nominal environment.

We test the robustness by introducing two changes to the environment: applying action perturbation and adding observation noise. In the action perturbation test, the actuator takes random actions with different probabilities. In the observation perturbation test, noise with zero mean and different standard deviations is added to the nominal states in each step. We grind search $\delta\in\{0.25,0.5, 0.75, 1.0\}$ and find DR-SAC has the best performance when $\delta=0.75$. We also use $\rho=0.75$ to train the RFQI model.

In the \cps environment, DR-SAC has the best overall performance under both type of perturbation. Figure \ref{fig:cp and lld}(a) extends Figure~\ref{fig:performance}(b). DR-SAC has performance improvement over 75\% compared to non-robust algorithms SAC and DDPG when the standard deviation of noise is 0.2 and 0.3.

\begin{figure}[h]
  \centering
  \subfloat[\centering Observation Perturbation \\ \cp]
  {%
    \includegraphics[width=0.485\textwidth,valign=t]{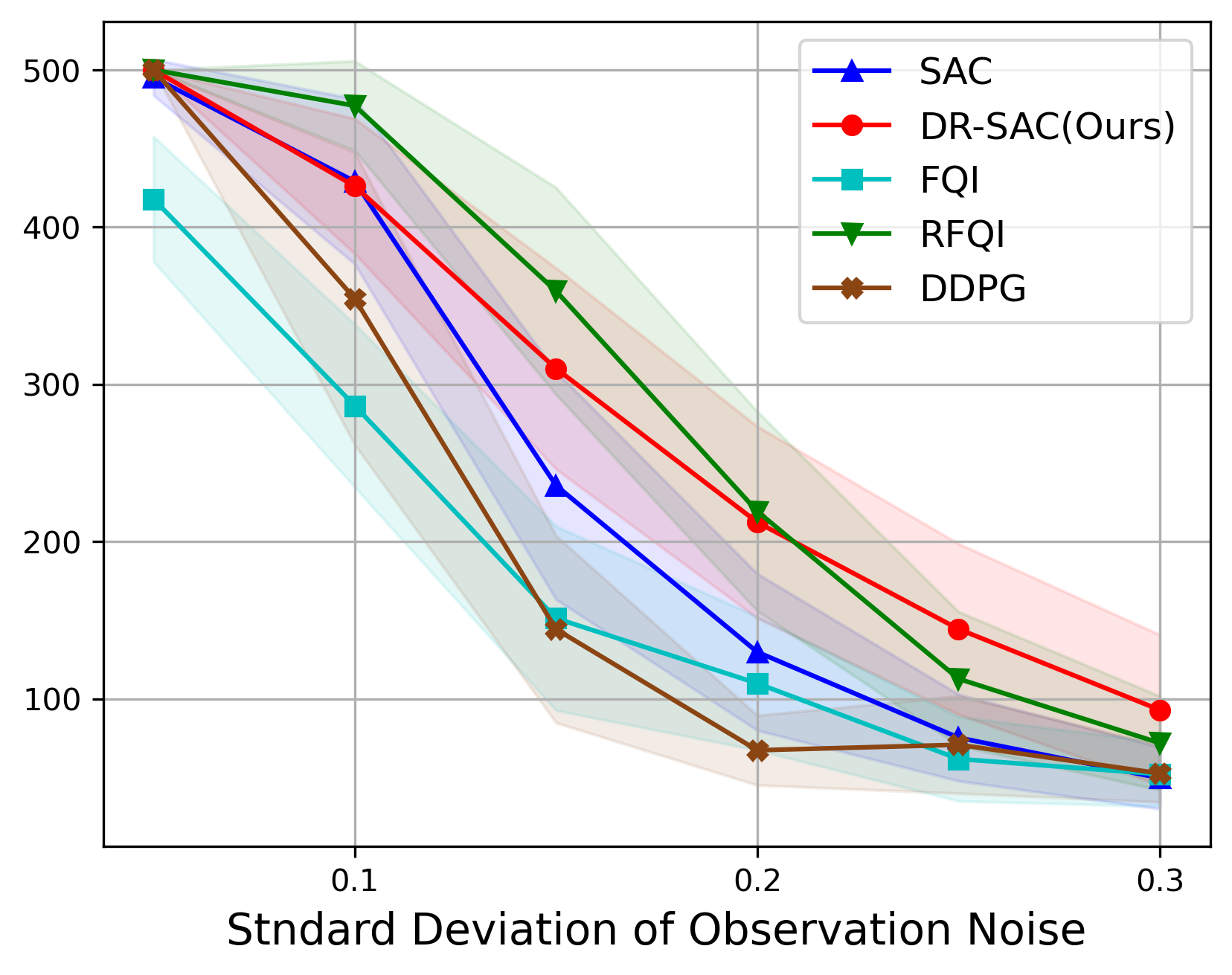}%
  } \hfill
  \subfloat[\centering Wind Power Perturbation \\ \lld]
  {%
    \includegraphics[width=0.485\textwidth,valign=t]{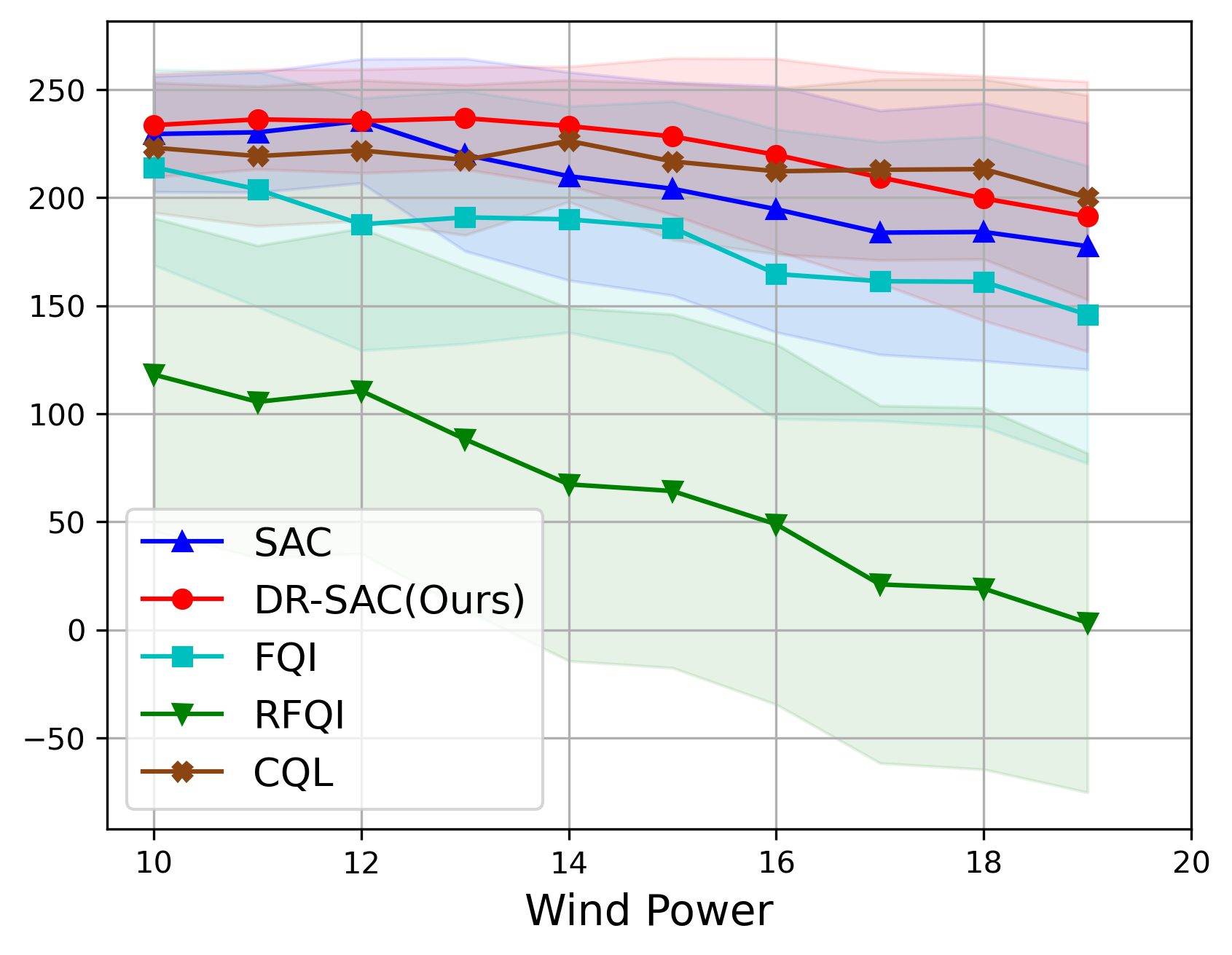}%
  }
\caption{Additional \textit{Cartpole} results on SAC-dataset and \llds results on TD3-dataset. }
\label{fig:cp and lld}
\end{figure}

%%%%%%%%%%%%%%%%%%%%%%%%%%%%%%%%%%%%%%%%%% Lunarlander
\paragraph{\lld}
In the \llds environment, we compare DR-SAC with non-robust algorithms SAC, CQL, FQI, and robust algorithm RFQI. All algorithms are trained on TD3-dataset. In the \llds environment, the lander has main and side engines, and the actuator can control the throttle of the main engine. We change environment parameters \textit{engine\_power} (main and side engine power) and \textit{wind\_power} (magnitude of linear wind) to validate algorithm robustness. We grind search $\delta\in\{0.25,0.5, 0.75, 1.0\}$ and find DR-SAC has the best performance when $\delta=0.25$. We also use $\rho=0.25$ to train the RFQI model.

% \begin{figure}%{0.5\textwidth}
%    \begin{center}
%     \includegraphics[width=0.48\textwidth]{figures/lunarlander/wind.png}
%   \end{center}
%   \caption{\llds results on TD3-dataset. The curves show the average reward of 50 episodes, shaded by $\pm$0.5 standard deviation.}
%   \label{fig:lld}
% \end{figure}

Under all types of perturbations, DR-SAC shows superior robustness compared to other algorithms. The performance under \textit{engine\_power} perturbation is presented in Figure~\ref{fig:performance}(c). In Figure~\ref{fig:cp and lld}(b), DR-SAC shows the highest average reward in most levels of wind perturbation. It is worth noting that the robust algorithm RFQI does not have an acceptable performance in this test, even compared to its non-robust counterpart FQI.

% \begin{figure}[h]
%   \centering

%   \subfloat[Engine Perturbation:  environment parameters main and side engine power change.]
%   {%
%     \includegraphics[width=0.485\textwidth,valign=t]{figures/lunarlander/engine.png}%
%   } \hfill
%   \subfloat[Wind Perturbation: linear wind applied.]
%   {%
%     \includegraphics[width=0.485\textwidth,valign=t]{figures/lunarlander/wind.png}%
%   }
% \caption{\llds results on TD3-Dataset. The curves show average reward of 50 episodes shaded by $\pm$0.5 standard deviation.}
% \label{fig:lld}
% \end{figure}

%%%%%%%%%%%%%%%%%%%%%%%%%%%%%%%%%%%%%%%%%% Reacher
\paragraph{\rv} 
In the \rvs environment, we compare DR-SAC with non-robust algorithms SAC, FQI, CQL, and robust algorithm RFQI. All algorithms are trained on TD3-dataset. In \rvs environment, the actuator controls a two-jointed robot arm to reach a target. We use \textit{joint\_damping} to denote the damping factor of both \textit{joint0} and \textit{joint1}, with default value as $1.0$. We grind search $\delta\in\{0.1,\,0.2,\,0.3\}$ and find DR-SAC has the best performance when $\delta=0.2$. We also use $\rho=0.2$ to train the RFQI model.

To test the robustness of all algorithms, we compare their performance after adding observation noise and changing parameters \textit{joint\_damping}. In the observation perturbation test, we add zero-mean Gaussian noise to the nominal state in dimensions $4-9$. The first 4 dimensions in state are trigonometric function values and are kept unperturbed. Performance under both perturbations is presented in Figure~\ref{fig:performance} (d) and (e). Moreover, in Figure~\ref{fig:performance} (e), the standard deviation regions were computed but omitted from the final plot because the overlapping shaded areas of multiple algorithms made the figure unreadable. We provide them in Table~\ref{tabel:reacher std}.

\begin{table}[h]
\centering
\caption{Standard Deviation of Model Performance under Damping Perturbation in \rv}
\label{tabel:reacher std}
\begin{tabular}{ c|c|c|c|c|c } \hline
 Joint Damping Value & SAC & DR-SAC & FQI & RFQI & CQL \\ \hline
 2.0 & 1.452   & 1.459 & 1.638 & 1.509 & 1.396  \\ %\hline
 3.0 & 1.534  & 1.540 & 1.754 & 1.601 & 1.456 \\ %\hline
 4.0 & 1.631   & 1.628 & 1.865 & 1.699 & 1.529  \\ 
 5.0 & 1.735 & 1.721 & 1.969 & 1.802 & 1.614\\
 6.0 & 1.841 & 1.819 & 2.057 & 1.910 & 1.711\\ 
 7.0 & 1.946 & 1.925 & 2.135 & 2.018 & 1.816 \\ \hline
\end{tabular}
\end{table}

%%%%%%%%%%%%%%%%%%%%%%%%%%%%%%%%%%%%%%%%%% HalfCheetah
\paragraph{\hc}
In the \hcs environment, we compare DR-SAC with SAC baseline only due to the unsatisfactory performance of FQI and RFQI. All algorithms are trained on TD3-dataset. In the \hcs environment, the actuator controls a cat-like robot consisting of $9$ body parts and $8$ joints to run.  We use \textit{front\_stiff} and \textit{front\_damping} to denote the stiffness and damping factor of joint \textit{fthigh}, \textit{fshin}, and \textit{ffoot}. Also, \textit{back\_stiff} and \textit{back\_damping} can be denoted in a similar way. The default value of these parameters can be found through the environmental assets of Gymnasium MuJoCo in \url{https://github.com/Farama-Foundation/Gymnasium/blob/main/gymnasium/envs/mujoco/assets/half_cheetah.xml}. We grind search $\delta\in\{0.1,\,0.2,\,0.3\}$ and find DR-SAC has the best performance when $\delta=0.2$. 

Performance of \textit{back\_damping} test is presented in Figure~\ref{fig:performance}(f). Combining it with Figure~\ref{fig:hc}, we can see DR-SAC has notable robustness improvement across all perturbation tests. For example, in \textit{front\_stiff} perturbation test, DR-SAC achieves an improvement as much as $10\%$ when the change is $80\%$. 

\begin{figure}[h]
  \centering

  \subfloat[Front Stiffness Perturbation]
  {%
    \includegraphics[width=0.485\textwidth,valign=t]{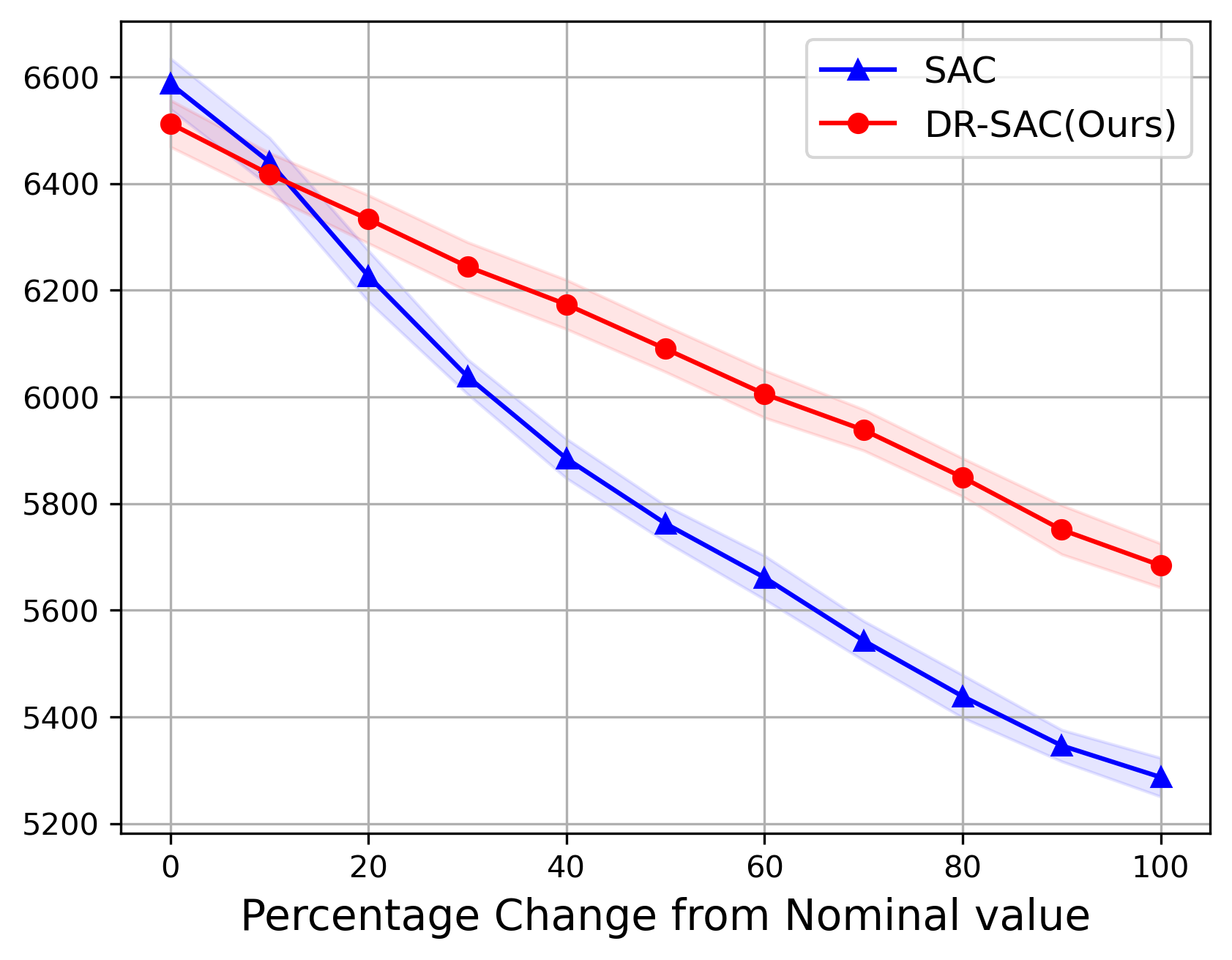}%
  } \hfill
  \subfloat[Front Damping Perturbation]
  {%
    \includegraphics[width=0.485\textwidth,valign=t]{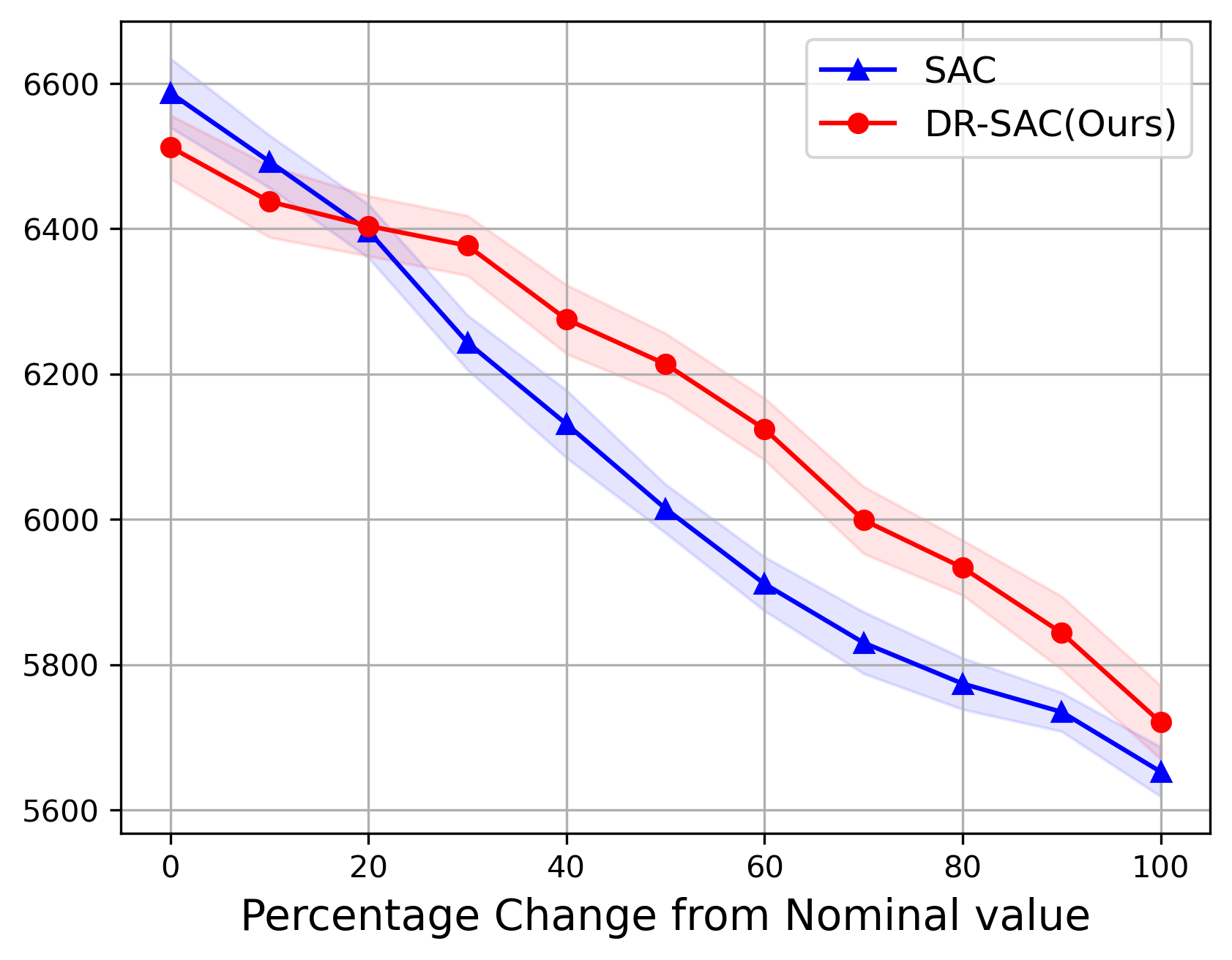}%
  }
\caption{Additional \hcs results on TD3-dataset.}
\label{fig:hc}
\end{figure}

\subsection{Ablation Study Details}

\subsubsection{Training Efficiency of DR-SAC.}
\label{sec:ablation efficiency}
In this section, we want to show that DR-SAC with functional optimization finds a good balance between efficiency and robustness. We compare training time and robustness of Algorithm \ref{alg:DR-SAC}, DR-SAC without functional optimization, and robust algorithm RFQI, to show our DR-SAC algorithm has the best overall performance.

\paragraph{Comparison with Accurate Bellman Operator}
We first introduce DR-SAC algorithm without functional optimization. Most steps are the same as Algorithm \ref{alg:DR-SAC}, instead of following modifications. Step 5 in Algorithm \ref{alg:DR-SAC} is removed. $Q$-network loss is replaced by 
\begin{equation}
J_Q^{\text{DR\_acc}} = \E_{(s,a)\sim\cD}\left[Q^\pi_{\M}(s,a) - \widetilde{\mathcal{T}}^\pi_{\d} Q^\pi_{\M_\d}(s,a)\right]^2,
\end{equation}
where $\widetilde{\mathcal{T}}^\pi_{\d}$ is the empirical version of $\mathcal{T}^\pi_{\d}$ by replacing $p_{s,a}^0$ with $\widetilde{p}_{s,a}^0$. We call this modified algorithm \textit{DR-SAC-Accurate} and call Algorithm \ref{alg:DR-SAC} \textit{DR-SAC-Functional} in this section. 

We train SAC, \textit{DR-SAC-Functional}, and \textit{DR-SAC-Accurate} algorithms in \pvs environment. The optimization problem in Equation \eqref{eq: q_bellman dual} is a problem over scalar $\beta>0$ and solved via \textit{Scipy} for each $(s,a)$ pair. Table \ref{tabel:train time accurate} presents the training steps and time for three algorithms. Training time of \textit{DR-SAC-Accurate} is over $150\times$ longer than standard SAC and over $50\times$ longer than \textit{DR-SAC-Functional}. Considering \pvs environment is relatively simple, \textit{DR-SAC-Accurate} algorithm is hard to utilize in large-scale problems.

\begin{table}[h]
\centering
\caption{Training steps and time for three algorithms in \pv}
\label{tabel:train time accurate}
\begin{tabular}{ c|c|c } \hline
 Algorithm & Training Steps  & Training Time (Minute)  \\ \hline
 SAC & $10$k   & 1.7  \\ %\hline
 \textit{DR-SAC-Functional} & $10$k  & 4.7  \\ %\hline
 \textit{DR-SAC-Accurate} & $8$k   & 260  \\ \hline
\end{tabular}
\end{table}

Moreover, we test the robustness of three algorithms by comparing their average reward under different perturbations. To be specific, we change \pvs environment parameters: \textit{length} and \textit{mass}. \textit{DR-SAC-Functional} and \textit{DR-SAC-Accurate} are trained with $\delta=0.5$. Figure \ref{fig:pv acc} shows that \textit{DR-SAC-Functional} achieves comparable and even better performance under small-scale perturbation. For example, \textit{DR-SAC-Functional} and \textit{DR-SAC-Accurate} have almost the same performance under \textit{mass} perturbation test when change is less than $120\%$. In \textit{length} perturbation test, \textit{DR-SAC-Functional} has better performance when the change is less than $30\%$.

% \begin{figure}[h]
%   \centering

%   \subfloat[Gravity Factor Perturbation.]
%   {%
%     \includegraphics[width=0.329\textwidth,valign=t]{figures/pendulum/gravity-time.png}%
%   } \hfill
%   \subfloat[Mass Factor Perturbation.]
%   {%
%     \includegraphics[width=0.329\textwidth,valign=t]{figures/pendulum/mass-time.png}%
%   } \hfill
%   \subfloat[Length Factor Perturbation.]
%   {%
%     \includegraphics[width=0.329\textwidth,valign=t]{figures/pendulum/length-time.png}%
%   }
% \caption{\pvs results on TD3-dataset. Curves show average reward of 50 episodes, shaded by $\pm$0.5 standard deviation. Algorithms are SAC, DR-SAC with and without functional approximation.}
% \label{fig:pv acc}
% \end{figure}

\paragraph{Efficiency Comparison with RFQI.}
In Section \ref{sec:experiment performance}, existing DR-RL algorithm RFQI also shows comparable performance under some perturbations. In this paragraph, we want to show that DR-SAC requires much less training time than RFQI, improving its applicability to large scale problems. Table \ref{table:training time fqi} lists the training time of SAC, DR-SAC, FQI, and RFQI algorithms in three testing environments. DR-SAC is demonstrated to be well-trained in at most $20\%$ time required by RFQI. Compared with each non-robust counterpart, the training time of DR-SAC is at most $3.6\times$ of SAC, while RFQI requires $10-13\times$ more training time than FQI. In Figure \ref{fig:training time rv}, we provide a plot of performance changes against the training time in the \rvs environment, where RFQI is shown to be under-trained when the curve of DR-SAC converges.

\begin{figure}[h]
    \centering
    \includegraphics[width=0.5\textwidth]{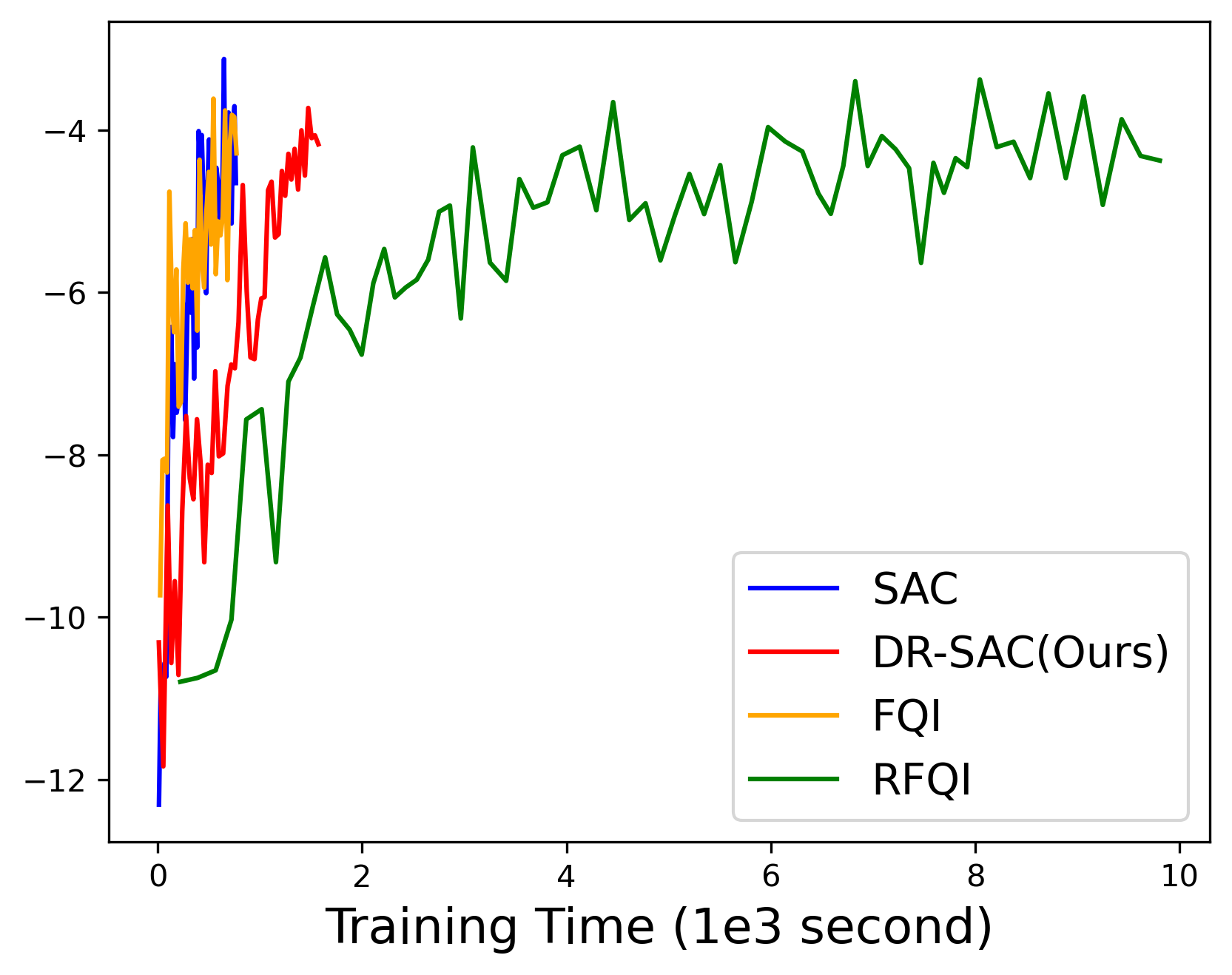}
    \captionof{figure}{Average Reward of 20 Episodes over Training Time in \rvs Environment.}
   \label{fig:training time rv}
\end{figure}

Moreover, this efficiency improvement does not solely arise from the functional approximation step, but also from the inherent optimization efficiency in the loss function structure. The RFQI algorithm considers the RMDP framework with uncertainty sets defined by the TV distance and is empirically built on the BCQ algorithm. In RFQI, there exists a step similar to Equation~\eqref{eq: empirical opt g} to find the optimal functional under empirical measurement. Experimental results show that the efficiency gap arises from the number of GD steps in solving this optimization problem. RFQI sets the default GD steps as $1000$ while DR-SAC achieves comparable robustness performance with only 5 steps. 
To further investigate, we vary the GD steps in RFQI to 5, 10 and 100 in the \llds environment and report the model performance in the unperturbed environment. As shown in Table~\ref{tabel:gd steps}, performance drops sharply when RFQI uses fewer GD steps, indicating that the loss function structure in RFQI inherently leads to slower convergence and requires more optimization steps. In our framework, the choice of actor-critic based non-robust baseline, KL divergence induced uncertainty set and generative modeling in nominal distribution estimation together yields a more optimization-friendly formulation, contributing to our method's practical efficiency.
%We believe this difference stems from the divergence choice, leading to different dual formulations and loss function structures. The KL divergence in our framework naturally yields a more optimization-friendly formulation, contributing to our method's practical efficiency.

\begin{table}[h]
\centering
\caption{GD steps, training time and performance in \lld}
\label{tabel:gd steps}
\begin{tabular}{ c|c|c|c|c|c } \hline
 Algorithm & DR-SAC	& RFQI	& RFQI	& RFQI	& RFQI (Used)  \\ \hline
 GD Steps & 5	& 5	& 10	& 100	& 1000 \\ 
 Training Time (min) & 36	& 12 & 21 & 139 & 238 \\ 
Nominal Env Performance & 240.0	& 175.9	& 181.9	& 192.9	& 201.2  \\ \hline
\end{tabular}
\end{table}

\subsubsection{Selection of Generative Model}
\paragraph{Robustness of VAE.}
A consistent challenge in DR-RL algorithm design is that unknown nominal distributions $p_{s,a}^0$ often appear in the loss function. In Section~\ref{sec:gen model} and Appendix~\ref{sec:discuss vae}, we review methods used in other model-free DR-RL algorithms and motivate the necessity of generative models in our setting. Although generative models inevitable introduce additional estimation error when constructing empirical measures $\tilde{p}_{s,a}^0$, our ablation studies demonstrate that DR-SAC is largely insensitive to the VAE modeling, therefore improving its applicability. In the \pvs environment, where the state and action space dimensions are $3$ and $1$ respectively, we train DR-SAC with VAEs of latent dimensions ${1, 5, 10, 20, 50}$ and evaluate performance under perturbed pendulum mass. As shown in Figure~\ref{fig:gen_type}(a), DR-SAC maintains superior robustness over the SAC baseline as long as the latent dimension lies within a reasonable range (between 5 and 20 in our experiments).

% \label{sec:ablation vae model}
% \begin{figure}[h]
%     \centering
%     \includegraphics[width=0.5\textwidth]{figures/pendulum/vae-dimension.png}
%     \captionof{figure}{\pvs results on TD3-dataset with mass perturbation and different VAE latent dimensions. The curves show the average reward of 50 episodes, shaded by $\pm$0.5 standard deviation.}
%    \label{fig:vae dimension}
% \end{figure}

\paragraph{Comparison with Other Generative Models.}
To demonstrate the choice of VAE over other generative models, we implemented Diffusion Probabilistic Models and Normalizing Flows as alternatives to the VAE in DR-SAC and conducted ablation studies on \pvs and \cp. Model performance is provided in Figure~\ref{fig:gen_type}(b), (c) and Figure~\ref{fig:add_gen_type}. Flow-based models showed unstable performance even in unperturbed \pvs environment. DR-SAC with Diffusion models achieved comparable robustness to the VAE in \pv. Crucially, the efficiency of sampling process with Diffusion models is a major bottleneck. Diffusion-based training is at least $4.5\times$ slower than VAE-based training.

\begin{figure}[h]
  \centering

  \subfloat[\centering Gravity Perturbation \\ \pv]
  {%
    \includegraphics[width=0.325\textwidth,valign=t]{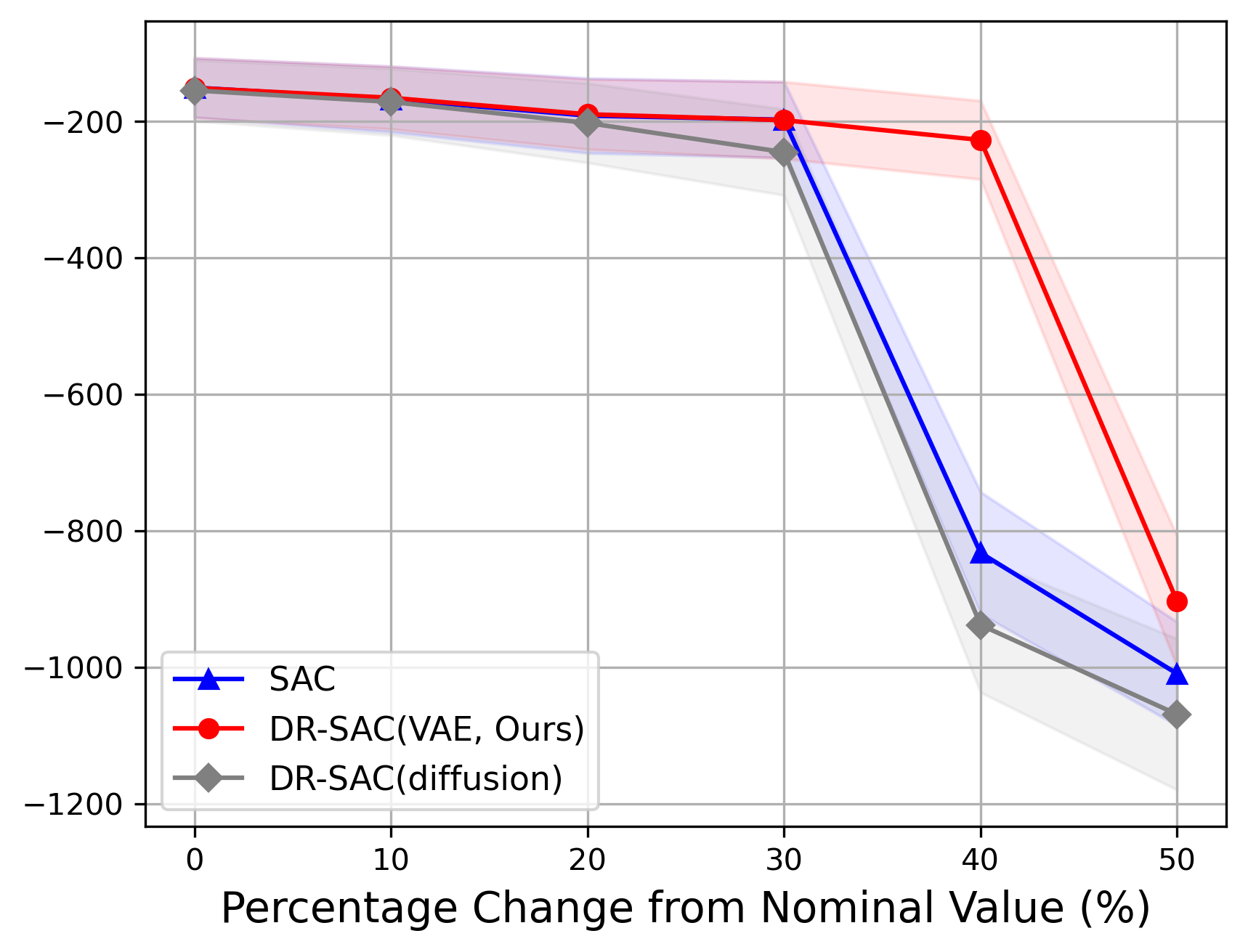}%
  } \hfill
  \subfloat[\centering Mass Perturbation \\ \pv]
  {%
    \includegraphics[width=0.325\textwidth,valign=t]{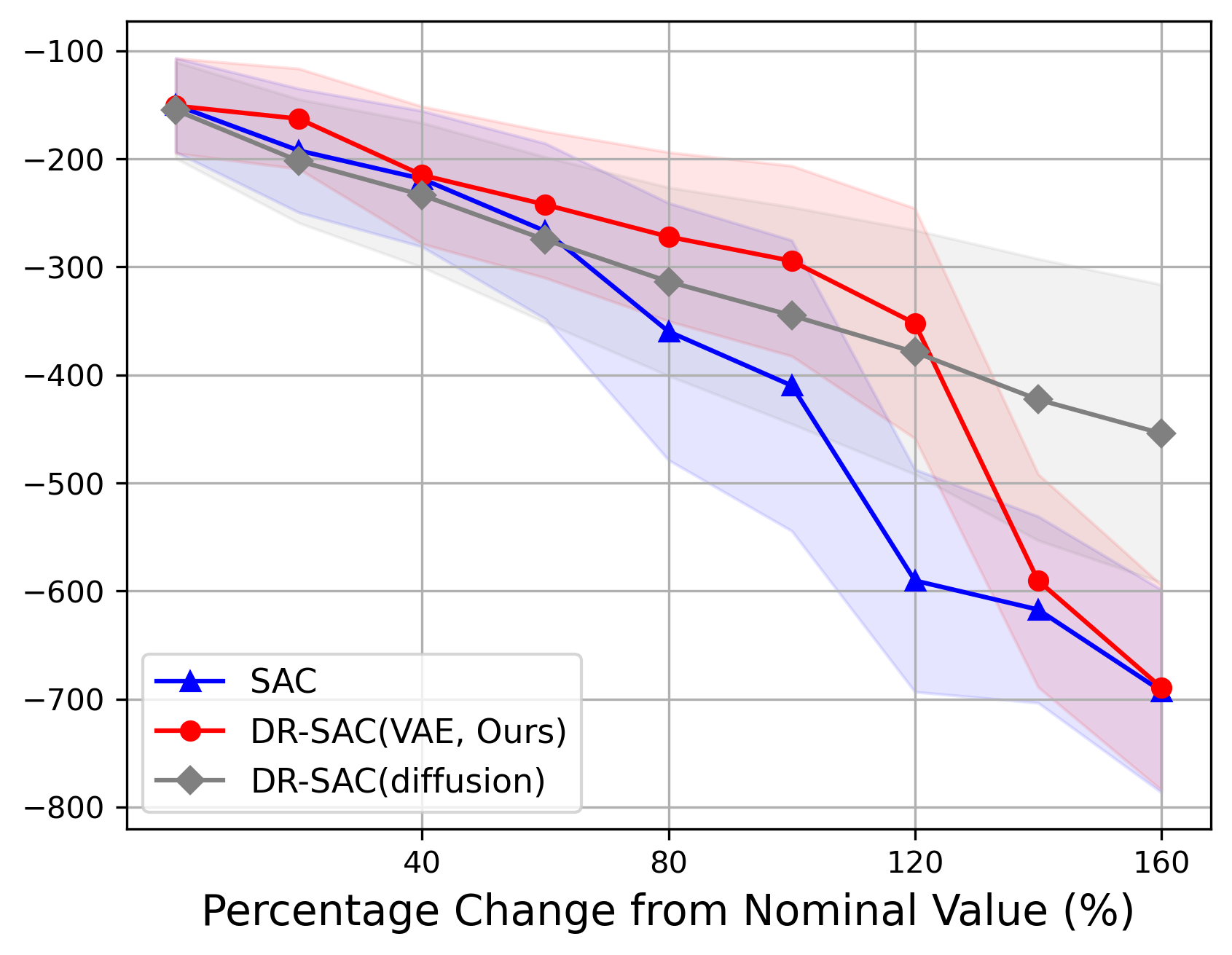}%
  }  \hfill
  \subfloat[\centering Action Perturbation \\ \cp]
  {%
    \includegraphics[width=0.325\textwidth,valign=t]{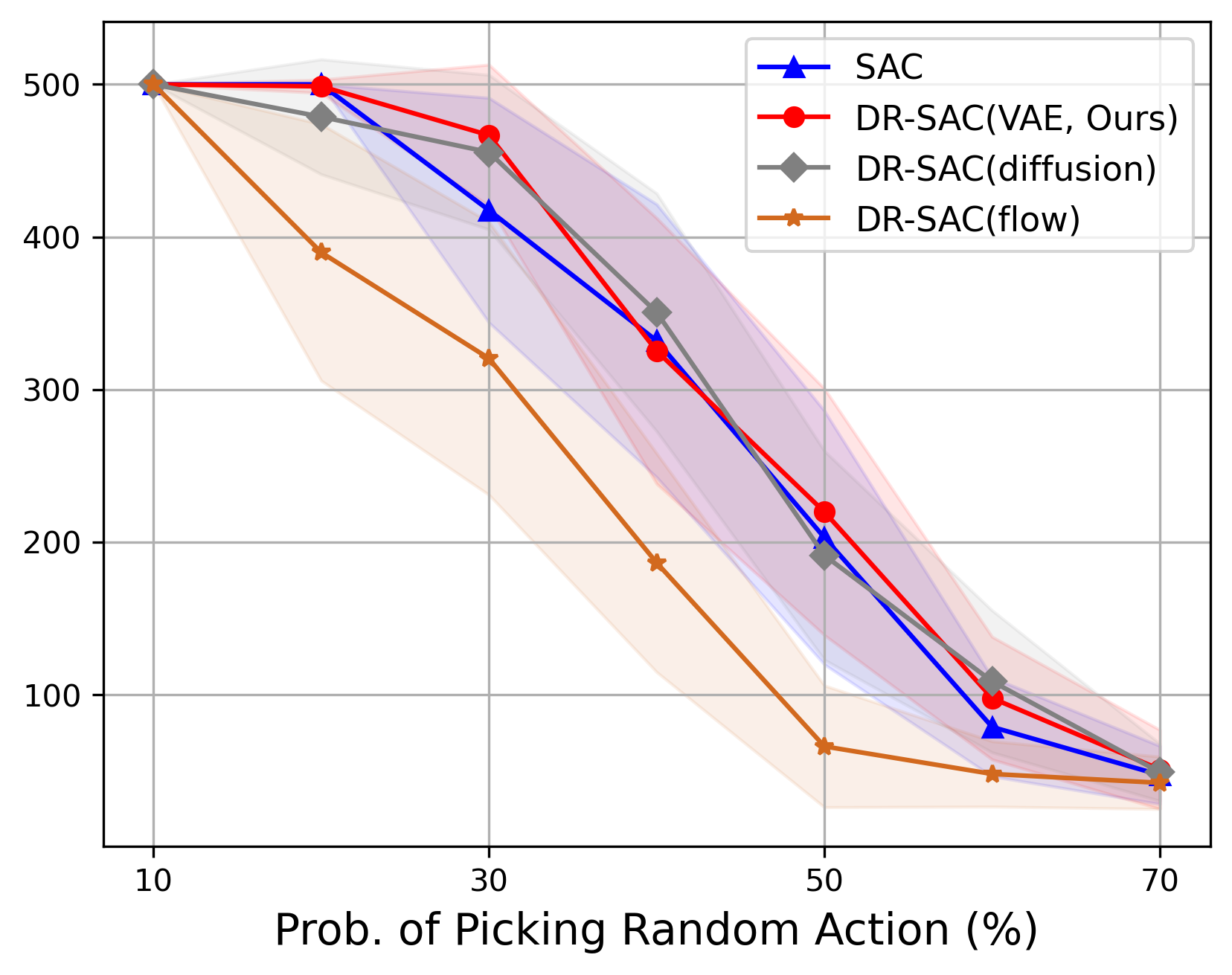}%
  }
\caption{Additional results with different generative models in \pvs and \cp.}
\label{fig:add_gen_type}
\end{figure}

% \begin{figure}[h]
%   \centering
%   \subfloat[Observation Perturbation.]
%   {%
%     \includegraphics[width=0.485\textwidth,valign=t]{figures/cartpole/action_gen_type.png}%
%   } \hfill
%   \subfloat[``Force\_mag'' Perturbation.]
%   {%
%     \includegraphics[width=0.485\textwidth,valign=t]{figures/cartpole/obs_noise_gen_type.png}%
%   }
% \caption{Additional \textit{Cartpole} results with different generative models  on SAC-dataset.}
% \label{fig:cp_gen_type}
% \end{figure}

\subsubsection{Usage of V-Network}
\label{sec:ablation use v}
In this section, we demonstrate that keeping the $V$-network in the SAC algorithm reduces the sensitivity on dataset distribution. As introduced in Appendix~\ref{sec:exp_setting}, offline datasets in this work are generated by first training a behavior policy and applying the epsilon-greedy method to collect data. Experimental results shows that SAC without the $V$-network exhibits unstable performance when the behavior policy differs across datasets.

Our experiments are conducted in the \pvs environment. We generate two datasets with behavior policy trained by an online version of SAC and TD3, denoted as SAC-dataset and TD3-dataset, respectively. Figure~\ref{fig:use v pv} presents the average reward of $20$ episodes against training steps in four scenarios: SAC-dataset vs. TD3-dataset, SAC algorithm with vs. without $V$-network. Removing the $V$-network shows minor influence on offline SAC learning using SAC-dataset. However, for TD3-dataset, SAC with $V$-network achieves a stable average reward around $-150$ quickly, but the average reward of SAC without $V$-network fluctuates intensely and never exceeds $-200$. This validates that SAC with a $V$-network is less sensitive to behavior policy and dataset distribution.

\begin{figure}[h]
    \centering
    \includegraphics[width=0.5\textwidth]{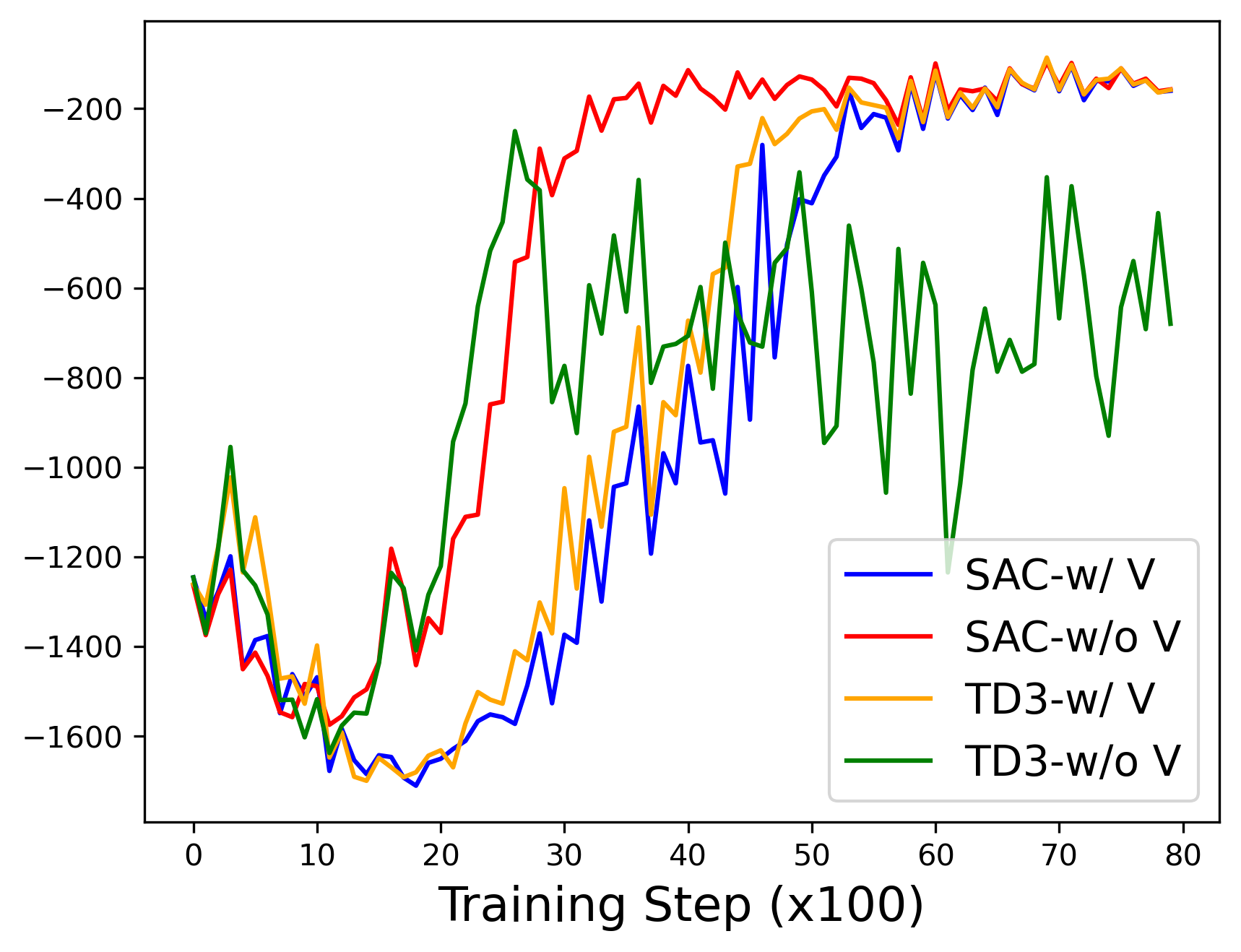}
    \captionof{figure}{Average Reward of 20 Episodes over Training Step in \pvs Environment.}
   \label{fig:use v pv}
\end{figure}

\newpage
\section{Regret Bound}
\label{sec:bound}
%-----------------------additional definition--------------------%
\begin{definition}
    The distributionally robust regret $R_{\mathcal{M}_\delta}(\pi)$ of a policy $\pi \in \Pi$ is defined as:
$$
R_{\mathcal{M}_\delta}(\pi):=\left\|V_{\mathcal{M}_\delta}^\star-V_{\mathcal{M}_\delta}^{\pi}\right\|_{\infty} .
$$
\end{definition}

For any policy $\pi$, the soft value and soft $Q$-functions satisfy:
$$V_{\mathcal{M}_\delta}^\pi(s)=\mathbb{E}_{a \sim \pi(\cdot \mid s)}\left[Q_{\mathcal{M}_\delta}^\pi(s, a)-\alpha \log \pi(a \mid s)\right].$$
The following inequality holds:
$$\left\|V_{\mathcal{M}_\delta}^\star-V_{\mathcal{M}_\delta}^{\pi}\right\|_{\infty} \le \left\|Q_{\mathcal{M}_\delta}^\star-Q_{\mathcal{M}_\delta}^{\pi}\right\|_{\infty} . $$

Based on the estimator $\hat p_{s,a}^0$, we define the corresponding estimate DR soft value function
$$\hat V_{\mathcal{M}_\delta}^\pi(s)=\inf_{\mathbf{p}\in\hat{\mathcal{P}}(\delta)} \mathbb{E}_{\mathbf{p}}\left[\sum_{t=1}^{\infty} \gamma^{t-1}\left(r_t+\alpha \cdot \mathcal{H}\left(\pi\left(s_t\right)\right)\right) \mid \pi, s_1=s\right],$$
where $\hat{\mathcal{P}}_{s, a}(\delta):=\left\{p_{s, a} \in \Delta(|\mathcal{S}|): D_{\mathrm{KL}}\left(p_{s, a} \| \hat p_{s, a}^0\right) \leq \delta\right\}$. Similarly, the estimate DR soft $Q$-function is given by
$$\hat Q_{\mathcal{M}_\delta}^\pi(s,a)=\inf_{\mathbf{p}\in\hat{\mathcal{P}}(\delta)} \mathbb{E}_{\mathbf{p}}\left[r_1+ \sum_{t=2}^{\infty} \gamma^{t-1}\left(r_t+\alpha \cdot \mathcal{H}\left(\pi\left(s_t\right)\right)\right) \mid \pi, s_1=s, a_1=a\right].$$
Define $\hat V_{\mathcal{M}_\delta}^\star=\max_{\pi\in\Pi}\hat V_{\mathcal{M}_\delta}^\pi$ and $\hat \pi_{\mathcal{M}_\delta}^\star\in\operatorname*{argmax}_{\pi\in\Pi} \hat V_{\mathcal{M}_\delta}^\pi$.

The estimate $\hat{\mathcal{T}}_\delta^\pi $ is defined as 
$$
\hat{\mathcal{T}}_\delta^\pi Q(s, a)=\mathbb{E}[r]+\gamma \cdot \sup _{\beta \geq 0}\left\{-\beta \log \left(\mathbb{E}_{\hat p_{s, a}^0}\left[\exp \left(\frac{-V\left(s^{\prime}\right)}{\beta}\right)\right]\right)-\beta \delta\right\},
$$

\begin{assumption}\label{assump:transition-close}
    Assume $\text{KL}(p^0_{s,a}\|\hat{p}^0_{s,a})\le\varepsilon_1^2$ and $\text{supp}(p^0_{s,a})=\text{supp}(\hat p^0_{s,a})$.
    %$\text{TV}(p^0_{s,a},\hat{p}^0_{s,a})=\frac{1}{2}\sum_{s'\in S} |p^0_{s,a}(s')-\hat{p}^0_{s,a}(s')| \le\varepsilon_1$ and $\text{supp}(p^0_{s,a})=\text{supp}(\hat p^0_{s,a})$.
\end{assumption}
By Pinsker’s inequality, $\text{TV}(p^0_{s,a},\hat{p}^0_{s,a})\le \frac{1}{2}\sqrt{\text{KL}(p^0_{s,a}\|\hat{p}^0_{s,a})}\le \frac{1}{2}\varepsilon_1$.

%------------------difference between \hat\tau and \tau----------------------%

\paragraph{Bound of $\|\hat{\mathcal{T}}_\delta^\pi Q(s, a)-\mathcal{T}_\delta^\pi Q(s, a)\|$.}

\begin{lemma}\label{lem: tau Q difference}
    Under Assumption~\ref{assump:transition-close}, $$\|\hat{\mathcal{T}}_\delta^\pi Q(s, a)-\mathcal{T}_\delta^\pi Q(s, a)\| \le 2\gamma  \varepsilon_1 \frac{R_{\max}+\alpha\log |A|}{(1-\gamma)\delta} e^{(R_{\max}+\alpha\log |A| )/ (1-\gamma)\underline{\beta}}.$$
\end{lemma}

\begin{proof}
As we defined in Section 3.2,
\[
f((s,a),\beta) := -\beta \log \Big( \mathbb{E}_{p^0_{s,a}}\!\left[ e^{-V(s')/\beta} \right] \Big) - \beta \delta, 
\quad 
\hat f((s,a),\beta) := -\beta \log \Big( \mathbb{E}_{\hat p^0_{s,a}}\!\left[ e^{-V(s')/\beta} \right] \Big) - \beta \delta.
\]
From (Xu 2023, Proposition 5), the maximums of $f((s,a),\beta)$ and $\hat f((s,a),\beta)$ are achieved at $\beta^\star,\hat\beta^\star \in[0,V_{\max}/\delta]$, that is,
\begin{align*}
    \mathcal{T}_\delta^\pi Q(s,a)
= \mathbb{E}[r] + \gamma \sup_{\beta \ge 0} f((s,a),\beta) = \mathbb{E}[r] + \gamma \sup_{\beta \in [0,V_{\max}/\delta]} f((s,a),\beta), 
\\
\hat{\mathcal{T}}_\delta^\pi Q(s,a)
= \mathbb{E}[r] + \gamma \sup_{\beta \ge 0} \hat f((s,a),\beta) = \mathbb{E}[r] + \gamma \sup_{\beta \in [0,V_{\max}/\delta]} \hat f((s,a),\beta).
\end{align*}
Hence,
\[
\big| \hat{\mathcal{T}}_\delta^\pi Q(s,a) - \mathcal{T}_\delta^\pi Q(s,a) \big|
\le 
\gamma \sup_{\beta \in [0,V_{\max}/\delta]} \big| \hat f((s,a),\beta) - f((s,a),\beta) \big|.
\]

Note that $\text{supp}(p^0_{s,a})=\text{supp}(\hat p^0_{s,a})$, which implies that $F_{ p}(0)=\text{essinf}_{s'\sim p^0_{s,a}} V(s') =\text{essinf}_{s'\sim \hat p^0_{s,a}} V(s')=F_{\hat p}(0)$.
Now, we can assume that the optimal  $\beta^\star,\hat\beta^\star$ is achieved in $[\underline{\beta},V_{\max}/\delta]$, where $\underline{\beta}=\min\{\beta^\star/2,\hat\beta^\star/2,1/2\}$. 

Then, we aim to bound the supremum of $\big| \hat f((s,a),\beta) - f((s,a),\beta) \big|$ over the interval $[\underline{\beta},V_{\max}/\delta]$. Since $\log x\le x-1$ when $x\ge 1$, we have
\begin{align*}
|\hat f((s,a),\beta) - f((s,a),\beta)|
&= \beta \left| \log \Big( \mathbb{E}_{\hat p^0_{s,a}}\!\left[ e^{-V(s')/\beta} \right] \Big) - \log \Big( \mathbb{E}_{p^0_{s,a}}\!\left[ e^{-V(s')/\beta} \right] \Big) \right|\\
&\le \frac{V_{\max}}{\delta} \frac{\left| \mathbb{E}_{\hat p^0_{s,a}}\!\left[ e^{-V(s')/\beta} \right]  - \mathbb{E}_{p^0_{s,a}}\!\left[ e^{-V(s')/\beta} \right] \right|}{\min\left\{ \mathbb{E}_{\hat p^0_{s,a}}\!\left[ e^{-V(s')/\beta} \right] , \mathbb{E}_{p^0_{s,a}}\!\left[ e^{-V(s')/\beta} \right] \right\}}.
\end{align*}
Since $\text{TV}(p^0_{s,a} , \hat p^0_{s,a}) \le \varepsilon_1$ and $V_{\min}\ge0$, we have
\begin{align*}
    &\left| \mathbb{E}_{\hat p^0_{s,a}}\!\left[ e^{-V(s')/\beta} \right]  - \mathbb{E}_{p^0_{s,a}}\!\left[ e^{-V(s')/\beta} \right] \right|
= \Big| \sum_{s'\in S} \big(\hat p^0_{s,a}(s') - p^0_{s,a}(s')\big) e^{-V(s')/\beta} \Big|
\\& \le \sum_{s'\in S} \Big| \hat p^0_{s,a}(s') - p^0_{s,a}(s')\Big| e^{-V_{\min}/\beta}
\le 2\text{TV}(p^0_{s,a},\hat{p}^0_{s,a}) \le   \varepsilon_1.
\end{align*}
In addition,
$$\min\left\{ \mathbb{E}_{\hat p^0_{s,a}}\!\left[ e^{-V(s')/\beta} \right] , \mathbb{E}_{p^0_{s,a}}\!\left[ e^{-V(s')/\beta} \right] \right\} \ge e^{-V_{\max}/\underline{\beta}}.$$ 
Thus, we obtain
\[
|\hat f((s,a),\beta) - f((s,a),\beta)|
\le  \varepsilon_1\frac{V_{\max}}{\delta}   e^{V_{\max}/ \underline{\beta}},\quad \text{where } V_{\max}=\frac{R_{\max}+\alpha\log |A|}{1-\gamma}.
\]
Combining this with the earlier inequality gives
\[
\big\| \hat{\mathcal{T}}_\delta^\pi Q(s,a) - \mathcal{T}_\delta^\pi Q(s,a) \big\|
\le 
\gamma  \varepsilon_1 \frac{R_{\max}+\alpha\log |A|}{(1-\gamma)\delta} e^{(R_{\max}+\alpha\log |A| )/ (1-\gamma)\underline{\beta}} :=\varepsilon_2 .
\]

\end{proof}

%----------------------Q bound--------------------%
\paragraph{Bound of $\|\hat Q^{\pi}_{\mathcal{M}_\delta}-Q^{\pi}_{\mathcal{M}_\delta}\|$.} 
For any $\pi\in\Pi$, let $Q^{k+1}=\mathcal{T}_\delta^\pi Q^k$, $\hat Q^{k+1}=\hat{\mathcal{T}}_\delta^\pi \hat Q^k$, and $\hat Q^0=Q^0$. 
By Proposition~\ref{prop: soft policy evaluation}, we know that $Q^k$ will converge to the DR soft $Q$-value $Q^{\pi}_{\mathcal{M}_\delta}$, which is the fixed point of $T_\delta^\pi$. That is, $T_\delta^\pi  Q^{\pi}_{\mathcal{M}_\delta}=Q^{\pi}_{\mathcal{M}_\delta}$ and $Q^k \to  Q^{\pi}_{\mathcal{M}_\delta}$.
Similarly, there exists a fixed point of $\hat T_\delta^\pi$ such that $\hat T_\delta^\pi \hat Q^{\pi}_{\mathcal{M}_\delta}=\hat Q^{\pi}_{\mathcal{M}_\delta}$ and $\hat Q^k \to \hat Q^{\pi}_{\mathcal{M}_\delta}$. Then
\begin{align*}
    \|\hat Q^{\pi}_{\mathcal{M}_\delta}- Q^{\pi}_{\mathcal{M}_\delta}\| &= \|\hat T_\delta^\pi \hat Q^{\pi}_{\mathcal{M}_\delta}- T_\delta^\pi Q^{\pi}_{\mathcal{M}_\delta}\| 
    \\&= \|\hat T_\delta^\pi \hat Q^{\pi}_{\mathcal{M}_\delta} - T_\delta^\pi \hat Q^{\pi}_{\mathcal{M}_\delta} + T_\delta^\pi \hat Q^{\pi}_{\mathcal{M}_\delta} -T_\delta^\pi Q^{\pi}_{\mathcal{M}_\delta}\| 
    \\&\le \varepsilon_2 + \gamma\|\hat Q^{\pi}_{\mathcal{M}_\delta}- Q^{\pi}_{\mathcal{M}_\delta}\|\\
    \implies \|\hat Q^{\pi}_{\mathcal{M}_\delta}- Q^{\pi}_{\mathcal{M}_\delta}\| &\le \frac{\varepsilon_2}{1-\gamma}.
\end{align*}

%---------------------Method 2-------------------------%
% By Lemma~\ref{lem: tau Q difference}, $\hat Q^{k+1}$ can be written as
% $$\hat Q^{k+1}=\hat{\mathcal{T}}_\delta^\pi \hat Q^k \le \mathcal{T}_\delta^\pi \hat Q^k+\varepsilon_2 \le  \mathcal{T}_\delta^\pi  Q^k+\gamma\|\hat Q^k-Q^k\|_\infty +\varepsilon_2.$$

% Then, we subtract $Q^{\pi}_{\mathcal{M}_\delta}$ on both sides of the above equation,
% $$\hat Q^{k+1}-Q^{\pi}_{\mathcal{M}_\delta} \le  \mathcal{T}_\delta^\pi  Q^k - T_\delta^\pi Q^{\pi}_{\mathcal{M}_\delta} +\gamma\|\hat Q^k-Q^k\|_\infty +\varepsilon_2 \le    \gamma\|Q^k - Q^{\pi}_{\mathcal{M}_\delta}\|_\infty +\gamma\|\hat Q^k-Q^k\|_\infty +\varepsilon_2$$

%---------------------Regret Bound-------------------------%
\paragraph{Regret bound.} 
We define the updating policy as 
$$
\hat \pi_{k+1}=\underset{\pi \in \Pi}{\operatorname{argmin}} \ D_{\mathrm{KL}}\left(\pi(\cdot \mid s) \left\| \frac{\exp \left(\frac{1}{\alpha} \hat Q_{\mathcal{M}_\delta}^{\hat \pi_k}(s, \cdot)\right)}{Z^{\hat \pi_k}(s)}\right)\right., k=0,1, \cdots
$$
By Proposition~\ref{thm: soft policy iteration}, the policy sequence $\{\hat \pi^k\}$ converges to the optimal policy $\hat\pi^\star_{\mathcal{M}_\delta}$ under the estimate DR soft policy iteration as $k\to\infty$.

For each state $s\in\mathcal{S}$, we have $V_{\mathcal{M}_\delta}^\star(s)-V_{\mathcal{M}_\delta}^{\hat \pi_{\mathcal{M}_\delta}^\star}(s)\ge 0$. By definition, $\hat V_{\mathcal{M}_\delta}^\star(s) = \hat V_{\mathcal{M}_\delta}^{\hat \pi_{\mathcal{M}_\delta}^\star}(s)$. Then, we have
\begin{align*}
  V_{\mathcal{M}_\delta}^\star(s)-V_{\mathcal{M}_\delta}^{\hat \pi_{\mathcal{M}_\delta}^\star}(s) &\le \left|V_{\mathcal{M}_\delta}^\star(s)-\hat V_{\mathcal{M}_\delta}^\star(s)\right| + \left|\hat V_{\mathcal{M}_\delta}^{\hat \pi_{\mathcal{M}_\delta}^\star}(s)-V_{\mathcal{M}_\delta}^{\hat \pi_{\mathcal{M}_\delta}^\star}(s)\right|\\
  & = \left|\sup_{\pi} V_{\mathcal{M}_\delta}^\pi(s)-\sup_{\pi} \hat V_{\mathcal{M}_\delta}^\pi(s)\right| + \left|\hat V_{\mathcal{M}_\delta}^{\hat \pi_{\mathcal{M}_\delta}^\star}(s)-V_{\mathcal{M}_\delta}^{\hat \pi_{\mathcal{M}_\delta}^\star}(s)\right|\\
  & \le \sup_{\pi} \left| V_{\mathcal{M}_\delta}^\pi(s)-\hat V_{\mathcal{M}_\delta}^\pi(s)\right| + \left|\hat V_{\mathcal{M}_\delta}^{\hat \pi_{\mathcal{M}_\delta}^\star}(s)-V_{\mathcal{M}_\delta}^{\hat \pi_{\mathcal{M}_\delta}^\star}(s)\right|\\
  & \le 2\sup_{\pi} \left| V_{\mathcal{M}_\delta}^\pi(s)-\hat V_{\mathcal{M}_\delta}^\pi(s)\right| 
\end{align*}
Thus,
\begin{align*}
R_{\mathcal{M}_\delta}(\hat \pi_{\mathcal{M}_\delta}^\star)=\left\|V_{\mathcal{M}_\delta}^\star-V_{\mathcal{M}_\delta}^{\hat \pi_{\mathcal{M}_\delta}^\star}\right\|_{\infty}  \le 2\sup_{\pi} \left\| V_{\mathcal{M}_\delta}^\pi-\hat V_{\mathcal{M}_\delta}^\pi\right\|_\infty  \le 2\sup_{\pi} \left\| Q_{\mathcal{M}_\delta}^\pi-\hat Q_{\mathcal{M}_\delta}^\pi\right\|_\infty \le \frac{2\varepsilon_2}{1-\gamma}.
\end{align*}

\end{document}